\documentclass[11pt]{article}
\usepackage{amsmath}
\usepackage[colorlinks=true, allcolors=blue]{hyperref}
\usepackage[algo2e, ruled, vlined]{algorithm2e}
\usepackage{amsfonts}
\usepackage{graphicx}
\usepackage{amsthm}
\usepackage{dsfont}
\usepackage{mathtools}
\usepackage{bbm}
\usepackage{cleveref}
\usepackage{complexity}
\usepackage{enumitem}
\usepackage{tikz}
\usetikzlibrary{positioning}
\usepackage{thm-restate}
\usepackage{float}
\usetikzlibrary{calc}
\usetikzlibrary{shapes,decorations.markings}
\usepackage{caption}
\usepackage{subcaption}

\usepackage[english]{babel}
\usepackage[utf8x]{inputenc}
\usepackage[T1]{fontenc}

\usepackage[margin=1.in]{geometry}

\usepackage[colorinlistoftodos]{todonotes}

\newtheorem{theorem}{Theorem}
\newtheorem{claim}{Claim}
\newtheorem{fact}{Fact}
\newtheorem{lemma}{Lemma}
\newtheorem{corollary}{Corollary}

\theoremstyle{definition}
\newtheorem{definition}{Definition}

\newcommand{\Ex}{\mathop{\mathbf{E}}}
\renewcommand{\Pr}{\mathop{\mathbf{Pr}}}

\newcommand{\doubleP}{\mathbb{P}}
\newcommand{\doubleR}{\mathbb{R}}

\newcommand{\rext}{{r_{\mathrm{ext}}}}
\newcommand{\kext}{{k_{\mathrm{ext}}}}
\newcommand{\ellext}{{\ell_{\mathrm{ext}}}}

\newcommand{\ellsigv}{{\ell_{\mathrm{sigv}}}}
\newcommand{\rlen}{{r_{\mathrm{len}}}}

\newcommand{\ellhigh}{\ell_{\mathrm{high}}}
\newcommand{\ellbias}{\ell_{\mathrm{bias}}}
\newcommand{\ellsigs}{\ell_{\mathrm{sigs}}}
\newcommand{\ellflat}{\ell_{\mathrm{flat}}}
\newcommand{\ellgood}{\ell_{\mathrm{good}}}

\newcommand{\SigV}{\mathrm{SigV}}

\newcommand{\Bad}{\mathrm{Bad}}
\newcommand{\High}{\mathrm{High}}
\newcommand{\ch}{\mathrm{cnt}_{\mathrm{high}}}
\newcommand{\cb}{\mathrm{cnt}_{\mathrm{bias}}}
\newcommand{\wt}{\ \widetilde{\to} \ }

\makeatletter
\newcommand*\rel@kern[1]{\kern#1\dimexpr\macc@kerna}
\newcommand*\widebar[1]{%
  \begingroup
  \def\mathaccent##1##2{%
    \rel@kern{0.8}%
    \overline{\rel@kern{-0.8}\macc@nucleus\rel@kern{0.2}}%
    \rel@kern{-0.2}%
  }%
  \macc@depth\@ne
  \let\math@bgroup\@empty \let\math@egroup\macc@set@skewchar
  \mathsurround\z@ \frozen@everymath{\mathgroup\macc@group\relax}%
  \macc@set@skewchar\relax
  \let\mathaccentV\macc@nested@a
  \macc@nested@a\relax111{#1}%
  \endgroup
}
\makeatother

\renewcommand{\bar}{\widebar}
\renewcommand{\hat}{\widehat}
\renewcommand{\tilde}{\widetilde}

\newcommand{\wtone}{\overset{\mathrm{flat}}{\longrightarrow}}
\newcommand{\T}{\mathcal{T}}
\newcommand{\Tone}{\T^{(1)}_{v_0}}
\newcommand{\Ttwo}{\T^{(2)}_{v_1}}

\newcommand{\fail}{\mathbf{halt}}
\newcommand{\chj}{\ch^{(j)}}
\newcommand{\cbj}{\cb^{(j)}}

\usepackage{csquotes}

\newcommand{\xin}[1]{}
\newcommand{\avishay}[1]{}
\newcommand{\hongxun}[1]{}
\newcommand{\junzhao}[1]{}

\title{Tight Time-Space Lower Bounds for Constant-Pass Learning}

\author{Xin Lyu\footnote{Department of Computer Science, University of California, Berkeley. Email: \href{mailto:lyuxin1999@gmail.com} {\url{lyuxin1999@gmail.com}}. Supported by Avishay Tal's Sloan Research Fellowship and NSF CAREER Award CCF-2145474, and Jelani Nelson's ONR grant N00014-18-1-2562.} \and Avishay Tal\footnote{Department of Computer Science, University of California, Berkeley. Email: \href{mailto:atal@berkeley.edu}{\url{atal@berkeley.edu}}. Supported by a Sloan Research Fellowship and NSF CAREER Award CCF-2145474.} \and Hongxun Wu\footnote{Department of Computer Science, University of California, Berkeley. Email: \href{mailto:wuhx@berkeley.edu}{\url{wuhx@berkeley.edu}}. Supported by Avishay Tal's Sloan Research Fellowship, NSF CAREER Award CCF-2145474, and Jelani Nelson's ONR grant N00014-18-1-2562.} \and Junzhao Yang\footnote{Institute for Interdisciplinary Information Sciences, Tsinghua University. Email: \href{mailto:yang-jz20@mails.tsinghua.edu.cn}{\url{yang-jz20@mails.tsinghua.edu.cn}.}}}

\begin{document}

\maketitle
\begin{abstract}
In his breakthrough paper, Raz showed that any parity learning algorithm requires either quadratic memory or an exponential number of samples [FOCS'16, JACM'19]. A line of work that followed extended this result to a large class of learning problems. Until recently, all these results considered learning in the streaming model, where each sample is drawn independently, and the learner is allowed a single pass over the stream of samples. 
Garg, Raz, and Tal [CCC'19] considered a stronger model, allowing multiple passes over the stream. In the $2$-pass model, they showed that learning parities of size $n$ requires either a memory of size $n^{1.5}$ or at least $2^{\sqrt{n}}$ samples. (Their result also generalizes to other learning problems.)

In this work, for any constant $q$, we prove tight memory-sample lower bounds for any parity learning algorithm that makes $q$ passes over the stream of samples. We show that such a learner requires either $\Omega(n^{2})$ memory size or at least $2^{\Omega(n)}$ samples.  
Beyond establishing a tight lower bound, this is the first non-trivial lower bound for $q$-pass learning for any $q\ge 3$.
Similar to prior work, our results extend to any learning problem with many nearly-orthogonal concepts.

We complement the lower bound with an upper bound, showing that parity learning with $q$ passes can be done efficiently with $O(n^2/\log q)$ memory.
\end{abstract}
\thispagestyle{empty}

\newpage\thispagestyle{empty}
{\normalsize
\setcounter{tocdepth}{2}
\tableofcontents
\thispagestyle{empty}
}
\newpage
\pagenumbering{arabic}
\section{Introduction}
A growing recent line of works studied the efficiency of learning under memory constraints
\cite{Shamir14,SVW16,Raz16,KRT17,MT17,MM17,raz2017time,MM18, GargRT18-extractor,BGY18,
DaganS18,
garg2019time, SSV19,
DKS19,GKR20,DBLP:conf/approx/GargKLR21,MSSV22,liu2023memory}.
This study was initiated by the beautiful works of Shamir \cite{Shamir14} and Steinhardt, Valiant, and Wager \cite{SVW16}. Specifically, Steinhardt et al. \cite{SVW16} conjectured that any learning parity algorithm requires either quadratic memory or an exponential number of examples. 
In a breakthrough result, Raz \cite{Raz16} proved this conjecture.
While we have two simple algorithms for parity learning: (i) Gaussian Elimination that uses $O(n^2)$ space and $O(n)$ samples, and (ii) Brute-force search that uses $O(n)$ space and $O(2^n)$ samples, Raz showed that there is no learning algorithm that uses $o(n^2)$ space and $2^{o(n)}$ samples \cite{Raz16}. This demonstrated that efficient learning requires a large memory -- in this case, at least $\Omega(n^2)$ memory bits.

Follow-up work extended and generalized the lower bounds techniques to a wide array of learning problems such as learning sparse parities, learning DNFs, learning decision trees, learning juntas, \cite{KRT17,GargRT18-extractor} learning low-degree polynomials \cite{BGY18,GargRT18-extractor}, learning from sparse equations and low-degree equations \cite{GargRT18-extractor}, learning codewords from random coordinates \cite{raz2017time,MM18,GargRT18-extractor}, learning parities with noisy inputs \cite{DBLP:conf/approx/GargKLR21}, and more. In all the above, it is shown that any learning algorithm for the corresponding concept class on input size $n$, requires either super-linear size memory, or super polynomial number of samples. Work towards a tight characterization of memory-samples lower bounds was done by \cite{GLM20}, but such a full characterization is still missing with polynomial gaps on the memory required for efficiently learning classical concepts classes such as juntas, DNFs, decision trees \cite{KRT17}.

Most of the works above modeled the learner as a streaming algorithm, observing the random labeled examples one at a time. 
More precisely, the lower bounds proved were in the stronger model of read-once branching programs that captures bounded-space streaming computation in a non-uniform setting.
Recent exciting work by \cite{liu2023memory} extended the model to include quantum memory in addition to classical memory and showed that Raz's result extends even if the learner has additionally $o(n)$ qubits at its disposal.

Dagan and Shamir~\cite{DaganS18} and Garg, Raz, and Tal~\cite{garg2019time} considered the model of multi-pass learners. In this model, the learner makes several passes over the stream of examples in the same order. Dagan and Shamir~\cite{DaganS18} proved polynomial lower bounds on the number of samples in such setting. Garg, Raz and Tal~\cite{garg2019time} obtained a subexponential lower bound on the number of samples $2^{\Omega(\sqrt{n})}$ for any \emph{two-pass} learning parity algorithm with $o(n^{1.5})$ space.
The result more generally implies lower bounds for any of the aforementioned learning problems.
Indeed, the lower bounds are proved in the extractor-based framework of \cite{GargRT18-extractor} and all the aforementioned learning problems fall under this framework.

Despite the strong lower bound, the GRT result was not known to be tight for two-pass learning, as no efficient algorithm with $o(n^2)$ space was known in this setting. Moreover, their result did not translate to the multi-pass setting with more than two passes, and, as indicated in their paper, some of their techniques are quite delicate, and it is far from clear how to extend them to more than two passes \cite{garg2019time}.

Proving lower bounds for multi-pass learners is much more challenging, as such learners can store information during the first pass that would make examples in the second pass somewhat predictable, correlated with one another, or correlated with the hidden vector.

One might wonder whether more passes can help in learning. Indeed, when the number of passes is quasi-polynomial, a parity learning algorithm with $n^{O(\log n)}$ passes, $n^{O(\log n)}$ samples, and $O(n)$ space follows from the following two facts: (i) solving linear equations can be done in $O(\log^2 n)$ depth \cite{Csanky} (ii) Barrington's simulation of $O(\log^2 n)$ depth by length $n^{O(\log n)}$ read-once branching programs \cite{Barrington86}.

\subsection{Our Results}
We study time-space lower bounds for multi-pass learning problems. We provide a nearly tight lower bound for two-pass learning parity algorithms:

\begin{theorem}[Informal version of \Cref{theo:two-pass-main-result}]
Any two-pass algorithm for $n$-bit parity learning requires either $\Omega(n^2)$ bits of memory or $2^{\Omega(n)}$ many samples. Otherwise, the algorithm succeeds with probability at most $2^{-\Omega(n)}$.
\end{theorem}

Moreover, our results generalize to any constant-pass learner and, moreover, imply nearly similar bounds for any algorithm with at most $o(\log \log n)$ passes.

\begin{theorem}[Informal version of \Cref{theo:multi-pass-main-result}]
There is a universal constant $C>0$ such that the following holds. For any $q\ge 2$, letting $c_q = C\cdot 100^{3^q}$, any $q$-pass algorithm for $n$-bit parity learning requires either $n^2/c_q$ bits of memory or $\exp({n/c_q})$ many samples. Otherwise, the algorithm succeeds with probability at most $2^{-(n/c_q)}$.
\end{theorem}

We stress that the multi-pass lower bound is not a direct generalization of the two-pass one. It requires us to revisit a key technique in the two-pass proof (which we call the ``transfer lemma''), and extend the technique to the multi-pass case with a significantly more involved argument.

\paragraph*{Extractor-based framework.} Our results apply more generally to any learning problem with many nearly pairwise orthogonal concepts (i.e. concepts that agree on roughly half of the inputs). Alternatively, 
 to any learning problem whose associated matrix (as defined in~\cite{raz2017time}) exhibits an extractor-property~\cite{GargRT18-extractor}, as defined next.

Let $A$ be a finite domain, and let $X$ be a concept class over $A$, where each $x \in X$ represents a function (or concept) mapping $A$ to $\{-1, 1\}$. We naturally associate with the concept class a matrix $M \in \{-1,1\}^{A \times X}$ whose rows correspond to samples and columns correspond to concepts/functions. Then, $M$ describes the following learning problem: An unknown $x\in X$ is chosen uniformly at random. A learner tries to learn $x$ from a stream of labeled samples, 
$(a_1, M(a_1,x)), (a_2, M(a_2,x)), \ldots$ where each $a_i$ is uniformly distributed over $A$. In particular, we consider the setting in which the learner can see the \emph{same} stream of samples for $q\ge 2$ passes.

Our lower bounds apply to any learning problem whose corresponding matrix $M$ has certain extractor properties: Any large submatrix of $M$ has a similar fraction of $1$'s and $-1$'s. More precisely, we say that $M$ is a $(k, \ell, r)$-extractor if for any submatrix of at least $2^{-k}\cdot |A|$ rows and at least $2^{-\ell}\cdot |X|$ columns, the fraction of entries with value $1$ is $\frac{1}{2}\pm 2^{-r}$. 
(For example, parity learning has parameters $k,\ell, r = \Omega(n)$.) We show that any two-pass learning for the learning problem associated with $M$ requires either $\Omega(k\cdot \min(\ell,k))$ memory or at least $2^{\Omega(r)}$ samples. For $q$-pass learning, we show that the learning problem requires either $\Omega(k\cdot \min(\ell,k))/c_q$ memory or at least $2^{\Omega(r/c_q)}$ samples for $c_q = 100^{3^q}$.

Our main theorems, \Cref{theo:two-pass-main-result} and~\Cref{theo:multi-pass-main-result}, are actually stated for matrices that are $L_2$-extractors, since these extractors are more convenient to work with in our proof. However, a simple reduction from \cite[Corollary~3]{GargRT18-extractor} shows that any standard extractor as above is also a $(\Omega(k+r), \Omega(\ell+r), \Omega(r))$-$L_2$-Extractor. 
Our results thus apply to all the aforementioned concept classes (juntas, DNFs, Decision trees, low-degree polynomials, codewords) as their corresponding matrices form $L_2$-Extractors with good parameters.

\paragraph*{A non-trivial multi-pass algorithm.} One might wonder whether the lower bound can be strengthened to show that any $n^{o(1)}$-pass learner requires either $\Omega(n^2)$-memory or $2^{\Omega(n)}$ samples to learn parity. Our next result shows that this is not the case, and efficient learning with $o(n^2)$ memory is possible for any $q=\omega(1)$.

\begin{theorem}[Informal version of \Cref{theo:multi-pass-algorithm}]
For any $q\le 2^n$, there is a $q$-pass algorithm for $n$-bit parity learning that uses $O(n^2/\log(q))$ bits of memory and $O(qn)$ samples.
\end{theorem}

\section{Technical Overview}
In this section, we will present the road map for our paper, including the difficulties and a sketch of our main ideas for bypassing them.

\subsection{Recap of the One-Pass Lower Bound}

Both our work and the previous work on the two-pass learning bound \cite{garg2019time} are based on the proof techniques for one-pass lower bound \cite{raz2017time,GargRT18-extractor}. Let us sketch its main idea. 

\paragraph*{Computational Model.} The proof models the computation as a read-once branching program: The input to the branching program is a sequence of pairs $(a_1, b_1), (a_2, b_2), \dots, (a_T, b_T) \in A \times \{-1, 1\}$. Each of them represents an equation $M(a_i, x) = b_i$. These $a_1, a_2, \dots, a_T \in A$ are sampled uniformly at random, while $b_1, b_2, \dots, b_T$ are all generated according to the hidden vector $x \in X$. (In our paper, for simplicity, we will identify $X$ with $\{0,1\}^n$.) We label the layers of the branching program by $0, 1, 2, \dots, T$. Let $v$ be the current vertex. Initially, it is equal to the starting vertex of the branching program at layer $0$. At each layer $i$, we read $(a_{i + 1}, b_{i + 1})$ and move $v$ along one corresponding edge to layer $i + 1$. At the end of the computation, $v$ will reach the last layer $T$. Then, it outputs a vector $x_v$. We say it is successful if and only if $x_v = x$.

The length of this branching program is the number of samples $T$. The width is $2^S$ where $S \leq n^2/16$ is the memory bound. We want to prove that when $S \leq n^2 / 16$ and $T \leq 2^{n / 16}$, the program cannot succeed with constant probability.

\paragraph*{Main Idea of the One-Pass Lower Bound.} When outputting $x_v$, the optimal strategy is to output the $x'$ with the highest posterior probability $\doubleP_{x \mid v}(x')$. Intuitively, if the distribution $\doubleP_{x\mid v}$ is very spread, measured by its $\ell_2$ norm, the vertex $v$ will have a small chance of answering $x$ correctly. We will define its $\ell_2$ norm as
$$\|\doubleP_{x \mid v}\|_2 \coloneqq \Ex_{x' \sim X} \left[\doubleP_{x \mid v}^2(x')\right]^{\frac{1}{2}}.$$
Initially, as $X = \{0,1\}^n$, the uniform prior $\doubleP_x$ has $\|\doubleP_x\|_2 = 2^{-n}$.
As $v$ moves along the computational path, the posterior distribution $\doubleP_{x \mid v}$ evolves. In the end, one can show that, for some $\epsilon > 0$, if $\|\doubleP_{x \mid v}\|_2 \leq 2^{\epsilon n} \cdot 2^{-n}$, the probability that $v$ answers $x$ correctly will be less than $2^{-\Theta(n)}$. (Think of $\epsilon$ as a small constant, say $\epsilon = 0.1$.)

Hence, to upper bound the success probability, it is sufficient to upper bound the probability that we ever reach a vertex $v$ with $\|\doubleP_{x \mid v}\|_2 > 2^{\epsilon n} \cdot 2^{-n}$ on our computational path. To show this, we will enumerate all target vertices $t$ with $\|\doubleP_{x \mid t}\|_2 > 2^{\epsilon n} \cdot 2^{-n}$ and prove that the probability of reaching a fixed $t$ is less than $2^{-\Theta(n^2)}$. Then, our desired upper bound follows from a union bound over all $2^{n^2 / 16 + O(n)}$ possibilities of vertex $t$. 

\paragraph*{Progress Measure.} To study the probability of reaching $t$, we need to look at the similarity between our current posterior $\doubleP_{x \mid v}$ and the target $\doubleP_{x \mid t}$, captured by their inner product $\langle \doubleP_{x \mid v}, \doubleP_{x \mid t} \rangle$, defined as 
$$\left\langle \doubleP_{x \mid v}, \doubleP_{x \mid t} \right\rangle \coloneqq \Ex_{x'\sim X}\left[\doubleP_{x \mid v}(x') \cdot \doubleP_{x \mid t}(x')\right] = \frac{1}{|X|}\sum_{x' \in X} \doubleP_{x \mid v}(x') \cdot \doubleP_{x \mid t}(x') .$$
This measures our progress towards $t$. To show that we reach $t$ with a very small probability, we will show that for a uniform random $a \in A$, reading equation $M(a, x) = b$ will, w.h.p., makes little progress. 

Let the posterior distribution after reading this equation be $\doubleP^{(a,b)}_{x \mid v}$. Then after normalization, the similarity becomes 
$$\left\langle \doubleP^{(a,b)}_{x \mid v}, \doubleP_{x\mid t}\right\rangle = 
\frac{1}{|X|} \cdot \sum_{\substack{x' \in X \\ M(a,x') = b}} \doubleP_{x\mid v}(x') \cdot \doubleP_{x \mid t}(x') \Bigg/ \sum_{\substack{x' \in X \\ M(a,x') = b}} \doubleP_{x\mid v}(x').$$

We say $a \in A$ cuts $\doubleP_{x \mid v}$ evenly if $$\sum_{x' \in X : M(a, x') = 0} \doubleP_{x\mid v}(x') \approx \sum_{x' \in X : M(a, x') = 1} \doubleP_{x\mid v}(x').$$ Similarly, we say $a \in A$ cuts the point-wise product $\doubleP_{x \mid v} \cdot \doubleP_{x \mid t}$ evenly if this holds for $\doubleP_{x \mid v} \cdot \doubleP_{x \mid t}$ instead of $\doubleP_{x\mid v}$.

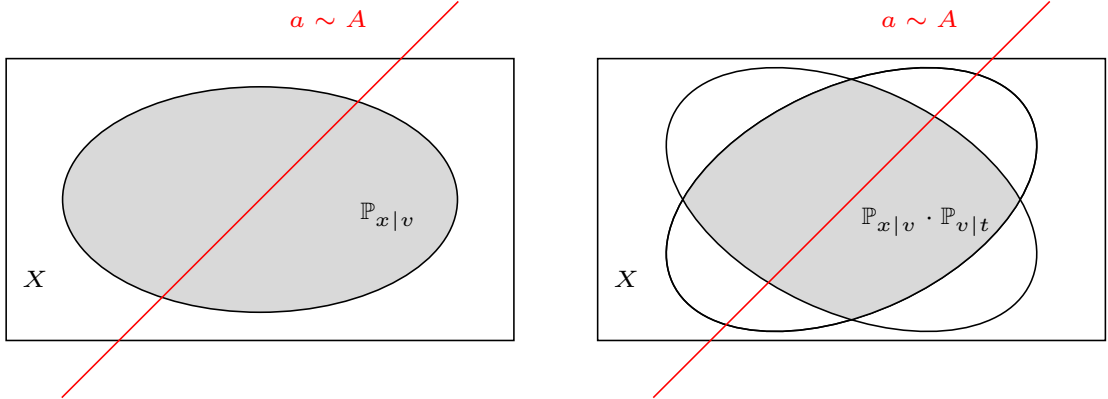
\begin{figure}
    \centering
     \begin{subfigure}[b]{0.4\textwidth}
         \centering
           \scalebox{1.5}{
            \begin{tikzpicture}
        
        \node[rectangle, draw, minimum width=4.5cm, minimum height=2.5cm] (rect) {};
        \node[ellipse, draw, fill=gray!30, minimum width=3.5cm, minimum height=2cm] (oval) at (rect.center) {};
        
        \coordinate  (aa) at ($(rect.south west) + (0.5cm, -0.5cm)$) {};
        \coordinate  (ab) at ($(rect.north east) + (-0.5cm, 0.5cm)$) {};
        \draw[red, decoration={markings, mark=at position 0.9 with {\node[red, xshift = -0.8cm, yshift = 0.2cm,font=\tiny] (a) {$a \sim A$};}}, postaction={decorate}] (aa) -- (ab);
        
        \node[right=-1cm of oval, yshift = -0.15cm, font=\tiny] (x) {$\doubleP_{x \mid v}$};
        \node[right=-4cm of oval, yshift = -0.7cm,font=\tiny] (x) {$X$};
        
        \end{tikzpicture}}
     \end{subfigure}
     \hspace{1cm}
     \begin{subfigure}[b]{0.4\textwidth}
         \centering
           \scalebox{1.5}{
            \begin{tikzpicture}

        \node[ellipse, minimum width=3.5cm, minimum height=2cm] (oval) at (rect.center) {};
\node[rectangle, draw, minimum width=4.5cm, minimum height=2.5cm] (rect) {};
\begin{scope}
  \clip[draw] (rect.center) ellipse [x radius=1.75cm, y radius=1cm, rotate=25];
  \fill[gray!30] (rect.center) ellipse [x radius=1.75cm, y radius=1cm, rotate=-25];
\end{scope}
\draw (rect.center) ellipse [x radius=1.75cm, y radius=1cm, rotate=25];
\draw (rect.center) ellipse [x radius=1.75cm, y radius=1cm, rotate=-25];

\coordinate  (aa) at ($(rect.south west) + (0.5cm, -0.5cm)$) {};
\coordinate  (ab) at ($(rect.north east) + (-0.5cm, 0.5cm)$) {};
\draw[red, decoration={markings, mark=at position 0.9 with {\node[red, xshift = -0.8cm, yshift = 0.2cm,font=\tiny] (a) {$a \sim A$};}}, postaction={decorate}] (aa) -- (ab);

\node[right=-2.3cm of rect, yshift = -0.2cm,font=\tiny] (x) {$\doubleP_{x \mid v} \cdot \doubleP_{v \mid t}$};
\node[right=-4cm of oval, yshift = -0.7cm,font=\tiny] (x) {$X$};

\end{tikzpicture}}
     \end{subfigure}
    \caption{A uniformly random equation $a \sim A$ will w.h.p. cut both $\doubleP_{x \mid v}$ and $\doubleP_{x \mid v} \cdot \doubleP_{x \mid t}$ almost evenly into two parts: (1) those $x \in X$ with $M(a,x) = 0$ and (2) those $x \in X$ with $M(a, x) = 1$. }
    \label{fig:balance-cut}
\end{figure}

As shown in  \Cref{fig:balance-cut}, due to the extractor property of $M$, when $\doubleP_{x\mid v}$ and $\doubleP_{x \mid v} \cdot \doubleP_{x \mid t}$ are spread enough, a \emph{uniformly random} $a \in A$ cuts both  $\doubleP_{x\mid v}$ and $\doubleP_{x \mid v} \cdot \doubleP_{x \mid t}$ in half with high probability ($1 - 2^{-\Theta(n)}$). Hence each time we see a random equation, it will most likely halve both the numerator and the denominator, which will not help us make progress. Suppose we are unlucky. the rare event with probability $2^{-\Theta(n)}$ happens. Let us see how the similarity might change:
\begin{enumerate}
    \item If $a \in A$ cuts $\doubleP_{x \mid v}$ unevenly: When the denominator $\sum_{x' \in X, M(a, x') = b} \doubleP_{x\mid v}(x') < 2^{-c}$, the similarity may be larger by a factor of $2^c$. This causes huge progress. For now, we ignore this issue and assume it never arises. We will later handle it by designing certain ``stopping rules''. 
    \item If $a \in A$ cuts $\doubleP_{x \mid v} \cdot \doubleP_{x \mid t}$ unevenly but still cut $\doubleP_{x \mid v}$ evenly: In this case, the worst case is that the numerator does not decrease at all, while the denominator is still halved. Then the similarity doubles.
\end{enumerate}
Initially, the similarity between uniform prior and target $t$ is $\left\langle \doubleP_x, \doubleP_{x \mid t}\right\rangle = 2^{-2n}$. In order to reach the target node $t$, which has similarity $\langle \doubleP_{x \mid t}, \doubleP_{x \mid t} \rangle = \|\doubleP_{x \mid t}\|_2^2 \geq 2^{2 \epsilon n} \cdot 2^{-2n}$ with itself, the second case has to happen $2 \epsilon n$ times. Intuitively, this tells us the probability of reaching $t$ is less than $\left(2^{-\Theta(n)}\right)^{2\epsilon n} = 2^{-\Theta(n^2)}$.

\paragraph*{Stopping Rules.} Now we turn to handle the first case. Although it happens with probability at most $2^{-\Theta(n)}$, it only needs to happen once to make a huge progress. Thus, the $2^{-\Theta(n)}$ probability is not enough to afford the union bound over all $2^{\Theta(n^2)}$ many targets $t$.

Luckily, we do not have to do union-bound. Observe that whether $a \in A$ cuts $\doubleP_{x \mid v}$ evenly is independent of the target $t$. Whenever we see an equation $a \in A$ that cuts $\doubleP_{x \mid v}$ unevenly in our computational path, we can stop right away. Since there are $T \approx 2^{n / 16}$ layers in our branching program, a simple union bound over them shows that the overall probability of stopping is still $2^{-\Theta(n)}$. Moreover, if we did not stop, the previous argument shows that we reach any target vertex $t$ with probability $2^{-\Theta(n^2)}$. Overall, our algorithm succeeds with a very small probability. 

But this is not the only stopping rule. Recall that a uniformly random $a \in A$ cuts both $\doubleP_{x \mid v}$ and $\doubleP_{x \mid v} \cdot \doubleP_{x \mid t}$ evenly w.h.p. (by extractor property) only when they are spread enough. We also need stopping rules to guarantee this. Formally, we have the following stopping rules. 

\begin{itemize}
    \item (Bad Edge) If $a$ does not cut $\doubleP_{x \mid v}$ evenly, we stop.
    \item (Significant State) If $\|\doubleP_{x \mid v}\|_2 \geq 2^{\epsilon n} \cdot 2^{-n}$, we stop.

    This guarantees that the distribution of $\doubleP_{x \mid v}$ will be spread enough for the extractor property.
    \item (Significant Value) If $\doubleP_{x \mid v}(x) > 2^{\epsilon n} \cdot 2^{-n}$, we stop. 
    
    After applying this rule, we know $\|\doubleP_{x \mid v}\|_{\infty} \leq 2^{\epsilon n}$. Since $\|\doubleP_{x \mid v} \cdot \doubleP_{x \mid t}\|_2 \leq \|\doubleP_{x \mid v}\|_{\infty} \cdot \|\doubleP_{x \mid t}\|_2$, this guarantees that $\doubleP_{x \mid v} \cdot \doubleP_{x \mid t}$ will be spread enough for the extractor property.  
\end{itemize}


\subsection{The Proof Framework: Two Passes} \label{subsec:framwork-twopass}

Our work builds on the approach taken by the previous two-pass lower bound \cite{garg2019time}. We will now sketch their proof framework.

\paragraph*{Computational Model. } A two pass branching program reads its input $(a_1, b_1), (a_2, b_2), \dots, (a_T, b_T)$ twice in the \emph{exact same order}. At the first pass, the starting vertex is $v_0$, and after reading its input, the computational path reaches a vertex $v_1$ at the end of the first pass (which is also the first layer of the second pass). In the second pass, the computational path starts from $v_1$ and reaches $v_2$ at the last layer after reading the input again. Then it will output a vector $x_{v_2} \in X$. 

For any two vertices $u$ and $v$ in the program, we use $u \wt v$ to denote the following event (over $x, a_1, a_2, \dots, a_T$): Imagine that we set the starting vertex of the branching program at $u$, the path from $u$ determined by $x, a_1, a_2, \dots, a_T$ reaches $v$ without stopping.
\begin{itemize}
    \item For a vertex $v_1$ in the last layer of the first pass, $v_0 \wt v_1$ means that the first pass ends at $v_1$. 
    \item For any vertex $v_1$ in the last layer of the first pass, and vertex $v_2$ in the last layer of the second pass, $v_1 \wt v_2$ means that the second pass will end at $v_2$ if it were starting at $v_1$. 
\end{itemize}


\paragraph*{First Attempt.} Moving from one pass to two passes, one might consider the following natural approach: First, apply the above argument to the first pass and conclude that, at the end of the first pass, the similarity $\langle \doubleP_{x \mid v}, \doubleP_{x \mid t}\rangle$ is small. Second, apply it to the second pass and argue that such similarity grows slowly in the second pass too. 

However, such a direct approach would not work. Consider a program that (1) magically learns $x$, (2) remembers $x \oplus a_1$ (thinking of $x$ and $a_1$ as bit strings), and (3) forgets $x$ and $a_1$ at the end of the first pass. Conditioning on what $v$ remembers, $x \oplus a_1$, the distribution $\doubleP_{x \mid v}$ is uniformly random, just like the prior $\doubleP_{x}$. This is because $x$ is encrypted by the one-time pad using $a_1$. So in the eyes of our analysis, this magical first pass is no different from a trivial first pass. If we do not rule out the possibility of such a magical first pass, what could happen in the second pass is that, after seeing $a_1$ again, the program combines $a_1$ with its knowledge of $x \oplus a_1$ and immediately decodes $x$.

\paragraph*{Remembering the First Pass. } To prove any non-trivial lower bound for two passes, it is necessary to rule out such a program. This program shows that analyzing two passes separately would not work (at least for this specific argument). Therefore, we will analyze two passes together. 

The first observation is that one can w.l.o.g. assume that, when at the $i$-th layer of the second pass, the program knows which vertex it was at in the $i$-th layer of the first pass. This is because the program can keep a copy of the first pass in its memory, which only blows up the memory by a factor of two.

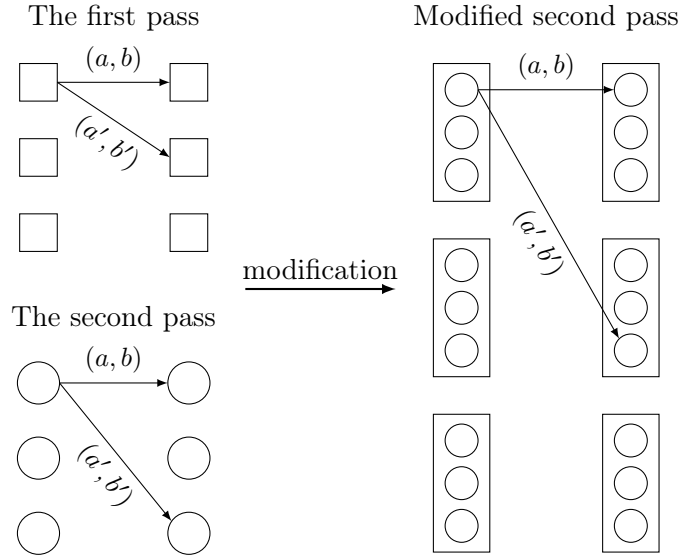
\begin{figure}[h]
    \centering
    \begin{tikzpicture}
        \coordinate (pass-1-bottom-left) at (-3, 1);
        \coordinate (sqr-size) at (0.5, 0.5);
            \coordinate (sqr-size-x) at (0.5, 0);
            \coordinate (sqr-size-y) at (0, 0.5);
        \def\circr{0.28}
        \coordinate (circ-o) at (0.25, 0.25);
        \coordinate (step-size-x) at (2, 0);
        \coordinate (step-size-y) at (0, 1);
        \coordinate (p1-00) at ($(pass-1-bottom-left)+0*(step-size-x)+0*(step-size-y)$);
        \draw (p1-00) rectangle ($(p1-00)+(sqr-size)$);
        \coordinate (p1-01) at ($(pass-1-bottom-left)+0*(step-size-x)+1*(step-size-y)$);
        \draw (p1-01) rectangle ($(p1-01)+(sqr-size)$);
        \coordinate (p1-02) at ($(pass-1-bottom-left)+0*(step-size-x)+2*(step-size-y)$);
        \draw (p1-02) rectangle ($(p1-02)+(sqr-size)$);
        \coordinate (p1-10) at ($(pass-1-bottom-left)+1*(step-size-x)+0*(step-size-y)$);
        \draw (p1-10) rectangle ($(p1-10)+(sqr-size)$);
        \coordinate (p1-11) at ($(pass-1-bottom-left)+1*(step-size-x)+1*(step-size-y)$);
        \draw (p1-11) rectangle ($(p1-11)+(sqr-size)$);
        \coordinate (p1-12) at ($(pass-1-bottom-left)+1*(step-size-x)+2*(step-size-y)$);
        \draw (p1-12) rectangle ($(p1-12)+(sqr-size)$);

        \coordinate (pass-2-bottom-left) at (-3, -3);
        \coordinate (p2-00) at ($(pass-2-bottom-left)+0*(step-size-x)+0*(step-size-y)+(circ-o)$);
        
        \draw ($(p2-00)$) circle (\circr);
        \coordinate (p2-01) at ($(pass-2-bottom-left)+0*(step-size-x)+1*(step-size-y)+(circ-o)$);
        \draw ($(p2-01)$) circle (\circr);
        \coordinate (p2-02) at ($(pass-2-bottom-left)+0*(step-size-x)+2*(step-size-y)+(circ-o)$);
        \draw ($(p2-02)$) circle (\circr);
        \coordinate (p2-10) at ($(pass-2-bottom-left)+1*(step-size-x)+0*(step-size-y)+(circ-o)$);
        \draw ($(p2-10)$) circle (\circr);
        \coordinate (p2-11) at ($(pass-2-bottom-left)+1*(step-size-x)+1*(step-size-y)+(circ-o)$);
        \draw ($(p2-11)$) circle (\circr);
        \coordinate (p2-12) at ($(pass-2-bottom-left)+1*(step-size-x)+2*(step-size-y)+(circ-o)$);
        \draw ($(p2-12)$) circle (\circr);
        \coordinate (arrow1-p) at ($(p1-02)+(sqr-size-x)+0.5*(sqr-size-y)$);
        \coordinate (arrow1-q) at ($(p1-12)+0.5*(sqr-size-y)$);
        \draw[-latex] (arrow1-p) -- (arrow1-q);
        \node[anchor=south,font=\small] at ($(arrow1-p)!0.5!(arrow1-q)$) {$(a, b)$};
        
        \coordinate (arrow2-p) at ($(p1-02)+(sqr-size-x)+0.5*(sqr-size-y)$);
        \coordinate (arrow2-q) at ($(p1-11)+0.5*(sqr-size-y)$);
        \draw[-latex] (arrow2-p) -- (arrow2-q);
        \node[anchor=north,font=\small,rotate=-36] at ($(arrow2-p)!0.56!(arrow2-q)$) {$(a', b')$};
        
        \coordinate (arrow3-p) at ($(p2-02)+(0:\circr)$);
        \coordinate (arrow3-q) at ($(p2-12)+(180:\circr)$);
        \draw[-latex] (arrow3-p) -- (arrow3-q);
        \node[anchor=south,font=\small] at ($(arrow3-p)!0.5!(arrow3-q)$) {$(a, b)$};
        
        \coordinate (arrow4-p) at ($(p2-02)+(0:\circr)$);
        \coordinate (arrow4-q) at ($(p2-10)+(145:\circr)$);
        \draw[-latex] (arrow4-p) -- (arrow4-q);
        \node[anchor=north,font=\small,rotate=-53] at ($(arrow4-p)!0.57!(arrow4-q)$) {$(a', b')$};


        \def\mpleftx{2.5}
        \def\mpbottomy{-3}
        \def\mptopy{3.5}
        \def\mpgapy{0.5}
        \def\mprectsizey{{((\mptopy-\mpbottomy)-2*\mpgapy)/3}}
        \def\mprectsizex{0.75}

        \coordinate (mp-bottom-left) at (\mpleftx,\mpbottomy); 
        \coordinate (mp-top-left) at (\mpleftx,\mptopy);
        \coordinate (mp-gap-size-x) at ($(1.5,0)$);
        \coordinate (mp-gap-size-y) at ($(0,\mpgapy)$);

        \coordinate (mp-rect-size-x) at ($(\mprectsizex,0)$);
        \coordinate (mp-tmp1) at ($(mp-top-left)-(mp-bottom-left)$);
        \coordinate (mp-rect-size-y) at (0, \mprectsizey);
        \coordinate (mp-rect-size) at ($(mp-rect-size-x)+(mp-rect-size-y)$);

        \def\mpcircmarginy{0.13}
        \def\mpcircgapy{0.13}
        \def\mpcircr{{(\mprectsizey-2*\mpcircmarginy-2*\mpcircgapy)/3/2}}

        \coordinate (mp-step-size-x) at ($(mp-gap-size-x)+(mp-rect-size-x)$);
        \coordinate (mp-step-size-y) at ($(mp-gap-size-y)+(mp-rect-size-y)$);
        
        \
        \coordinate (mp-circ-margin-y) at (0,\mpcircmarginy);
        \coordinate (mp-circ-gap-y) at (0, \mpcircgapy);
        \coordinate (mp-circr-y) at (0, \mpcircr);
        
        \coordinate (mp-circ-o) at ($0.5*(mp-rect-size-x)+ (mp-circ-margin-y)+(mp-circr-y)$);
        \coordinate (mp-circ-step) at ($(mp-circ-gap-y)+2*(mp-circr-y)$);

        \foreach \i in {0, 1} {
            \foreach \j in {0, 1, 2} {
                \coordinate (rect-corner) at ($(mp-bottom-left)+\i*(mp-step-size-x)+\j*(mp-step-size-y)$);
                \draw (rect-corner) rectangle +(mp-rect-size);
                \foreach \k in {0, 1, 2} {
                    \coordinate (circo) at ($(rect-corner)+(mp-circ-o)+\k*(mp-circ-step)$);
                    \draw (circo) circle (\mpcircr);
                }
            }
        }
        \coordinate (rect-corner1) at ($(mp-bottom-left)+0*(mp-step-size-x)+2*(mp-step-size-y)$);
        \coordinate (circo1) at ($(rect-corner1)+(mp-circ-o)+2*(mp-circ-step)$);

        \coordinate (rect-corner2) at ($(mp-bottom-left)+1*(mp-step-size-x)+2*(mp-step-size-y)$);
        \coordinate (circo2) at ($(rect-corner2)+(mp-circ-o)+2*(mp-circ-step)$);
        
        \coordinate (rect-corner3) at ($(mp-bottom-left)+1*(mp-step-size-x)+1*(mp-step-size-y)$);
        \coordinate (circo3) at ($(rect-corner3)+(mp-circ-o)+0*(mp-circ-step)$);
        
        \coordinate (mp2-tmp1) at (\mpcircr,0);
        \coordinate (mp2-tmp2) at (0,\mpcircr);
        
        \coordinate (mp-arrow1-p) at ($(circo1)+(mp2-tmp1)$);
        \coordinate (mp-arrow1-q) at ($(circo2)-(mp2-tmp1)$); 
        \draw[-latex] (mp-arrow1-p) -- (mp-arrow1-q)
        ;
        \node[anchor=south,font=\small] at ($(mp-arrow1-p)!0.5!(mp-arrow1-q)$) {$(a, b)$};
        
        \coordinate (mp-arrow2-p) at ($(circo1)+(mp2-tmp1)$);
        \coordinate (mp-arrow2-q) at ($(circo3)-0.70710678118654752440084436210485*(mp2-tmp1)+0.70710678118654752440084436210485*(mp2-tmp2)$); 
        \draw[-latex] (mp-arrow2-p) -- (mp-arrow2-q);
        
        \node[anchor=north,font=\small,rotate=-60] at ($(mp-arrow2-p)!0.57!(mp-arrow2-q)$) {$(a', b')$};
        
        \node[anchor=south] at (-1.75, 3.8) {The first pass};
        \node[anchor=south] at (-1.75, -0.2) {The second pass};
        \node[anchor=south] at (4, 3.8) {Modified second pass};
        \draw[-latex,thick,label={a}] (0,0.5) -- (2,0.5);
        \node[anchor=south] at (1, 0.5) {modification};
    \end{tikzpicture}
    \caption{Remembering the first pass.} \label{fig:remember}
\end{figure}

More formally, we modify the second pass (See \Cref{fig:remember}), so that every vertex is now a pair of an original first-pass vertex and an original second-pass vertex. The initial starting vertex of the second pass becomes $(v_0, v_1)$. When the program reads the first equation $(a_1, b_1)$, if in the first pass $v_0 \wt v'$ and in the second pass $v_1 \wt v$, in the modified second pass, $(v_0, v_1) \wt (v', v)$. For the exact details of this modification, please refer to \Cref{sec:two-pass-setup}. 

Now in the modified program, every modified vertex $v$ in the second pass corresponds to (``remembers'') a unique vertex $v'$ in the first pass. The event $v_1 \wt v$ now implies $v_0 \wt v'$, in the sense that for any input $x, a_1, a_2, \dots, a_T$ such that $v_1 \wt v$ happens, the event $v_0 \wt v'$ must also happen.

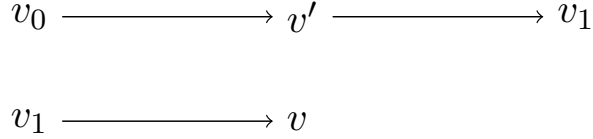
\begin{figure}[H]
    \centering
    \scalebox{1.4}{
    \begin{tikzpicture}
  \node (v0) {$v_0$};
  \node[right=2cm of v0] (v') {$v'$}; 
  \node[right=2cm of v'] (v1) {$v_1$};
  \node[below=0.5cm of v0] (v1_second) {$v_1$}; 
  \node[right=2cm of v1_second] (v) {$v$};

  \draw[->] (v0) -- (v');
  \draw[->] (v') -- (v1);
  \draw[->] (v1_second) -- (v);
\end{tikzpicture}}
    \caption{The computational path of two passes.} \label{fig:two-pass-path}
\end{figure}

Furthermore, we require every vertex in the second pass to remember $v_1$, the starting vertex of the second pass. This can be achieved with a similar modification. By this modification, for any vertex $v$ in the second pass (which remembers its corresponding vertex $v'$ in the first pass), 
\begin{align*}
(v_0 \wt v) &= (v_0 \wt v') \land (v' \wt v_1) \land (v_1 \wt v) \\ &= (v_1 \wt v) \land (v' \wt v_1).
\end{align*}
Here the last equality holds because $v_1 \wt v$ implies $v_0 \wt v'$, and $v_1$ is unique since $v$ remembers it. 

When $v = v_2$ for some vertex $v_2$ in the last layer (of the second pass), this simplifies to 
\begin{align*}
(v_0 \wt v_2) &= (v_0 \wt v_1) \land (v_1 \wt v_2) \\
&= (v_1 \wt v_2).
\end{align*}

\paragraph*{Progress Measure.} When the program reaches $v_2$, the optimal strategy for it is to output the $x' \in X$ with the highest $\doubleP_{x \mid v_0 \wt v_2}(x')$. Note that this equals $\doubleP_{x \mid v_1 \wt v_2}(x')$ (by the equation above). 

Hence, similar to the one-pass case, when the distribution $\doubleP_{x \mid v_1 \wt v_2}$ is spread enough ($\|\doubleP_{x \mid v_1 \wt v_2}\|_2 \geq 2^{\epsilon n} \cdot 2^{-n}$), vertex $v_2$ cannot answer $x$ correctly. To upper bound the probability of answering $x$ correctly, we only need to upper bound the probability of reaching any target state $t$ with $\|\doubleP_{x \mid v_1 \wt t}\|_2 \geq 2^{\epsilon n} \cdot 2^{-n}$. \\

Initially, $\doubleP_{x \mid v_1 \wt v_1} = \doubleP_x$. As the current vertex $v$ moves along the computational path from $v_1$, the posterior $\doubleP_{x \mid v_1 \wt v}$ evolves similarly to the one-pass case. Let the similarity $\langle \doubleP_{x \mid v_1 \wt v}, \doubleP_{x \mid v_1 \wt t}\rangle$ be the progress measure, and let $\doubleP^{(a,b)}_{x \mid v_1 \wt v}$ be the posterior after reading a new equation $(a,b)$. We have,
$$\left\langle \doubleP^{(a,b)}_{x \mid v_1 \wt v}, \doubleP_{x\mid v_1 \wt t}\right\rangle = 
\frac{1}{|X|}\sum_{\substack{x' \in X \\ M(a,x') = b}} \doubleP_{x\mid v_1 \wt v}(x') \cdot \doubleP_{x \mid v_1 \wt t}(x') \Bigg/ \sum_{\substack{x' \in X \\ M(a,x') = b}} \doubleP_{x\mid v_1 \wt v}(x').$$

Until now, this seems like a natural generalization of the one-pass case. However, for one pass, we heavily rely on the fact that $a \in A$ is uniformly random. In the second pass, we no longer have this property. For example, the program could simply remember $a_1 \in A$ from the first pass, then in the second pass, $a_1$ is completely deterministic, with no randomness at all.

\paragraph*{High-Probability Edges.} The previous work \cite{garg2019time} calls such $a_i \in A$ that is remembered by the program a \emph{high-probability edge}. Formally, for a vertex $v$ in the $i$-th layer of the second pass, we say that $a$ is a high-probability edge at $v$ (denoted by $a \in \High(v)$) if and only if 
$$\Pr[a_{i + 1} = a \mid v_0 \wt v] \geq 2^{\epsilon \cdot n} \cdot 2^{- n}.$$
Since these edges occur with too large probability (much higher than the uniform case, $2^{-n}$), we cannot simply stop when they cut distributions unevenly (like we did for one-pass).

\begin{enumerate}
    \item If this $a \in A$ cuts $\doubleP_{x \mid v_1 \wt v}$ unevenly: When the denominator $\sum_{x' \in X, M(a, x') = b} \doubleP_{x \mid v_1 \wt v}(x') < 2^{-c}$, the similarity might be larger by a factor of $2^c$, causing huge progress. For now, we ignore this issue and assume that it never arises. We will explain how we handle it in \Cref{subsec:new-counter}. 
    \item If this $a \in A$ cuts $\doubleP_{x \mid v_1 \wt v} \cdot \doubleP_{v_1 \wt t}$ unevenly, but still cuts $\doubleP_{x \mid v_1 \wt v}$ evenly: In this case, the worst case is the same as in one-pass. Namely, the numerator does not decrease, while the denominator is halved. Then the similarity at most doubles. \label{Item:High-prob-edge-Item-2}
    \end{enumerate}

Their key observation is the following. Intuitively, to remember a single $a \in A$, we need at least $\Omega(n)$ memory. Since the memory bound $S \leq n^2 / 16$, the program can only remember $O(n)$ many such $a$'s. So there can be at most $O(n)$ high probability edges. 

Hence, to handle the case in \Cref{Item:High-prob-edge-Item-2}, we observe that these $O(n)$ many high-probability edges only blow up the similarity by $2^{O(n)}$. Since initially similarity $\langle \doubleP_{x}, \doubleP_{x \mid v_1 \wt t}\rangle = 2^{-2n}$, and we want to prove that it would increase to $\|\doubleP_{x \mid v_1 \wt t}\|_2^2 \geq 2^{2 \epsilon n} \cdot 2^{-2n}$ with very small probability. As long as the constant hidden by big~$O$ is much smaller than $\epsilon$, this blow-up is negligible. We will pick the correct constants to ensure that this is indeed the case.

\subsection{New Ingredient: Bias Counters} \label{subsec:new-counter}

As mentioned in the first case, if a high probability edge $a$ cuts $\doubleP_{x \mid v_1 \wt v}$ unevenly, the similarity might grow a lot. First, we explain how the previous work \cite{garg2019time} gets around this issue. Then, we will introduce our new idea. This is the key idea for proving the tight memory lower bound for two passes. 

\paragraph*{Very-bad edge.} To get around this issue, they defined ``very-bad edges'', which is the high probability edges $a$ that cut $\doubleP_{x \mid v_1 \wt v}$ in a very biased way. Formally, $a \in \High(v)$ is ``very-bad'' if $$\sum_{\substack{x' \in X \\ M(a,x') = b}} \doubleP_{x \mid v_1 \wt v} (x') \leq 2^{-\sqrt{n}}.$$ 
\begin{itemize}
    \item On the one hand, when a high probability edge is not very-bad, it only blows up the similarity at most by a factor of $2^{\sqrt{n}}$. 
    
    If we set the memory bound $S$ to be $O(n^{3/2})$, there can only be $c \cdot \sqrt{n}$ many high probability edges for some constant $0 < c < 1$ (since remembering one needs $\Omega(n)$ memory). Then these $c \cdot \sqrt{n}$ many high probability not-very-bad edges can blow up the potential by at most $2^{c n}$. As long as $c \ll \epsilon$, this will be acceptable for our purpose.
    \item On the other hand, when an edge $(a,b)$ is very-bad, conditioning on we reached this vertex, $x$ is distributed as $\doubleP_{x \mid v_1 \wt v}$. Over the randomness of $x$, the probability that we traverse this edge $(a,b)$ instead of $(a,-b)$ is at most $2^{-\sqrt{n}}$.

    First of all, note this is completely independent of the target $t$. So we do not have to union bound over $t$. Then if we set the sample bound $T$ to be $2^{O(\sqrt{n})}$, we can stop immediately whenever we meet a very bad edge. For each step, we stop with probability $2^{-\sqrt{n}}$. By union-bound over all $T$ steps, we conclude that the overall stopping probability is still $2^{-\Omega(\sqrt{n})}$. 
\end{itemize}

Therefore, they can prove that for some constant $\epsilon > 0$, any two-pass algorithm with $S \leq \epsilon n^{3/2}$ memory and $T \leq 2^{\epsilon \sqrt{n}}$ samples cannot succeed with constant probability. 

\paragraph*{New Idea: Bias Counter.} Instead of stopping immediately at biased ``very-bad edges'', we introduce a counter $\cb$ to keep track of the accumulated biases: In the second pass, initially, when we were at the starting vertex $v_1$, we let $\cb \gets 0$. For any high probability edge $(a,b)$ from current vertex $v$ (satisfying $a \in \High(v)$), we say it is $\Delta$-biased if 
$$\sum_{\substack{x' \in X \\ M(a,x') = b}} \doubleP_{x \mid v_1 \wt v} (x') \in [2^{-\Delta-1}, 2^{-\Delta}).$$ 
Whenever we traverse a $\Delta$-biased edge, 
we will update our counter (roughly) by $$\cb \gets \cb + \Delta.$$ 
Note a $\Delta$-biased edge has to be a high probability edge. There can be at most $O(n)$ high probability edges when $S \leq n^2 / 16$. Hence we will make at most $O(n)$ such updates.

\begin{itemize}
    \item On one hand, when the counter $\cb \leq \epsilon \cdot n$, the high probability edges we have traversed can blow up the similarity by at most $2^{\epsilon n}$.  

    This follows almost directly from the definition of $\Delta$: We increase the counter by $\Delta$ if and only if we traversed a $\Delta$-biased edge. As we have discussed, such an edge would blow up the similarity by at most $2^{\Delta}$. Hence in total, they can blow up the similarity by at most $2^{\cb}$.
    \item On the other hand, the overall probability that $\cb > \epsilon \cdot n$ is small. If we stop whenever the counter exceeds the threshold $\epsilon \cdot n$, we can show the overall stopping probability will be small.
    
    This is because we traverse a $\Delta$-biased edge $(a,b)$ instead of the other edge $(a, -b)$ with probability at most $2^{-\Delta}$ over the randomness of $x$. To gain some intuition, let us think about two extreme cases: 
    \begin{itemize}
    \item Case 1: Each time, the counter increases a little, e.g., $\Delta \approx 100$. In this case, w.p. $1 - (2^{-100})^{(\epsilon/100)\cdot n} = 1-2^{-\epsilon n}$, the counter $\cb$ is going to increase like this for less than $(\epsilon/100) \cdot n$ times. 

    (For $\Delta \ll 100$, since there are at most $O(n)$ updates, as long as the constant hiding by this big-$O$ is much smaller than $\epsilon$, we can ignore these updates, as the total increase due to them is negligible compared to $\epsilon \cdot n$.)
    
    \item Case 2: Each time, the counter increases a lot, i.e., $\Delta = \delta n$ for $0 < \delta < 1$. In this case, w.p. $1 - (2^{-\delta n})^{(\epsilon/\delta)}=1-2^{-\epsilon n}$, the counter $\cb$ will increase like this for less than $\frac{\epsilon}{\delta}$ times. 
    \end{itemize}

In both cases, the counter $\cb$ will not overflow with high probability. To make $\cb$ larger than $\epsilon \cdot n$, it is necessary to have a sufficiently large number of (correspondingly) sufficiently large $\Delta$, and this is very unlikely. This resembles a famous quote:

\begin{displayquote}
\textit{``You can fool some of the people all of the time, and all of the people some of the time, but you can not fool all of the people all of the time.''--- Abraham Lincoln}
\end{displayquote}

Formally, we can show that for each layer $i \in [T]$, the counter overflows at layer $i$ with probability at most $2^{-\epsilon n}$. Together with a union bound over $T \ll 2^{\epsilon n}$ layers, we prove that the overall stopping probability for counter overflow is small.

\end{itemize}

This allows us to prove that for some constant $\epsilon' > 0$, any two-pass algorithm with $S \leq \epsilon' n^2$ memory and $T \leq 2^{\epsilon' n}$ samples cannot success with constant probability. This new idea is the key to proving a tight lower bound for two-pass learning.

In our actual proof, we will modify the program so that each vertex $v$ will ``remember'' a unique counter value $\cb(v)$. More details about this modification, the bias counter, and related stopping rules will be presented in \Cref{sec:two-pass-setup}.

\paragraph*{Potential Argument.} To implement the idea, we will introduce a new stopping rule: We stop whenever the counter $\cb$ exceeds $\epsilon \cdot n$. We have to prove that overall, we stop due to this new rule with a small probability. 

Unlike previous stopping rules, this new stopping rule is ``soft'', in the sense that when a rare event of probability $2^{-\Delta}$ happens, it does not stop right away. Instead, it accumulates such rare events and only stops when enough rare events have occurred. Previously, we only need to analyze stopping probability based on the randomness at the current step. But now, we have to look at the computation history and exploit the fact that we stop only when a lot of rare events have happened in history.

This makes it harder to analyze the stopping probability. Our main technical contribution to two-pass is to come up with a potential analysis resolving the issue. Roughly speaking, our potential function is defined as $\Phi \approx 2^{\cb}.$ Initially, the expectation of $\Phi$ is $1$ (as we have $\cb = 0$ at the starting vertex). Since in each step, $\cb$ increases by $\Delta$ with probability roughly $2^{-\Delta}$, the expectation of $\Phi$ (of the current state $v$) will be (almost) non-increasing. At the end of the computation, we can use Markov's inequality to bound the probability of $\Pr[\cb > \epsilon \cdot n]$ by $\frac{1}{2^{\epsilon n}}$.

We will give a more detailed overview and more intuitions at \Cref{subsec:potental-def-overview}.

\subsection{Transfer Lemma} \label{sec:overview-transfer}

However, the proof still has the last important missing piece. In \Cref{subsec:new-counter}, we argued intuitively that for a $\Delta$-biased edge $(a,b)$, with only $2^{-\Delta}$ probability over the randomness of $x$, we traverse $(a,b)$ instead of $(a,-b)$. But there is one important subtlety.

\paragraph{Informal discussion.} Note we defined the $\Delta$-biased edges w.r.t.~the posterior distribution $\doubleP_{x \mid v_1 \wt v}$. So we actually showed is the following:
\begin{quote}
For all starting vertex $v_1$ of the second pass, after $v_1 \wt v$, the computational path will then traverse a $\Delta$-biased edge from $v$ w.p. at most $2^{-\Delta}$. 
\end{quote}
But ideally, we want to argue about the two-pass branching program, starting from $v_0$. The ideal statement will be:
\begin{quote}
For the starting vertex $v_0$ of the first pass, after $v_0 \wt v$, the computational path will then traverse a $\Delta$-biased edge from $v$ w.p. at most $2^{-\Delta}$. 
\end{quote}

If $v = v_2$ is a vertex in the last layer of the second pass, as explained in \Cref{subsec:framwork-twopass}, we know that $v_0 \wt v_2$ is equivalent to $v_1 \wt v_2$. This is because the second pass remembers the first pass, so $v_1 \wt v_2$ can ``certify'' that the first pass indeed reaches $v_1$. 

But for stopping rules, it is crucial that we stop or update counters in the middle of the program. So $v$ is not always going to be in the last layer. Suppose $v$ is in the middle of the second layer and $v'$ is its corresponding vertex in the first layer. As explained in \Cref{subsec:framwork-twopass}, we have 
$$(v_0 \wt v) = (v_1 \wt v) \land (v' \wt v_1).$$
Hence the statement we showed differs from the ideal one. Intuitively, this is because $v_1 \wt v$ can only certify that $v_0 \wt v'$. We need extra arguments for controlling $v' \wt v_1$. This is why we need the transfer lemma, to transfer the statement we showed into the ideal statement we want.

This is the most technical part of \cite{garg2019time}. In \Cref{sec:two-pass-transfer}, we will frame this subtlety as an adaptive issue and give a slightly simplified proof together with a clean interpretation. This perspective enables us to generalize this lemma to multiple passes.

\paragraph*{Distribution Mismatch.} To explain this subtlety in full clarity, let us carefully examine the definition of being $\Delta$-biased. We say that a high probability edge $(a,b)$ is $\Delta$-biased if,
$$\sum_{\substack{x' \in X \\ M(a,x') = b}} \doubleP_{x \mid v_1 \wt v} (x') \leq [2^{-\Delta-1}, 2^{-\Delta}).$$
Note that the vector $x$ is sampled conditioning on the event $v_1 \wt v$. Hence it corresponds to the following random process:
\begin{itemize}
    \item First, sample a uniformly random $x \in X$ and let $v = v_1$.
    \item At the $i$-th step uniformly sample $a_i \in A$ and move $v$ along the edge $(a_i, M(a_i, x))$. 
\end{itemize}

The definition of $\Delta$-biased edges is saying that conditioning on this process reaches $v$, the probability that $M(a,x) = b$ is at most $2^{-\Delta}$. However, when we talk about our overall stopping probability, we are referring to a different random process:
\begin{itemize}
    \item First, sample a random $x \sim \doubleP_{x \mid v_0 \to v_1}$ and let $v = v_1$.
    \item At the $i$-th step uniformly sample $a_i \sim \Pr[a_i \mid v_0 \to v]$ and move $v$ along edge $(a_i, M(a_i, x))$. 
\end{itemize}

We can see there is a distribution mismatch between these two. Conditioning on we reached $v$, in the first process, the posterior distribution of $x$ is $\doubleP_{x \mid v_1 \wt v}$, while in the second process, the posterior distribution of $x$ is $\doubleP_{x \mid v_0 \wt v}$. As explained in \Cref{subsec:framwork-twopass}, in general for a vertex $v$, $(v_0 \wt v) = (v_1 \wt v) \land (v' \wt v_1)$. The posterior distribution at the end of the second process $$\doubleP_{x \mid v_0 \wt v} = \doubleP_{x \mid (v_1 \wt v) \land (v' \wt v_1)}$$  which is not only conditioning on (1) the event that the path from $v_1$ reaches $v$ but also on (2) the event that in the first pass, the path from $v'$ picks $v_1$ as the end vertex. 

The issue is that we ultimately want to prove that the stopping probability is small in the second process $v_0 \wt v$. But as $\Delta$-bias has been defined w.r.t. the first process, our previous argument only shows that the stopping probability is small in the first process $v_1 \wt v$. 

\paragraph*{Transfer Lemma.} To resolve this, \cite{garg2019time} introduced the following lemma, ``transferring'' an upper bound on the stopping probability of the first process $v_1 \wt v$ to an upper bound on the stopping probability of the second process $v_0 \wt v$. 

\paragraph*{Informal Statement:} For any two-pass algorithm, suppose for all fixed starting vertex $v_1$, the computational path from $v_1$ stops with probability less than $2^{-c n}$ for some $0 < c < 1$. The computational path from $v_0$ will stop in the second pass with probability less than $2^{-(c -\epsilon) n}$ for some $0 < \epsilon < 1$. 

\paragraph*{Our Perspective.} Note the assumption of this lemma says that \emph{for all fixed $v_1$}, over the randomness of $x$ and $a_1, a_2, \dots, a_T$, the probability of stopping is small. The reason there is this distribution mismatch is that the vertex $v_1$ is not fixed beforehand but picked by the first pass adaptively (depending on $x, a_1, a_2, \dots, a_T$ as well). Hence this is really an issue caused by adaptivity. 

We will then interpret the proof as exploiting the fact that the first pass (a memory-bounded one-pass algorithm) has a limited ability to adaptively choose the worst $v_1$. This perspective, which seems missing from the previous works, is crucial for us to generalize this lemma to multiple passes. A more detailed overview will be given in \Cref{sec:two-pass-transfer}.

\subsection{Extending to Multiple Passes}

To extend this to multiple passes, we first adapt our modification on the program and stop rules beyond two passes. This turns out to be simple, and the details will be presented in \Cref{sec:multi-pass-setup}. Our new counter and its potential analysis extend to multi-pass smoothly as well. This will be included in \Cref{sec:mult-pass-potential}.

\paragraph*{Main Challenge: Transfer Lemma.} However, for the transfer lemma, the original proof crucially relies on the fact that the algorithm has only two passes. Subtle technical difficulties arise when generalizing it to multiple passes. We will give detailed intuitions and discussions in \Cref{sec:mutli-pass-transfer}.  

\paragraph*{Current Bottleneck: Transfer Lemma.} Currently, our lower bound stops at $O(\log \log n)$ many passes. This is because of the transfer lemma. To prove a lower bound for $q$-passes, we will use our result for $(q-1)$-passes in a black-box way. This implicitly applies the transfer lemma for $q - 1$ times. It turns out that each application of the transfer lemma is quite costly: In the informal statement in \Cref{sec:overview-transfer}, the original $2^{-cn}$ bound on stopping probability is demoted to a $2^{-(c - \epsilon)n}$ bound after applying transfer lemma. Roughly speaking, if $\epsilon > 0$ is a constant, we can at most apply this lemma a constant number of times. Then we can only prove a lower bound for a fixed constant number of passes. 

By carefully choosing parameters, we can prove that for any algorithm with $q$ passes, to succeed with constant probability, it necessarily requires either $\Omega\!\left( \frac{n^2}{c_q} \right)$ memory or $2^{\Omega\left(n/c_q\right)}$ samples where $c_q = 100^{3^q}$. In particular, this implies a $n^{2-o(1)}$ memory lower bound for $o(\log\log n)$-pass learning.
The bottleneck in improving these results beyond $O(\log \log n)$ passes is the successive application of the transfer lemma. It remains an interesting open problem whether a better tradeoff such as $c_q = \exp(q)$ or even $c_q = \poly(q)$ can be proved. For example: can a $n^{o(1)}$-pass algorithm learn the hidden vector $x$ using $n^{2-\Omega(1)}$ memory and $2^{n^{1-\Omega(1)}}$ samples?

\subsection{Organization of the Proofs}

\subsubsection{Two-Pass Learning Lower bound}

\paragraph*{Setup.} In our analysis, we will always assume that the second pass program remembers the first pass and counters. In \Cref{sec:modify}, we will prove that this is without loss of generality by showing that we can modify any two-pass branching program into such a program. In the rest of \Cref{sec:two-pass-setup}, we will introduce the stopping rules formally.

\paragraph*{Analysis.} First, we need to show that the probability of stopping is small. Such analysis will be divided into the following sections in \Cref{fig:two-pass-road-map}. Specifically, for two stopping rules, counter overflow and significant value, we will need the transfer lemma in \Cref{sec:two-pass-transfer}.

\begin{figure}[H]
    \centering
    \begin{tikzpicture}
        \def\ygapa{1.5}
        \def\ygapb{1}
        \def\ygapc{0.5}
        \def\ygapd{0.5}
        \def\ygape{0.5}
        \def\xgapa{4.4}
        \def\xgapb{1.7}
        \coordinate (node1) at ({-\xgapb-\xgapa},0);
        \coordinate (node2) at ({-\xgapb}, {-\ygapa});
        \coordinate (node3) at (0, {-\ygapa-\ygapb});
        \coordinate (node4) at (0, {-\ygapa-\ygapb-\ygapc});
        \coordinate (node5) at (0, {-\ygapa-\ygapb-\ygapc-\ygapd});
        \coordinate (node6) at (0,{-\ygapa-\ygapb-\ygapc-\ygapd-\ygape});
        \node[anchor=west] at (node2) {Bias counter overflow (\Cref{sec:two-pass-potential})};
        \node[anchor=west] at (node3) {Significant value};
        \node[anchor=west] at (node4) {High-probability edge overflow};
        \node[anchor=west] at (node5) {Bad edge but not high-probability }; 
        \node[anchor=west] at (node6) {Significant state (\Cref{appendix:sig-state-two-pass})};
        \node[anchor=east,rotate=-90] at ($(node6)-(0.1,0.35)$){$\underbrace{\hspace{1.9cm}}$};
        \coordinate (node7) at ({-\xgapb},{-\ygapa-\ygapb-(\ygapc+\ygapd+\ygape)/2});
        \node[anchor=west] at (node7) {Others};
        \node[anchor=east,rotate=-90] at ($(node7)-(0.1,0.35)$){$\underbrace{\hspace{2.15cm}}$};
        
        \coordinate (node8) at ({-\xgapa}, {-\ygapa-\ygapb/2-\ygapc/4-\ygapd/4-\ygape/4});
        \node[anchor=west] at (node8) {$\mathbf{Pr}[\text{Stop}] \ll 1$};

        \def\xrightgap{4.5}

        \coordinate (node9) at ($(node3)+(\xrightgap,0)$);
        \node[anchor=west,rotate=90] at ($(node9)-(0.1,0.25)$){$\underbrace{\hspace{1.3cm}}$};
        \coordinate (node10) at ({\xrightgap}, {-\ygapa-\ygapb/2});
        \node[anchor=west] at (node10) {Transfer lemma (\Cref{sec:two-pass-transfer})};

        \def\xrightgapb{6.5}

        \coordinate (node11) at ($(node5)+(\xrightgapb,0)$);
        \node[anchor=west,rotate=90] at ($(node11)-(0.1,0.25)$){$\underbrace{\hspace{0.8cm}}$};
        \coordinate (node12) at ({\xrightgapb}, {-\ygapa-\ygapb-\ygapc-\ygapd/2});
        \node[anchor=west] at (node12) {\Cref{sec:analysis}};

    \end{tikzpicture}
    \caption{Upper-bound the stopping probability for two passes.}
    \label{fig:two-pass-road-map}
\end{figure}

Second, if the branching program does not stop (specifically, it does not reach any significant state), we will show that it is going to correctly guess $x$ with small probability in \Cref{sec:two-pass-wrap-up}.

\subsubsection{Multi-Pass Learning Lower bound}

\paragraph*{Setup.} For multi-pass, we will w.l.o.g. assume that each pass of the branching program is going to remember all the previous passes and counters. We will present corresponding modification in \Cref{sec:multi-pass-modify}. In the rest of \Cref{sec:multi-pass-setup}, we will present the stopping rules. 

\paragraph*{Analysis.} First, we show that the probability of stopping is small. Such analysis will be divided into the following sections in \Cref{fig:multi-pass-road-map}.

\begin{figure}[H]
    \centering    \begin{tikzpicture}
        \def\ygapa{1.5}
        \def\ygapb{1}
        \def\ygapc{0.5}
        \def\ygapd{0.5}
        \def\ygape{0.5}
        \def\xgapa{4.4}
        \def\xgapb{1.7}
        \coordinate (node1) at ({-\xgapb-\xgapa},0);
        \coordinate (node2) at ({-\xgapb}, {-\ygapa});
        \coordinate (node3) at (0, {-\ygapa-\ygapb});
        \coordinate (node4) at (0, {-\ygapa-\ygapb-\ygapc});
        \coordinate (node5) at (0, {-\ygapa-\ygapb-\ygapc-\ygapd});
        \coordinate (node6) at (0,{-\ygapa-\ygapb-\ygapc-\ygapd-\ygape});
        \node[anchor=west] at (node2) {Bias counter overflow (\Cref{sec:mult-pass-potential})};
        \node[anchor=west] at (node3) {Significant value};
        \node[anchor=west] at (node4) {High-probability edge overflow};
        \node[anchor=west] at (node5) {Bad edge but not high-probability }; 
        \node[anchor=west] at (node6) {Significant state (\Cref{appendix:sig-state-two-pass})};
        \node[anchor=east,rotate=-90] at ($(node6)-(0.1,0.35)$){$\underbrace{\hspace{1.9cm}}$};
        \coordinate (node7) at ({-\xgapb},{-\ygapa-\ygapb-(\ygapc+\ygapd+\ygape)/2});
        \node[anchor=west] at (node7) {Others};
        \node[anchor=east,rotate=-90] at ($(node7)-(0.1,0.35)$){$\underbrace{\hspace{2.15cm}}$};
        
        \coordinate (node8) at ({-\xgapa}, {-\ygapa-\ygapb/2-\ygapc/4-\ygapd/4-\ygape/4});
        \node[anchor=west] at (node8) {$\mathbf{Pr}[\text{Stop}] \ll 1$};

        \def\xrightgap{4.5}

        \coordinate (node9) at ($(node3)+(\xrightgap,0)$);
        \node[anchor=west,rotate=90] at ($(node9)-(0.1,0.25)$){$\underbrace{\hspace{1.3cm}}$};
        \coordinate (node10) at ({\xrightgap}, {-\ygapa-\ygapb/2});
        \node[anchor=west] at (node10) {Transfer lemma (\Cref{sec:mutli-pass-transfer})};

        \def\xrightgapb{6.5}

        \coordinate (node11) at ($(node5)+(\xrightgapb,0)$);
        \node[anchor=west,rotate=90] at ($(node11)-(0.1,0.25)$){$\underbrace{\hspace{0.8cm}}$};
        \coordinate (node12) at ({\xrightgapb}, {-\ygapa-\ygapb-\ygapc-\ygapd/2});
        \node[anchor=west] at (node12) {\Cref{sec:multi-pass-proof-of-main}};

    \end{tikzpicture}
    \caption{Upper-bound the stopping probability for multiple passes.}
    \label{fig:multi-pass-road-map}
\end{figure}
Second, if the branching program does not stop, we will show that it is going to correctly guess $x$ with small probability in \Cref{sec:multi-pass-wrap-up}.

\section{Preliminaries}
\subsection{Notation}

Let $(x,y,z)$ be a joint random variable. We use $x\perp y$ to denote that $x$ and $y$ are \emph{independent}. $x\perp y \mid z = z'$ denotes that $x$ and $y$ are independent, \emph{conditioned} on $z=z'$. 

$\mathbb{R}^{+}$ denotes the set of non-negative real numbers. All the logarithms will have base $2$. 

For a list of items $a = (a_1,a_2,\dots, a_T)$, we use $a_{\le i}$ (resp.~$a_{>i}$) to define the length-$i$ prefix (resp.~length-$(n-i)$ suffix) of $a$.

\subsection{Norms and Inner Products}

In this paper, we will deal with functions of the form $f:X \to \doubleR$. Given $p\ge 1$, define the $\ell_p$-norm of $f$ as
\[
\| f\|_p = \left(\Ex_{x\sim X} [f(x)^p] \right)^{1/p}.
\]

Let $f,g$ be two functions. We define $\langle f,g\rangle := \Ex_{x\sim X}[f(x)\cdot g(x)]$ as the inner product between $f$ and $g$. We collect some basic facts about norms in the following.

\begin{itemize}
\item Monotonicity: $\| f\|_p \le \|f \|_q$ for every $p\le q$.
\item Cauchy-Schwarz inequality: $\langle f, g\rangle \le \| f \|_2 \cdot \|g \|_2$
\item (Special case of) H\"older's inequality: $\|f g\|_p \le \|f\|_p \cdot \|g\|_\infty$.
\item Truncation trick: Suppose $\|f\|_2 \le A$. Let $f^{>B}(x) = \mathbbm{1}_{f(x)>B}\cdot f(x)$. Then 
\[
\|f^{>B}\|_1 := \Ex_{x\sim X}[f(x)\cdot \mathbbm{1}_{f(x)>B}] \le \Ex_{x\sim X}\left[ \frac{f(x)^2}{B} \right] \le \frac{A^2}{B}.
\]
\end{itemize}

\subsection{Learning Tasks and $L2$-Extractors}

Following~\cite{GargRT18-extractor}, we use a matrix $M\in \{-1,1\}^{A\times X}$ to describe a learning task. Here, $X$ is the domain of concepts, and $A$ is the domain of possible inputs. $M$ corresponds to the following learning problem: An unknown element $x\in X$ was chosen uniformly at random. The learner tries to learn $x$ from a stream of samples, $(a_1,b_1),(a_2,b_2),\dots$, where for each $i$, $a_i\in A$ is uniformly at random chosen and $b_i = M(a_i, x)$. In particular, we consider the setting where the learner is allowed to see the same stream of samples for $q\ge 2$ passes.

\begin{definition}\label{def:L2_ext}
We say a learning matrix $M:A\times X\to \{-1,1\}$ a $(\kext,\ellext,\rext)$-$L2$-Extractor, if for every non-negative $f\colon X\to \mathbb{R}$ with $\frac{\|f\|_2}{\|f\|_1} \le 2^{\ellext}$, there are at most $2^{-\kext}\cdot |A|$ rows $a\in A$ with 
\[
| \langle M_a, f\rangle | \ge 2^{-\rext} \cdot \| f\|_1,
\]
where $M_a\in \{-1,1\}^X$ denotes the restriction of $M$ to the $a$-th row.
\end{definition}

Below, we collect a few examples of learning tasks from \cite{GargRT18-extractor}, whose matrices are good $L2$-extractors. For a more comprehensive list, see \cite{GargRT18-extractor}.
\begin{itemize}
    \item \emph{Parities.} In parity learning, one identifies $X = \{0,1\}^n$ and $A = \{0,1\}^{n}$. Then, $M(a,x) = (-1)^{\langle a, x\rangle}$ for every $a,x$. It can be shown that $M$ is an $L2$-extractor with
    \[
    \kext = \Omega(n), \quad \ellext = \Omega(n), \quad \rext = \Omega(n).
    \]
    \item \emph{Sparse Parities.} In sparse parity learning,
    the hidden vector $x$ is known to have Hamming weight exactly $\ell$.
    %
    We still identify $A = \{0,1\}^n$. But now $X = \{x\in \{0,1\}^n \mid \sum_{i}x_i = \ell\}$. The learning matrix $M(a,x) = (-1)^{\langle a, x\rangle}$ is an $L2$-extractor with 
    \[
    \kext = \Omega(n), \quad \ellext,  \rext = \Omega(\ell) \qquad\text{Assuming $\ell \le n/2$.}
    \]
    \[
    \kext = \Omega(n/\ell^{0.1}), \quad \ellext, \rext = \Omega(\ell \log(\ell)) \qquad\text{Assuming $\ell \le n^{0.9}$.}
    \]
    Indeed, the problem requires either superlinear samples or superpolynomial memory for any $\ell =\omega( \log(n)/\log\log(n))$~\cite{KRT17}.
    \item \emph{Random matrix.} Consider a random learning matrix $M\in \{\pm 1\}^{A\times X}$ with $|A|=|X|=2^n$, where every entry of $M$ is chosen from $\{-1,1\}$ uniformly at random. With high probability, the resulting matrix $M$ is an $L2$-extractor with 
    \[
    \kext = \Omega(n), \quad \ellext = \Omega(n), \quad \rext = \Omega(n).
    \]
    That means, a ``random'' learning task will be subject to our lower bound with high probability.
\end{itemize}

\subsection{Computational Model}

In this paper, we will model a $q$-pass learning algorithm as a $q$-pass ordered branching program, which can see exactly the same input stream $q$ times in the same order. 

\begin{definition}[$q$-pass branching program for learning]
A $q$-pass (ordered) branching program for learning, with  length $T$ and width $d$, is a directed (muti) graph with vertices arranged in $nT + 1$ layers, and each layer contains at most $d$ vertices. In the first layer which we think of as layer $0$, we fix the vertex $0$ to be the starting vertex, denoted by $v_0$. The $j$-th pass vertices are those vertices from layer $(j - 1)T$ to layer $jT + 1$. We will denote these $j$-th pass layers using vertex sets $V^{(j)}_0, V^{(j)}_1, \dots, V^{(j)}_T$. Note for all $j > 1$, $V^{(j)}_0 = V^{(j - 1)}_T$. 

Every vertex except the last layer has $2^{n + 1}$ many outgoing edges. Each edge is labeled by a unique pair $(a,b) \in A \times \{-1, 1\}$.  Every last-layer vertex is labeled by one element $\tilde{x}(v) \in X$, which is the output of the program when reaching that vertex. 
\end{definition}

\begin{definition}[Computational Path] Initially, the current vertex is the start vertex , $v \gets v_0$. Given the input $(a_1, b_1), (a_2, b_2), \dots, (a_T, b_T) \in A \times \{-1, 1\}$, in the $i$-th step ($i = 1, 2, \dots, T$),
the program reads $(a_i, b_i)$, move from its current vertex $v$ in $V^{(j)}_{i - 1}$ along the edge with that label, and arrive at a vertex in $V^{(j)}_i$. Once $v \in V_T^{(q)}$, the program will stop and output $\tilde{x}(v)$. This defines a computational path.
\end{definition}

\begin{definition}[Success Probability]
The success probability of a program is defined as 
$$\Pr_{\substack{a_1, a_2, \dots, a_T \sim A, \ x \sim X \\  \forall i \in [T] \ b_i = M(a_i, x)}}\left[\text{The output of the program matches $x$}\right]$$
where $x$ and $a_1, a_2, \dots, a_T$ are sampled uniformly at random. 
\end{definition}






\section{Setup for Two Passes} \label{sec:two-pass-setup}
\subsection{Modifying the Program} \label{sec:modify}

To ease our analysis, we would like to have the following modifications to the program. The modification will be a two-stage process. We will use $B_0$ to denote the original program, use $B_1$ to denote the program after the first stage, and use $B_2$ to denote the final program.

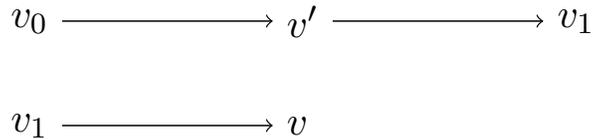
\begin{figure}[H]
    \centering
    \scalebox{1.4}{
    \begin{tikzpicture}
  \node (v0) {$v_0$};
  \node[right=2cm of v0] (v') {$v'$}; 
  \node[right=2cm of v'] (v1) {$v_1$};
  \node[below=0.5cm of v0] (v1_second) {$v_1$}; 
  \node[right=2cm of v1_second] (v) {$v$};

  \draw[->] (v0) -- (v');
  \draw[->] (v') -- (v1);
  \draw[->] (v1_second) -- (v);
\end{tikzpicture}}
    \caption{The computational path of two passes.}
\end{figure}

\paragraph*{Stage 1: Remember the First Pass.} We keep the first pass of the program unchanged. For the second pass of the program, we will now force it to remember the last state of the first pass, which we denote as $v_1 \in V^{(1)}_T$. (This is also the starting vertex of the second pass because $V^{(1)}_T = V^{(2)}_0$.) Besides this, the second pass program also runs a copy of the first pass in its memory. Namely, we modify the second pass so that now each vertex in layer $i$ is a triple $(v', v_1, v_{B_0}) \in V^{(1)}_i \times V^{(1)}_T \times V^{(2)}_i$, which means:
\begin{itemize}
    \item The program remembers that the first pass leads to the vertex $v_1$.
    \item When traversing for the first pass, it visited vertex $v' \in V^{(1)}_i$ in the $i$-th layer.
    \item The (original) second pass program visits $v_{B_0}$ in the $i$-th layer.
\end{itemize}

Having this definition in mind, it is easy to modify the topology of the program. Namely, if after reading the sample $(a_{i + 1}, b_{i + 1})$, $v_{B_0}$ reaches $w_{B_0}$ and $v'$ reaches $w'$, we add an edge from $(v',v_1,v_{B_0})$ to $(w',v_1,w_{B_0})$ with label $(a_{i + 1},b_{i + 1})$. 

Note such modification only blows up the memory by a constant factor (from $\log d$ to $3 \log d$ if the branching program has width $d$). We can without loss of generality only work with modified programs. In the following, we will use $v_{B_1}$ to denote a (modified) vertex in the second pass. That is, $v_{B_1} = (v', v_1, v_{B_0})$.
Note that $v_{B_1}$ uniquely determines both $v'$ and $v_1$. We use $B_1$ to denote the program after this modification.

\newcommand{\wtd}{\widetilde}

\paragraph*{Stage 2: Biasness and High-probability edge Counters.} In this stage, we also keep the first pass of the program unchanged. For the second pass, we will further attach two ``counters'' to each state, denoted by $\cb$ and $\ch$. Initially, they are both zero. When we traverse with an edge $(a,b)$ from vertex $v_{B_1}$, we will increase these two counters by some amount that is uniquely determined by $v_{B_1}$ and $(a,b)$. The update rules will be specified later in \Cref{sec:two-pass-stop-rule}. 

Similar to how we added $v'$ and $v_1$ to the memory of the second pass in Stage 1, we will also add $\cb$ and $\ch$ to the memory. Namely, for each vertex $v_{B_1}$, we will create many duplicates of it and label them by $(v_{B_1}, \ch, \cb)$. In this way, a vertex in the second pass remembers not only $v_{B_1} = (v', v_1, v_{B_0})$ but also $\ch$ and $\cb$. We will apply the modification layer-by-layer as we describe in the next.

Now we will describe how we modify the edges. In layer-$0$ of the second pass, the counters are initialized to $0$. In general, suppose we have modified the first $i-1$ layers. Every vertex in layer $i$ should be of the form $(v_{B_1},\ch,\cb)$. Fix one such vertex $v_{B_1}$. For each edge labeled $(a,b)$ in $B_1$ that goes from $v_{B_1}$ to $w_{B_1}$, we calculate the increased value of the counters, $c_1$ and $c_2$, based on the available information: $v_{B_1},\ch,\cb$ and $(a,b)$. Then we add an edge from $(v_{B_1},\ch,\cb)$ to $(w_{B_1},c_1,c_2)$ with label $(a,b)$. 
We apply the modification for each edge between the $i$-th layer and the $(i+1)$-th layer in $B_1$. 

We will restrict the range of the counters to be integers in $[0, \log |X|]$. Hence, they require only $O(\log \log |X|)$ bits to store. 
Since the desired output $x \in X$ already requires at least $\log |X|$ memory to store, we can without loss of generality ignore the memory blowup of this modification. Hence we will only work with the modified program $B_2$. In the paper, we simply use $v$ to denote a vertex in the program $B_2$. Note that in the second pass, $v$ uniquely determines $v',v_1, v_{B_0}, \cb, \ch$.  As the counters $\cb$ and $\ch$ are now uniquely determined by vertex $v$, we will denote them by $\cb(v)$ and $\ch(v)$.

Finally, note that our modification does not change the functionality of $B$. Therefore, if $B_2$ succeeds in learning with very low probability, so does $B$.

\paragraph*{Comparison with \cite{garg2019time}.} The idea of ``running a copy of the previous pass'' is inherited from their work. However, our implementation is slightly simpler. For one, we observed that we don't need to remember the set of indices on which the second pass traverses a high-probability edge. Only knowing the number of high-probability edges one has traversed so far, i.e. $\ch$, is sufficient.

The introduction of the new counter $\cb$ is novel in our work. Although $\cb$ and $\ch$ require us to modify the program in the same way, these two counters are for completely different purposes. This new counter $\cb$ is crucial for achieving the tight quadratic lower bound on the memory of any efficient parity learner (and more generally the tight lower bound on the memory for any learning problem whose corresponding matrix is an $L_2$-Extractor).

\paragraph*{Comparison with \cite{liu2023memory}.} This recent work by Liu, Raz and Zhan introduced a notion of ``badness level'', which is a technical alternative to the previous potential function in \cite{raz2017time, GargRT18-extractor},  namely, the function $\langle \doubleP_{x\mid v}, \doubleP_{x\mid s}\rangle^k$. Although such ``badness level'' has also a counter-like definition, $\cb$ and ``badness level'' are counting completely different things. In our work, the potential function is nice enough to work with and gives shorter proofs. So we will use the original potential function of \cite{raz2017time, GargRT18-extractor}. 

\subsection{The First Pass}

In this section, we will set up the notation and stopping rules for the first pass.

\paragraph{Computational Path.} The computational path starting from $v_0$ is uniquely determined by $x \in X, a_1, a_2, \dots, a_T \in A$. We use $v_0\wt \wtd{v}$ to denote the event that the (untruncated) computational path reaches $\wtd{v}$ (on $x,a_1,\dots, a_T$).

For a vertex $v' \in V^{(1)}_i$, we can consider the subprogram which consists of only layers $i, i + 1, \dots, T$. In this subprogram, we set $v'$ as the starting vertex. The computational path starting from $v'$ is uniquely determined by $x \in X, a_{i + 1}, a_{i + 2}, \dots, a_T$. We use $v' \wt \wtd{v}$ to denote the event that the (untruncated) computational path from $v'$ reaches $\wtd{v}$ (on $x, a_{i + 1}, a_{i + 2}, \dots, a_T$). \\

To define the stopping rules, we need to first define truncated paths.

\paragraph*{Truncated Path from $v_0$ (Informal).} For any $x \in X, a_1, a_2, \dots, a_T \in A$, we have a uniquely determined computational path. The truncated path will be a prefix of it. The stopping rules specify when the truncated path stops. We will let $v_0 \to \wtd{v}$ denote the event that from $v_0$, we reached $\wtd{v}$ without stopping. Even if we arrived at $\wtd{v}$ and then stopped, this still counts as $v_0 \to \wtd{v}$ since we have reached $\wtd{v}$. 

We will give the formal definition of the truncated path after we have introduced the stopping rules. For any vertex $v' \in V^{(1)}_i$, we will define the truncated path starting from $v'$. But this definition is much more subtle. We will defer it to after the formal definition.

\paragraph*{Bad Events and Stopping Rules.} In the definition of stopping rules, we will need $\doubleP_{x \mid v_0 \to v'}$, defined as the distribution of $x$ conditioning on events $v_0 \to v'$. (The prior distribution of $x,a_1,\dots, a_T$ is uniform.) Although $v_0 \to v'$ is defined using the stopping rules, this is not a circular definition because we will define stopping rules and truncated paths layer by layer: For $v' \in V^{(1)}_i$, the event $v_0 \to v'$ will only depend on the stopping rules of layers $1, 2, \dots, i - 1$. Also the stopping rules of layer $i$ will only depend on events $v_0 \to v'$ for $v' \in V^{(1)}_i$. 

Suppose at layer $i$, the truncated path currently reaches vertex $v' \in V^{(1)}_i$. Recall that $\doubleP_{x \mid v_0 \to v'}$ is the distribution of $x$ conditioning on event $v_0 \to v'$. Based on it, we define the following bad events. They will appear in the layer $i$ stopping rules. 

\begin{itemize}
    \item \emph{Bad edges.}
    The set of bad edges at state $v'$ is defined as 
    \[
    \Bad(v') = \left\{(a',b')\in A\times \{-1,1\} : \Pr_{x'\sim \doubleP_{x|v_0\to v'}}[M(a',x')=1] \not\in \left( \frac{1}{2}-2^{-\rext},\frac{1}{2}+2^{-\rext} \right) \right\}.
    \]
    Note that an edge $(a',b')$ being bad or not only depends on $a'$. For brevity, we sometimes abuse notation by using $\Bad(v')$ to refer to a subset of $A$, which consists of all ``bad'' $a'\in A$.
    \item \emph{Significant States.} If $v'$ is a state that satisfies $\|\doubleP_{x|v_0\to v'}\|_2 \ge 2^{-n}\cdot 2^{\ellsigs^{(1)}}$, we will call $v'$ a significant state.
    \item \emph{Significant Values.} Define the set of significant values at $v'$ as 
    \[
    \SigV(v') = \{x': \doubleP_{x|v_0\to v'}(x') \ge 2^{-n + \ellsigv} \}.
    \]
\end{itemize}

In the first pass, we will immediately stop whenever we encounter these bad events. Formally, these are the stopping rules that are applied \emph{before} traversing from $V^{(1)}_i$ to $V^{(1)}_{i+1}$:

\begin{enumerate}
    \item 
    If we are currently at $v'\in V^{(1)}_i$, which is itself a significant state, then we stop.\label{rule:one-pass-sig-state}
    \item If $v'$ is not significant, but the next edge $(a,b)$ satisfies that $a\in \Bad(v')$, then we stop. \label{rule:one-pass-bad-edge}
    \item If $v'$ is not significant but $x\in \SigV(v')$, then we stop. \label{rule:one-pass-sig-value}
\end{enumerate}

If none of these rules apply, we take a step forward from $v'$ to some $\wtd{v}\in V^{(1)}_{i+1}$ following the next edge. Then we say the truncated path reaches $\wtd{v}$. 

\paragraph*{Truncated Path from $v_0$ (Formal).} Initially for layer $0$, $v_0 \to v_0$ is always true. Suppose we have finished the definition for layers $0, 1, \dots, i$. Then, for each $\wtd{v}\in V^{(1)}_{i+1}$, we can define $v_0\to \wtd{v}$ recursively.

$$v_0\to \wtd{v} \equiv \bigvee_{\wtd{u}\in V^{(1)}_i} (v_0\to \wtd{u}) \land \left[\begin{aligned}
    &\text{From $\wtd{u}$, we traverse an edge and reach $\wtd{v}$} \\ &\text{without triggering any stopping rules at $\wtd{u}$}
\end{aligned}\right].$$

\paragraph*{Truncated Path from $v'$.} For a vertex $v' \in V_i^{(1)}$, we now define the truncated path starting from $v'$. First, from $x, a_{i + 1}, a_{i + 2}, \dots, a_T$, we can uniquely determine the computational path from $v'$. Then $v' \to \wtd{v}$ is defined as the event that the computational path reaches $\wtd{v}$ without triggering the stopping rules. 

The subtlety here is that the conditional distributions $\doubleP_{x \mid v_0 \to \wtd{v}}$ in the stopping rules are still defined using truncated paths from $v_0$ instead of that from $v'$. Formally, for each $\wtd{v} \in V_{i + 1}$,
$$v' \to \wtd{v} \equiv \bigvee_{\wtd{u} \in V^{(1)}_i} (v' \rightarrow \wtd{u}) \land \left[\begin{aligned}&\text{From $\wtd{u}$, we traverse an edge and reach $\wtd{v}$} \\  &\text{without triggering any stopping rules at $\wtd{u}$ where} \\ &\text{the distributions $\doubleP_{x \mid v_0 \to \wtd{u}}$ in the stopping rules}\\ & \text{are still defined w.r.t. $v_0 \to \wtd{u}$.}\end{aligned}\right]$$

We will list several facts to fully explain the subtlety.
They are heavily used in both the previous work \cite{garg2019time} and ours:
\begin{itemize}
    \item For any vertex $v_1$ in the first pass, the event $v_0 \to v_1$ is implied by $(v_0 \to v')\land (v' \to v_1)$. \\ 
    This fact emphasizes that when we define $v' \to v_1$, we are still defining the stopping rules w.r.t. the truncated paths from $v_0$, not $v_1$. Otherwise, this fact would not hold. 
    \item The probability $\Pr[v' \to \wtd{v}]$ can be interpreted as follows: First, we sample $x \in X$ and $a_{i + 1}, a_{i +2}, \dots, a_T$ \emph{uniformly at random}. Then, $\Pr[v' \to \wtd{v}]$ is the probability the truncated path from $v'$ defined by $x, a_{i + 1}, a_{i + 2}, \dots, a_T$ reaches $\wtd{v}$. We emphasize its comparison to the next item.
    
    \item The probability $\Pr[v' \to \wtd{v} \mid v_0 \to v']$ can be interpreted as follows: First we sample $x \in X$ \emph{according to $\mathbb{P}_{x \mid v_0 \to v'}$}. (Note the event $v_0 \to v'$ only depends on $x, a_1, a_2, \dots, a_i$.) Then we sample $a_{i + 1}, a_{i + 2}, \dots, a_T$ \emph{uniformly at random}, because $v_0 \to v'$ is independent of them. 
    
    Then, $\Pr[v' \to \wtd{v} \mid v_0 \to v']$ is the probability that the truncated path from $v'$ determined by these $x, a_{i + 1}, a_{i + 2}, \dots, a_T$ reaches $\wtd{v}$.
\end{itemize}

\paragraph*{The first pass stops with small probability.} The following theorem is proved in \cite{GargRT18-extractor}.

\begin{theorem}[Main Theorem of \cite{GargRT18-extractor}]\label{theo:GRT-extractor}
Suppose $M$ is a $(\kext,\ellext,\rext)$-$L_2$-extractor. Let $\ell,r$ be two parameters satisfying:
\begin{itemize}
    \item $4\ell \le \min\{\kext,\ellext\} - 4$,
    \item $r \le \frac{1}{3} \min\{\rext, \kext, \ell\}$.
\end{itemize}
Let $B$ be a one-pass branching program for the learning task of $M$. Suppose $B$ has:
\begin{itemize}
    \item Width: $2^{\frac{1}{8}\kext \ell}$.  
    \item Length: $T=2^{r}$. 
\end{itemize}
Then, over uniformly random $x\sim X$, $(a_1,\dots, a_T)\sim A^T$, the probability that $B$ stops is at most $2^{-2r}$. Put in other words, there exists an event $G\subseteq X\times A^{T}$ such that
\[
\Pr_{x,a_1,\dots, a_T}[(x,a_1,\dots, a_T) \not\in G] \le 2^{-\ell}.
\]
Furthermore, denote by $v_0$ the starting vertex of $B$. For any final vertex $v$ of $B$, it holds that
\[
\begin{aligned}
& \| \doubleP_{x|(v_0\wt v)\land G} \|_2 \le 2^{\ell + 1} 2^{-n}, \quad \text{and} \\
& \| \doubleP_{x|(v_0\wt v)\land G} \|_\infty \le 2^{3\ell + 4} 2^{-n}.
\end{aligned}
\]
\end{theorem}

\subsection{The Second Pass} 
\label{sec:two-pass-stop-rule}

Now we will set up the second pass. Below we will fix a vertex $v_1 \in V^{(2)}_0$ as the starting vertex. Let $\ell$ be a ``meta'' parameter. Most of the relevant parameters in our proof are stated as multiples of $\ell$ for brevity.

\paragraph*{Computational Path from $v_1$.} For the subprogram containing layers $V^{(2)}_0, V^{(2)}_1, \dots, V^{(2)}_T$, we set the vertex $v_1 \in V^{(2)}_0$ as the starting vertex. For any $x \in X, a_1, a_2, \dots, a_T \in A$, the (untruncated) computational path from $v_1$ is defined as the computational path in that subprogram. Recall that we use $v_1 \wt v$ to denote the event that the computational path from $v_1$ reaches $v$.

\paragraph*{Truncated Path from $v_1$.} We use $v_1 \to v$ to denote the event that from $v_1$, the computational path reached $v$ without stopping. Here we are only considering the second pass subprogram.

To avoid circular definitions, we will define truncated paths and stopping rules layer by layer. In layer $i$, as the stopping rules for all previous layers are well-defined, the event $v_1 \to v$ will be well-defined for all $v \in V^{(2)}_i$. Then we can define the layer $i$ bad events and stopping rules. 

Since truncated paths with an arbitrary starting vertex in the second pass are not used in our two-pass result, we will not define them here.

\paragraph*{Truncated Path from $v_0$.} For a vertex $v$ in the second pass, we use $v_0 \to v$ to denote the event that (1) the computational path starting from $v_0$ reaches $v$ and (2) before reaching $v$ it did not trigger any first pass or second pass stopping rule. Note $v_0 \to v$ is equivalent to $v_0 \to v_1 \land v_1 \to v$. 


\paragraph*{Bad Events} Let $v \in V^{(2)}_i$ be a vertex in layer $i$ of the second pass. We will define the bad events of layer-$i$. Unlike the first pass, we do \emph{not} always stop immediately when we encounter them. 

\begin{itemize}
    \item \emph{High-probability edges.} An input $a' \in A$ is of high-probability, if $$\Pr[a_{i+1}=a'|v_0 \to v] \ge 2^{-n} \cdot 2^{\frac{\kext}{2}}.$$ We define $\High(v)\subseteq A$ as the set of such inputs at $v$.%
    \footnote{Observe that there is no analog of such set for the first pass since there, conditioned on the past, the next $a_{i+1}$ is completely random.}
    \item \emph{Bad edges.} Define the set of bad edges at $v$ as 
    \[
    \Bad(v) = \left\{a'\in A : \Pr_{x'\sim \doubleP_{x|v_1\to v}}[M(a',x')=1] \not\in \left( \frac{1}{2}-2^{-\rext},\frac{1}{2}+2^{-\rext} \right) \right\}.
    \]
    \item \emph{Significant Values.} $\SigV(v)$ is the set of all $x' \in X$ such that $\doubleP_{x|v_1\to v}(x') \ge 2^{-n} \cdot 2^{\ellsigv}$, where $\ellsigv = 50\ell$. 
    \item \emph{Significant States.} Finally, $v$ is called a significant state, if $\|\doubleP_{x|v_1\to v}\|_2\ge 2^{-n}\cdot 2^{\ellsigs^{(2)}}$ where $\ellsigs^{(2)} = 18\ell$. 
\end{itemize}

Recall that, from the second stage of \Cref{sec:modify}, we create two counters $\ch(v)$ and $\cb(v)$. They are uniquely determined by the vertex $v$.  We will now define their update-rules.

\paragraph*{High-probability Counter.}  Whenever we traverse an edge labeled $(a',b') \in A \times \{-1, 1\}$ from $v$ to a new vertex $w$, if $a' \in \High(v)$, we will increase the counter $\ch$ by $1$, i.e. $\ch(w) = \ch(v) + 1$.  Otherwise $\ch(w) = \ch(v)$.

\paragraph*{Biasness Counter.}  
Whenever we traverse an edge labeled $(a',b') \in A \times \{-1, 1\}$ from $v$ to a new vertex $w$, if $a' \in \High(v)$, we will increase the bias-counter $\cb$ by $$\Delta = \left\lfloor - \log\left(\Pr_{x' \sim \doubleP_{x \mid v_1 \to v}}[M(a',x') = b']\right)\right\rfloor.$$
Hence $\cb(w) = \cb(v) + \Delta$. If $a'\notin \High(v)$, set $\cb(w) = \cb(v)$.


\paragraph*{Stopping Rules.} We will now define the stopping rules. 

Now for a vertex $v \in V_i^{(2)}$, we have the following stopping rules. They are applied after we visit $v$ and before we reach $V_{i+1}^{(2)}$.

\begin{enumerate}
    \item Before traversing the next edge, if $x\in \SigV(v)$, we stop.
    \item When we are traversing an edge $(a,b)$ where $a\in \Bad(v) \setminus \High(v)$, we stop.    
    \item Recall that from the first stage of \Cref{sec:modify}, we are keeping a copy of the first pass in our memory. If that first pass copy stops for any reason, we also stop.
    \item If $v$ is a significant state, we stop.
    \item When $\ch(v) > \ellhigh^{(2)}\coloneqq \ell$ or $\cb(v) > \ellbias^{(2)} \coloneqq 14\ell$, we stop.
\end{enumerate}

\subsection{Two-Pass Main Result}

The following statement summarizes our main result for the two-pass learning lower bound.

\begin{theorem}\label{theo:two-pass-main-result}
    Suppose $M$ is a $(\kext,\ellext,\rext)$-$L2$-extractor. Let $\ell,\rlen$ be two parameters 
    satisfying
    \begin{itemize}
        \item $\ell \le \min\{ \frac{\ellext - 4}{100}, \frac{\kext}{8} \}$,
        \item $\rlen \le \min\{ \frac{1}{2}\rext, \frac{\ell}{3}-4 \}$.
    \end{itemize}
    Let $B$ be a two-pass learning program for the learning task of $M$. Suppose $B$ has
    \begin{itemize}
        \item Width $2^{\frac{1}{32}\kext\ell}$.\footnote{This width upper bound is \emph{before} modifying the program. After modifying $B$ as described in \Cref{sec:modify}, we can ensure the resulting program has width at most $2^{\frac{1}{8}\kext\ell}$.}
        \item Length of each pass $T = 2^{\rlen}$.
    \end{itemize}
    Denote by $v_0$ the starting vertex of $B$.
    Then, there exists an event $G\subseteq X\times A^{T}$ (which captures that the program $B$ does not stop early) such that
    \[
    \Pr_{x,a_1,\dots, a_T}[(x,a_1,\dots, a_T) \notin G] \le 2^{-\frac{2}{3}\ell},
    \]
    and for any final vertex $v$ of $B$, it holds that
    \[
    \begin{aligned}
        & \|\doubleP_{x|(v_0\wt v)\land G}\|_2 \le 2^{18\ell+1}\cdot 2^{-n}, \quad \text{and} \\
        & \|\doubleP_{x|(v_0\wt v)\land G}\|_{\infty} \le 2^{50\ell+1} \cdot 2^{-n}.
    \end{aligned}
    \]
\end{theorem}

\subsubsection{Keeping Track of Parameters} 

For a better presentation, we would like to provide the following table, summarizing the main parameters of the two-pass proof and providing a name for each. These parameters are written as multiples of $\ell$, where we assume $\ell$ to be sufficiently large. Also, for the first time reading, we recommend considering $\ellext$ and $\kext$ as large compared with $\ell$ (e.g., $\min\{ \ellext, \kext\} = 1000\ell$). The last row is only required in \Cref{sec:two-pass-transfer}. We will explain its meaning by then.

\begin{table}[H]
    \centering
    \begin{tabular}{c|c|c}
      Name  & Explanation & Quantity \\[4pt]
      \hline 
        $\rlen$  & $2^{\rlen}$: Length of the program & $\rlen \le \frac{1}{3}\ell - 4$      \\[4pt]
       \hline
        $\ellsigs^{(1)}$  & First-pass Significant State Threshold & $\ell$      \\[4pt]
       \hline
       $\ellsigs^{(2)}$  & Second-pass Significant State Threshold & $18\ell$   \\[4pt]
       \hline
       $\ellhigh^{(2)}$  & Second-pass $\ch$ Threshold & $\ell$     \\[4pt]
       \hline
       $\ellbias^{(2)}$  & Second-pass $\cb$ Threshold & $14\ell$     \\[4pt]
       \hline
       $\ellsigv$  & Significant Value Threshold for Both Passes & $50\ell$ \\[4pt]
       \hline        
       $\ellflat^{(1)}$  & First-Pass Flat Threshold (Only for \Cref{sec:two-pass-transfer}) & $3\ell$      \\[4pt]
    \end{tabular}
    \caption{Parameters for the two-pass proof}
    \label{tab:parameter-two-pass}
\end{table}



\section{Potential Analysis} \label{sec:two-pass-potential}
In this section, we analyze the stopping probability due to $\cb$ overflow. 

For every fixed $v_1\in V^{(2)}_0$, every layer $i\in [T]$, and every $v'\in V^{(1)}_i$, we show the program, when starting from $v_1$, stops due to $\cb$ overflow at the $i$-th layer with small probability, even \emph{conditioning on $v_0\to v'$}. That is, we show
\begin{align}
\Pr_{v_1\to v,v\in V^{(2)}_i}[\cb(v) > \ellbias^{(2)} \mid v_0 \to v'] \le 2^{-\ellbias^{(2)}+\ell+O(1)}. \label{equ:two-pass-potential-goal}
\end{align}
However, this statement is not sufficient, because what we want to understand is
\begin{align*}
\Pr_{v_0\to v,v\in V^{(2)}_i}[\cb(v) > \ellbias^{(2)}] = \sum_{v_1\in V^{(2)}_0} \Pr[v_0\to v_1] \cdot \Pr_{v_1\to v}[\cb(v) > \ellbias^{(2)} \mid v_0\to v_1]. 
\end{align*}
Namely, we would like to bound: (note the change of conditioned event)
\begin{align}
 \Pr_{v_1\to v,v\in V^{(2)}_i}[\cb(v) > \ellbias^{(2)} \mid v_0\to v_1]. \label{equ:two-pass-potential-next}
\end{align}

In \Cref{sec:two-pass-transfer}, we show how \eqref{equ:two-pass-potential-goal} implies that \eqref{equ:two-pass-potential-next} is small. In this section, we mainly focus on bounding \eqref{equ:two-pass-potential-goal}.


\subsection{Definition and Overview} \label{subsec:potental-def-overview}

\paragraph*{Definition} Recall that for each vertex $v$, we maintained two counters $\ch(v)$ and $\cb(v)$. For any vertex $v$, we define its potential to be $$\Phi(v) \coloneqq 2^{\cb(v) - \ch(v)}.$$  

For any edge $e$ between $(u,v)$ with label $(a',b')$ define its potential to be 
\[
\Phi(e) = \Phi^{(a',b')}(u) \coloneqq \Phi(v).
\]

If we stop in the second pass on a vertex $u$ \emph{not} due to counter overflow, we will force the next vertex $v$ to be a special vertex $\fail$. Similarly, in the second pass, if we stopped when traversing an edge $(a,b)$ (due to $(a,b)\in \Bad(u)\setminus \High(u)$), we will also let $v \gets \fail$. We define $\Phi(\fail) \coloneqq 0$. Once $v$ becomes $\fail$, it remains $\fail$ for the rest of the second pass.

\paragraph*{Probability Space} Before getting into our proof strategy, we first rigorously define the underlying probability space. 

Fixing a vertex $v_1 \in V_0^{(2)}$ and a layer $i \in [T]$, we uniformly at random sample $x \in X$ and $a_1, a_2, \dots, a_T \in A^T$. They determine the realization of the truncated path from $v_1$, denoted by $\Ttwo$. We emphasize that we start the truncated path from $v_1$ right away, \emph{regardless} of whether $v_0\to v_1$ happens or not under $(x,a_1,\dots, a_T)$.
For any realization of $\Ttwo$, we use $\Ttwo[i]$ to denote the $i$-th layer vertex traversed in $\Ttwo$, and $\Ttwo[i,i+1]$ to denote the edge from $V^{(2)}_{i}$ to $V^{(2)}_{i+1}$ traversed in $\Ttwo$.
We use $|\Ttwo|$ to denote the length of the truncated path (before stopping). For $i > |\Ttwo|$, we define $\Ttwo[i] = \fail$ and $\Ttwo[i,i+1]=\fail$. When $i$ is fixed in context, we write $v \sim \Ttwo$ to express that $v$ is determined by the random process $\Ttwo$. 

Notice that $(x, a_1, a_2, \dots, a_T) \in X \times A^T$ also determines realization of the truncated path from $v_0$, denoted by $\Tone$. This gives us a natural coupling between $\Tone$ and $\Ttwo$. Namely, we draw $(x,a_1,\dots, a_T)$, and consider the realization of $\Tone$ and $\Ttwo$ under the same input. Let $v' = \Tone[i]$ be the $i$-th layer vertex in the first pass. Recall that the second pass of the program remembers the first pass. Hence from $v$, we can uniquely determine the corresponding Pass-$1$ vertex $v'$. 

On the other hand, conditioning on any realization of $v'=\Tone[i]$, the vertex $v = \Ttwo[i]$ is naturally a random variable. Formally, this is the random variable $\Ttwo[i] \mid \Tone[i] = v'$ under the natural coupling between $\Tone$ and $\Ttwo$. To avoid such heavy notation, in this section, we will always fix a $v_1 \in V^{(2)}_0$ and denote the random variable as $(v \mid v_0 \rightarrow v')$. Note if $|\Ttwo| < i$, we would have $v = \fail$. Hence $\fail$ is also in the support of such a random variable.  

\paragraph*{Proof Strategy} 

Having defined the probability space, we will show that: 
\begin{enumerate}
    \item For any fixed $v' \in V^{(1)}_i, v_1 \in V^{(2)}_0$, we show that the expected potential $\Ex_{v \mid v_0 \to v'}[\Phi(v)]$ is small.
    \item By our stopping rule, $\ch(v)$ is always at most $\ell$. 
\end{enumerate}

Together these show that $\Ex_{v \mid v_0 \to v'}[2^{\cb(v)}]$ is small. Hence we can upper bound the probability that the counter $\cb$ overflows. Namely:
$$\Pr[\cb(v) > \ellbias^{(2)} \mid v_0 \to v'] \leq \frac{\Ex_{v \mid v_0 \to v'}[2^{\cb(v)}]}{2^{\ellbias^{(2)}}} \leq \frac{\Ex_{v\mid v_0 \to v'}[\Phi(v)] \cdot 2^{\ell}}{2^{\ellbias^{(2)}}} \le 2^{\ell-\ellbias^{(2)}}.$$

\paragraph*{Intuition.}  We give some intuition for this potential function: After traversing an edge with input $a \in \High(v)$, the counter $\cb$ increases when some small-probability bad event happens. More specifically, when traversing an edge $(a,b)$ with label $b$ where $\Pr[M(a,x)=b\mid v_1\to v] = 2^{-\Delta} \ll \frac{1}{2}$. Moreover, the amount of increase is $\left\lfloor\Delta\right\rfloor$, the logarithm of that probability (up to the rounding).

In other words, with probability $2^{-\Delta}$ we traverse $(a,b)$ and the counter $\cb$ increases by $\lfloor{\Delta\rfloor}$. With probability $1-2^{-\Delta}$, we traverse the edge $(a,-b)$, in which case we do not update the counter. Overall, the expected increase of $2^{\cb(v)}$, after traversing an edge $(a,*)$ with $a\in \High(v)$, is
\[
2^{-\Delta} \cdot 2^{\lfloor{\Delta\rfloor}} + (1-2^{-\Delta}) \le 2.
\]
So the expectation of $2^{\cb(v)}$ at most doubles after traversing a high-probability edge. Thus, by introducing an extra $2^{-\ch(v)}$ term to account for the high-probability edges traversed so far, we can show the expectation of $\Phi(v) = 2^{\cb(v) - \ch(v)}$ is almost\footnote{In the real proof, we need to analyze the expectation \emph{conditioning} on the event $v_0\to v'$. The conditioning makes the argument more subtle. As a result, the expectation of $\Phi$ might grow by a factor of $(1+2^{-\rlen})$ after traversing a layer. Since there are at most $i \le T$ layers, the total growth is bounded by $O(1)$.} non-increasing.

\subsection{Evolution of Potential}

Let us first analyze how the potential function changes after traversing an edge. Below we will always fix a vertex $v_1 \in V_0^{(2)}$ and a layer $i \in [T]$. For the second second-pass-only truncated path $\Ttwo$ and any vertex $u \in V^{(2)}_{i - 1}$, we will let $(v_1 \to u)$ denote the event $\Ttwo(i - 1) = u$. 

\begin{lemma} \label{lem:edge-potential}
Fixing $v_1 \in V_0^{(2)}$, for any vertex $u$ in the second pass and an edge labeled $(a,b)$, 
$$\Phi^{a,b}(u) \leq \Phi(u) \cdot \frac{1 + 2 \cdot 2^{-\rext}}{2 \Pr_{x \mid v_1 \to u}[M(a,x) = b]}.$$
\end{lemma}
\begin{proof}
We need to discuss the four cases: Let $v$ be the vertex we reach after traversing this edge.
\begin{itemize}
    \item We stopped on $u$ due to stopping rules. Then $\Phi^{a,b}(u) = \Phi(\fail) = 0$ since we force $u$ to traverse to $\fail$ after reading $(a,b)$.
    \item $a \not\in \Bad(u)$: In this case, both $\cb$ and $\ch$ are unchanged. We know $\Phi^{a,b}(u) = \Phi(u)$.  By the definition of bad edges, we know that $$\Pr_{x \mid v_1 \to u}[M(a,x) = b] \in \left(\frac{1}{2} - 2^{-\rext}, \frac{1}{2} + 2^{-\rext}\right).$$
    
    As a result $$\frac{\Phi^{a,b}(u) }{\Phi(u)} = 1 \leq \frac{1 + 2 \cdot 2^{-\rext}}{2 \Pr_{x \mid v_1 \to u}[M(a,x) = b]}.$$
    \item  $a \in \Bad(u) \setminus \High(u)$: By our stopping rule, we will stop on this edge. Hence $\Phi^{a,b}(u) = \Phi(\fail) = 0$.  The inequality trivially holds. 
    \item $a \in \High(u)$: In this case, the counter $\ch(v) = \ch(u) + 1$ and $\cb(v) = \cb(u) + \Delta$ with $ \Delta = \left\lfloor-\log\left( \Pr_{x \mid v_1\to v}[M(a,x)=b] \right) \right\rfloor$. Hence, 
    $$\frac{\Phi^{a,b}(u) }{\Phi(u)} = 2^{\Delta - 1} \leq \frac{1}{2 \Pr_{x \mid v_1 \to u}[M(a,x) = b]}.$$
\end{itemize}
Having verified all possible cases, we conclude the validity of the lemma.
\end{proof}

\subsection{Potential Grows Slowly}

As mentioned in \Cref{subsec:potental-def-overview}, we will analyze the conditional expectation:
\[
\Ex[\Phi(v) \mid v_0\to v'].
\]
We observe that
\[
\Ex[\Phi(v) \mid v_0\to v'] = \frac{\Ex_{v\sim \Ttwo}[\Phi(v) \cdot \mathbbm{1}[v_0\to v']]}{\Pr[v_0\to v']}.
\]
Here $v \sim \Ttwo$ is a random variable determined by the random process $\Ttwo$, while $(v_0 \to v')$ is an event for the random process $\Tone$. As explained in \Cref{subsec:potental-def-overview}, we are taking the natural coupling between $\Tone$ and $\Ttwo$. So the expectation on the numerator is well-defined. Since $v$ runs a copy of the first pass, the indicator $\mathbbm{1}[v_0 \to v']$ is the same as $\mathbbm{1}[v' \text{ is remembered in } v]$. 
\\

We are ready to state the core lemma in this section.

\begin{lemma}\label{lemma:two-pass-pot-slow-grow}
    For every $i\in [T]$ and every $v'\in V^{(1)}_i$ such that $\Pr[v_0\to v'] \ne 0$, we have
    \[
    \Ex[\Phi(v) \mid v_0\to v'] = \frac{\Ex_{v\sim \Ttwo}[\Phi(v) \cdot \mathbbm{1}[v_0\to v']]}{\Pr[v_0\to v']}\le (1+2^{-2\rlen+2})^{i}.
    \]
\end{lemma}

\begin{proof}
We use induction on $i\in [T]$. For the case that $i=0$, $v_0\in V^{(1)}_0$ (resp. $v_1\in V^{(1)}_1$) is the only vertex such that $\Pr[v_0\to v_0] \ne 0$ (resp. $\Pr[v_1\to v_1]\ne 0$). The lemma holds trivially. Next, suppose the lemma holds for $i$. We prove it for the case of $i+1$. 

For any edge $e'$ linking $u'\in V^{(1)}_{i}$ and $v'\in V^{(1)}_{i+1}$, let $v_0\to e'$ denote the event that the program traverses the edge $e'$ without stopping at $u'$. (In particular, this implies that $e'\notin \Bad(u')$).

Fix $v'\in V^{(1)}_{i+1}$. Denote by $\Gamma^-(v')$ the set of incoming edge $e'$ to $v'$ such that $\Pr[v_0\to e'] \ne 0$. We observe that
\[
\Pr[v_0\to v'] = \sum_{e'\in \Gamma^-(v')}\Pr[v_0\to e'],
\]
and
\[
\Ex[\Phi(v) \cdot \mathbbm{1}[v_0\to v']] = \sum_{e'\in \Gamma^{-}(v')} \Ex_{e\sim \Ttwo[i,i+1]}[\Phi(e) \cdot \mathbbm{1}[v_0\to e']].
\]
Again, when we write $\mathbbm{1}[v_0\to e']$ in the expectation, we are considering the natural coupling between $\Tone$ and $\Ttwo$. Thus, $\mathbbm{1}[v_0\to e']$ denotes the event that the first pass traverses the edge $e'$. Observe that the event indicator $\mathbbm{1}[v_0\to e']$ is determined after conditioning on $v_1\to e$. More precisely, suppose $e'$ connects $(u',v')$ and has label $(a',b')$. Then $\mathbbm{1}[v_0\to e']$ is true, if and only if $e$ links two vertices $(u,v)$ where $u$ remembers $u'$, and the label of $e$ is exactly $(a',b')$. 

The following lemma is the key step in the induction proof.

\begin{lemma}\label{lemma:two-pass-pot-vertex-to-edge}
    Assuming that for every $u'\in V^{(1)}_i$ with $\Pr[v_0\to u'] \ne 0$, we have
    \[
    \Ex[\Phi(u) \mid v_0\to u'] := \frac{\Ex_{u\sim \Ttwo}[\Phi(u) \cdot \mathbbm{1}[v_0\to u']]}{\Pr[v_0\to u']}\le (1+2^{-2\rlen+2})^{i}.
    \]
    Then, for every $e'$ being an edge from $V^{(1)}_{i}$ to $V^{(1)}_{i+1}$ such that $\Pr[v_0\to e'] \ne 0$, we have
    \[
    \frac{\Ex_{e\sim \Ttwo[i,i+1]}[\Phi(e) \cdot \mathbbm{1}[v_0\to e']]}{\Pr[v_0\to e']} \le (1+2^{-2\rlen+2})^{i+1}.
    \]
\end{lemma}

We defer the proof of \Cref{lemma:two-pass-pot-vertex-to-edge} to the end of the subsection. Assuming \Cref{lemma:two-pass-pot-vertex-to-edge} for now, we quickly finish the proof of \Cref{lemma:two-pass-pot-slow-grow}. Note that
\[
\begin{aligned}
    \Ex[\Phi(v) \cdot \mathbbm{1}[v_0\to v']] 
    &= \sum_{e'\in \Gamma^{-}(v')} \Ex_{e\sim \Ttwo[i,i+1]}[\Phi(e) \cdot \mathbbm{1}[v_0\to e']] \\
    &\le \sum_{e'\in \Gamma^-(v')} \Pr[v_0\to e'] \cdot (1+2^{-2\rlen+2})^{i+1} \\
    &\le \Pr[v_0\to v'] \cdot (1+2^{-2\rlen+2})^{i+1}.
\end{aligned}
\]
Re-arranging proves the desired bound for $v'$. As this holds for every $v'\in V^{(1)}_{i+1}$, we have verified the lemma for $i+1$. By induction on $i$, this completes the proof.
\end{proof}

We prove \Cref{lemma:two-pass-pot-vertex-to-edge} now.

\begin{proof}[Proof of \Cref{lemma:two-pass-pot-vertex-to-edge}]
    Suppose $e'$ links $(u',v') \in V^{(1)}_{i}\times V^{(1)}_{i+1}$ with label $(a',b')$. As $\Pr[v_0\to e'] > 0$, we know $e'$ is not a bad edge. Thus,
    \begin{align*}
    \Pr[v_0\to e'] 
    &= \Pr[a_{i+1}=a'] \cdot \Pr_{\Tone}[(v_0\to u')\land (x\notin \SigV(u'))\land (M(a',x)=b')\mid a_{i+1}=a'] \\
    &\ge  2^{-n} \cdot \Pr[v_0\to u'] \left( \Pr_{x'\sim \doubleP_{x|v_0\to u'}}[M(a',x')=b'] - \Pr_{x'\sim \doubleP_{x|v_0\to u'}}[x'\in \SigV(u')]\right) \\
    &\ge 2^{-n} \cdot \Pr[v_0\to u'] \left(\frac{1}{2} - 2^{-\rext} - 2^{\ellsigv^{(1)} -\ellsigv} \right) \tag{$a'$ is not bad} \\
    &\ge \Pr[v_0\to u'] \cdot 2^{-n} \cdot \left(\frac{1}{2} - 2^{-2\rlen}\right).
    \end{align*}
    On the other hand, consider a realization of $(\Tone,\Ttwo)$ (under the natural coupling). Given $e\sim\Ttwo[i,i+1]$, recall that $\mathbbm{1}[v_0\to e']$ is true if and only if $e$ has label $(a',b')$ and connects two vertices $(u,v)$ where $u$ remembers $u'$. We consider the distribution of $(u,x)$ under $\Ttwo$, where $u= \Ttwo[i]$ and $x\in X$ is the hidden vector.
    We then calculate the probability that the program traverses an edge from $V^{(2)}_i$ to $V^{(2)}_{i+1}$ with label $(a',b')$. That is, we can write
    \begin{align*}
    \Ex_{\Ttwo}[\Phi(e) \cdot \mathbbm{1}[v_0\to e']] 
    &= \Ex_{(u,x)\sim \Ttwo}\left[ \Phi^{a',b'}(u)\cdot \mathbbm{1}[v_0\to u'] \cdot \mathbbm{1}[(x\notin \SigV(u))\land (M(a',x)=b')] \right] \cdot \\ &\qquad \Pr[a_{i+1}=a'] \tag{$a_{i+1}$ is independent of $(u,x')$}\\
    &\le 2^{-n} \Ex_{u\sim \Ttwo}\left[ \Phi^{a',b'}(u)\cdot \mathbbm{1}[v_0\to u'] \cdot \Pr_{x'\sim \doubleP_{x|v_1\to u}}[M(a',x')=b'] \right] \tag{Dropping $\mathbbm{1}[x\notin \SigV(u)]$} \\
    &\le 2^{-n} \Ex\left[ \Phi(u)\cdot \frac{1+2^{-\rext+1}}{2} \mathbbm{1}[v_0\to u']  \right]  \tag{\Cref{lem:edge-potential}} \\
    &\le \Ex\left[ \Phi(u) \cdot \mathbbm{1}[v_0\to u']\right] \cdot 2^{-n} \left( \frac{1}{2} + 2^{-\rext}\right).
    \end{align*}

    Combining the two inequalities and using the assumption, we obtain
    \[
    \begin{aligned}
    \frac{ \Ex_{e}[\Phi(e) \cdot \mathbbm{1}[v_0\to e']] } { \Pr[v_0\to e'] } 
    \le \frac{ \Ex_{u}[\Phi(u) \cdot \mathbbm{1}[v_0\to u']] } { \Pr[v_0\to u'] } \cdot  \frac{\frac{1}{2} + 2^{-\rext}}{\frac{1}{2} - 2^{-2\rlen}} 
    \le (1+2^{-2\rlen+2})^{i+1},
    \end{aligned}
    \]
    as claimed.
\end{proof}

\subsection{Upper Bounding the Probability of $\cb$ Overflow}

\begin{corollary}\label{cor:edge-potential-exp}
For any layer $i \in [T]$, $v_1 \in V^{(2)}_0$ and $v' \in V^{(1)}_i$, we have 
$$\Pr[\cb(v) > \ellbias^{(2)} \mid v_0 \to v'] \leq \frac{(1 + 2^{-2\rlen+2})^i\cdot 2^{\ell}}{2^{\ellbias^{(2)}}} \le 2^{\ell - \ellbias^{(2)}+1}.$$
\end{corollary}

\begin{proof}
By the stopping rule for $\ch$, we know that if we have not stopped, we always have $\ch(v) \leq \ell$. Consequently, we have $\Phi(v)\ge 2^{\cb(v)-\ell}$. Thus,
\[
\Pr[\cb(v) > \ellbias^{(2)} \mid v_0 \to v'] \leq \frac{\Ex_{v \mid v_0 \to v'}[2^{\cb(v)}]}{2^{\ellbias^{(2)}}} \leq \frac{\Ex_{v\mid v_0 \to v'}[\Phi(v)] \cdot 2^{\ell}}{2^{\ellbias^{(2)}}}.
\]
We finish the proof by noting that $\Ex_{v\mid v_0 \to v'}[\Phi(v)] \le (1 + 2^{-2\rlen+2})^i \le 2$.
\end{proof}

\section{Transfer Lemma} \label{sec:two-pass-transfer}
In this section, we will present a key lemma (the ``transfer lemma'') used in the analysis, which is a generalization of Lemma 9 in \cite{garg2019time} (proceeding version).
Jumping ahead, the statement we will prove in this section is the following lemma.

\begin{restatable}[Transfer Lemma]{lemma}{twoPassTransfer}\label{lemma:Gv-magic}\label{lemma:two-pass-transfer}
    Let $v' \in V^{(1)}_i$ be a vertex in the $i$-th layer of the first pass, and $v_1 \in V^{(2)}_0$ be a vertex at the beginning of the second pass. Define a set 
    $$
    S_{v',v_1,i} \coloneqq \{ v\in V^{(2)}_i: \text{$v$ remembers $v'$ and $v_1$} \}.
    $$
    Let $E:X\times V^{(2)}_i\to \{0,1\}$ be any event such that $E(x,v)$ that only depends on $x$ and $v$. Assume that for all fixed $v'\in V^{(1)}_i, v_1\in V^{(2)}_0$, 
    \begin{align}\ \label{equ:assumption}
    \sum_{v \in S_{v',v_1,i} } \Pr[(v_1 \to v) \land E(x,v)] \leq \Pr[v_0 \to v'] \cdot 2^{-k}.
    \end{align}
    This implies for all $v'\in V^{(1)}_i$, $$\sum_{v_1 \in V^{(2)}_0} \sum_{v \in S_{v',v_1,i}} \Pr[(v_0 \to v) \land E(x, v)] \le \Pr[v_0 \to v'] \cdot (2^{-k + 12\ellsigs^{(1)}+4} + 2^{-\ellsigs^{(1)}+2}).$$
\end{restatable}

\Cref{lemma:two-pass-transfer} is the \emph{only} technical statement in this section that is required in \Cref{sec:two-pass-analysis} (the main proof of the two-pass lower bound). We will use the lemma to bound the stopping probability due to Significant Value and $\cb$ overflow in the second pass.

The rest of the section is devoted to the proof of \Cref{lemma:two-pass-transfer}. We will give a clean interpretation and a simplified proof of it. 

\subsection{Overview}

For each layer $i$, we would like to prove that the probability the program stops at the $i$-th layer of the second pass (i.e. $V^{(2)}_i$) is small. Since a layer-$i$ vertex $v$ in the second pass remembers the corresponding layer-$i$ vertex $v'$ in the first pass, we have
\begin{align}
&\Pr_{x , a_1, a_2, \dots, a_T}\left[\text{truncated path from }v_0\text{ stops at  $V^{(2)}_i$} \right] \notag \\
&\quad= \sum_{v'\in V^{(1)}_i}\Pr[v_0\to v'] \cdot \Pr_{x , a_1, a_2, \dots, a_T}\left[\text{truncated path from }v_0\text{ stops at  $V^{(2)}_i$} \ \middle\vert \ v_0 \to v'\right]. \label{equ:transfer-sum}
\end{align}

\paragraph*{Informal Statement.} 
Let $k$ be a parameter (think of $k$ as a large constant times $\ell$). Assume the following inequality holds on the truncated path \emph{from $v_1$}: 
\begin{equation} \label{equ:transfer-asumpt}
\forall v' \in V_i^{(1)}, v_1 \in V^{(2)}_0, \quad \Pr_{x, a_{\leq i}}\left[\text{truncated path from }v_1\text{ stops at  $V^{(2)}_i$} \ \middle\vert \ v_0 \to v'\right] \leq 2^{-k}.
\end{equation}

Given the assumption, we prove the following bound on the truncated path \emph{from $v_0$}:
\begin{equation} \label{equ:transfer-conclu}
\forall v' \in V_i^{(1)}, \quad \Pr_{x , a_1, a_2, \dots, a_T}\left[\text{truncated path from }v_0\text{ stops at  $V^{(2)}_i$} \ \middle\vert \ v_0 \to v'\right] \leq 2^{-k+O(\ell)}.
\end{equation}
Plugging \eqref{equ:transfer-conclu} into \eqref{equ:transfer-sum} shows that the program stops at $V_{i}^{(2)}$ with an exponentially small probability. Summing up $i\in [T]$ shows that the probability of stopping somewhere along the second pass is also small. Thus, it remains to prove \eqref{equ:transfer-conclu}.

\paragraph*{Intuition.} To get some intuition, let us looks at a much stronger assumption than \eqref{equ:transfer-asumpt}. It will trivially imply \eqref{equ:transfer-conclu}, but at the same time,  is too good to be true:
\begin{equation} \label{equ:strong}
\forall v' \in V_i^{(1)}, \quad \Pr_{x, a_{\leq i}}\left[\exists  v_1 \in V^{(2)}_0, \text{truncated path from }v_1\text{ stops at  $V^{(2)}_i$}  \ \middle\vert \ v_0 \to v'\right] \leq 2^{-k}.
\end{equation}
Note here we are simply changing the order between the qualifier and the probability. It is a stronger statement because there are $2^{\Theta(n^2)} \gg 2^k$ many possibilities of $v_1$. Hence \eqref{equ:strong} cannot be derived from \eqref{equ:transfer-asumpt} by a simple union bound. 

Now we will prove that trivially  $\eqref{equ:strong} \implies \eqref{equ:transfer-conclu}$. Given $v_0 \to v'$, the vertex $v_1$ is determined by the truncated path from $v'$, which depends on $x, a_{> i}$. By \eqref{equ:strong}, with probability at least $1 - 2^{-k}$, for all vertex $v_1$, the truncated path from $v_1$ will not stop at $V_i^{(2)}$. As a result, starting at the particular vertex $v_1$ picked by the program does not stop at $V^{(2)}_i$.

But for $\eqref{equ:transfer-asumpt} \implies \eqref{equ:transfer-conclu}$, the argument breaks down. This is due to the following adaptivity issue: Consider an adversary who, after observing $(a_{i+1},b_{i+1}),\dots, (a_{T},b_{T})$, figures out the value of $x$, and picks the worst vertex $v_1$ maximizing $\Pr_{a_{\leq i}}\left[\text{truncated path from }v_1\text{ stops at  $V^{(2)}_i$} \ \middle\vert \ \text{value of }x, \ v_0 \to v'\right].$ Then we will have no control over such probability. 

\begin{figure}[H]
    \centering
    \scalebox{1.4}{
    \begin{tikzpicture}
  \node (v0) {$v_0$};
  \node[right=2cm of v0] (v') {$v'$}; 
  \node[right=2cm of v'] (v1) {$v_1$};
  \node[below=0.5cm of v0] (v1_second) {$v_1$}; 
  \node[right=2cm of v1_second] (v) {$v$};
  \node[above = of v'](s1) {};
  \node[below = of v](t1) {};
  \node[left = of s1](sl1) {};
  \node[below = 0.2cm of sl1] {$x, a_{\leq i}$};
  \node[right = of s1](sr1) {};
  \node[below = 0.2cm of sr1] {$x, a_{> i}$};

  \draw[->] (v0) -- (v');
  \draw[->] (v') -- (v1);
  \draw[->] (v1_second) -- (v);
  \draw[dashed, red, thick] (s1) -- (t1);
\end{tikzpicture}}
    \caption{The adaptivity issue.}
    \label{fig:two-pass-decouple}
\end{figure}
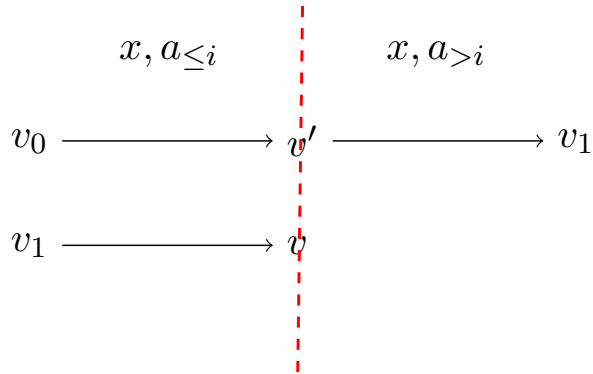
Here the adaptivity issue comes from the following fact: The choice of $v_1$ (conditioning on $v_0\to v'$) depends on $x$, but at the same time, the truncated path from $v_1$ (before layer $i$) also depends on $x$.  (See \Cref{fig:two-pass-decouple}.) Although the choice of $v_1$ depends on $v', x, a_{> i}$, we observe that $v'$ and $a_{> i}$ do not contribute to the adaptivity issue. This is because (1) we have already conditioned on a fixed $v'$, and (2) $a_{\leq i}$ and $v'$ are (jointly) independent of $a_{> i}$.

Luckily, such an adversary does not exist in our setting. This is because not only $v_1$ is picked depending on $v', x, a_{\leq i}$, but also it is picked by a \emph{one-pass algorithm with a small memory}. 
Since a one-pass algorithm does not learn much about $x$, it cannot pick an adversarial $v_1$ based on $x$. This is the main idea of our proof. 

\paragraph*{Proof Overview.} We will now formalize the intuition that the one-pass algorithm $v' \wt v_1$ cannot be a powerful adversary. By the analysis for one-pass algorithms, there are two possibilities. 

\begin{itemize}
    \item With $1 - 2^{-\Theta(\ell)}$ probability, $v' \wt v_1$ does not trigger any stopping rule. We will denote this case as the ``good event'' $G_{v'}$. In this case, the algorithm learns only a negligible amount of information about $x$. 
    
    More specifically, by our stopping rules, the posterior $\doubleP_{x \mid G_{v'} \land v' \wt v_1}$ cannot have significant values, i.e., $\doubleP_{x \mid G_{v'} \land v' \wt v_1}(x') < 2^{O(\ell)} \cdot 2^{-n}$.
    \footnote{We will formally define $G_{v'}$ in \Cref{def:good-event}. For technical reasons, $G_{v'} \land v' \wt v_1$ is not the same as $v' \to v_1$. We will explain such difference fully in the remark following \Cref{def:good-event}. }
    We say the distribution $\doubleP_{x \mid G_{v'} \land v' \wt v_1}$ is flat if and only if this holds. This implies that for all $x' \in X$,
    \begin{align*}\Pr[G_{v'} \land (v' \wt v_1) \mid x = x'] &= \frac{\Pr[G_{v'} \land (v' \wt v_1)] \cdot \doubleP_{x \mid G_{v'} \land (v' \wt v_1)}(x') }{\Pr[x = x']}\\
    &\leq 2^{O(\ell)} \cdot \Pr[G_{v'} \land (v' \wt v_1)].
    \end{align*}
    Intuitively, because the distribution $\doubleP_{x \mid G_{v'} \land (v' \wt v_1)}$ is flat, when $G_{v'}$ happens, the choice of $v_1$ almost does not depend on the value of $x$. Hence the adversary is picking $v_1$ ``almost non-adaptively''.

    \item With $2^{-\Theta(\ell)}$ probability, $v' \wt v_1$ stops. We denote this as the event $\bar{G_{v'}}$, and this is saying that for all fixed $v'$,  $\Pr[\bar{G_{v'}}] \leq 2^{-\Theta(\ell)}$. If $\bar{G_{v'}}$ happens, the one pass adversary $v' \wt v_1$ may have learned some nontrivial information about $x$. 
    
    We need to upper bound the overall probability that, from $v_0$, we reach a vertex $v' \in V^{(1)}_i$ such that $G_{v'}$ happens. Formally, this is the probability
    $$\sum_{v' \in V^{(1)}_i} \Pr[(v_0 \to v') \land \bar{G_{v'}}].$$
    Note the events $\{v_0 \to v'\}_{v' \in V_i^{(1)}}$ are disjoint. So we can interpret it as follows: A one-pass adversary $v_0 \to v'$ picks a vertex $v'$. Then we check if $\bar{G_{v'}}$ happens for that $v'$.

    Here a similar adaptivity issue arises: $\bar{G}_{v'}$ depends on $x, a_{> i}$, but $v'$ is picked by the adversary $v_0 \to v'$ which depends on $x, a_{\leq i}$. Luckily, we are able to perform the same trick again. In this case, the posterior distribution $\doubleP_{x \mid v_0 \to v'}$ is flat (for all $x'$, $\doubleP_{x \mid v_0 \to v'}(x') \leq 2^{O(\ell)} \cdot 2^{-n}$) due to our stopping rules.\footnote{Technically, $v_0 \to v'$ only ensures that $v_0$ reaches $v'$, and those $x \in \SigV(v
    ')$ is not truncated yet. So the $\doubleP_{x \mid v_0 \to v'}$ may not be flat. This will make the actual analysis a bit more complicated. We will handle this issue in \Cref{sec:two-pass-transfer-flat}. } Hence,
    $$\Pr[v_0 \to v' \mid x = x'] = \frac{\Pr[v_0 \to v'] \cdot \doubleP_{x \mid v_0 \to v'}(x')}{\Pr[x = x']} \leq 2^{O(\ell)} \cdot \Pr[v_0 \to v'].$$
    Intuitively, because $\doubleP_{x \mid v_0 \to v_1}$ is flat, the choice of $v'$ almost does not depend on the value of $x$. The adaptivity issue then goes away.
\end{itemize}

\subsection{Proof of the Transfer Lemma} \label{sec:two-pass-transfer-proof}

We prove the transfer lemma in this section. We need two technical ingredients, introduced in \Cref{sec:two-pass-transfer-flat} and \Cref{sec:two-pass-transfer-good-E}. The statement and the proof of the lemma are given in \Cref{sec:two-pass-transfer-core-proof}.

\subsubsection{Flat Truncated Path}\label{sec:two-pass-transfer-flat}

First, we can assume $v'$ is not significant (as otherwise the program stops in the first pass anyway). For proving the transfer lemma, we introduce $v_0 \wtone v'$ to denote the event 
\[
(v_0\to v') \land \mathbbm{1}[\doubleP_{x|v_0\to v'}(x) \le 2^{\ellflat^{(1)}}\cdot 2^{-n}],
\]
where we define $\ellflat^{(1)} = 3\ellsigs^{(1)}$ (note that this is \emph{different} from the significant value threshold).

This definition ensures that $\doubleP_{x \mid v_0 \wtone v'}$ is ``flat'', in the sense that it contains no significant values, as captured by the following claim.

\begin{claim}\label{claim:flat-no-sigv}
For all $x' \in X$, $\doubleP_{x \mid v_0 \wtone v'} (x') \leq 2^{\ellflat^{(1)}+1} \cdot 2^{-n}.$
\end{claim}
\begin{proof}
As $v'$ is not significant, we have $\|\doubleP_{x \mid v_0 \to v'}\| \leq 2^{\ellsigs^{(1)}} \cdot 2^{-n}$. Hence, we have
$$
c\coloneqq \sum_{x' \in X} \doubleP_{x\mid v_0 \to v'}(x')\cdot \mathbbm{1}[\doubleP_{x\mid v_0 \to v'}(x') \geq 2^{\ellflat^{(1)}} \cdot 2^{-n}] \leq \frac{2^n \cdot \|\doubleP_{x \mid v_0 \to v'}\|^2}{2^{ \ellflat^{(1)} } \cdot 2^{-n}} \le \frac{1}{2}.
$$
On the other hand, if $\doubleP_{x \mid v_0 \wtone v'} (x') > 0$, we must have $\doubleP_{x \mid v_0 \to v'}(x') \leq 2^{\ellflat^{(1)}} \cdot 2^{-n}$. Putting these two together, we get 
$$
\doubleP_{x \mid v_0 \wtone v'}(x') \leq \frac{\doubleP_{x \mid v_0 \to v'}(x')}{1 - c} \leq 2^{\ellflat^{(1)}+1}\cdot 2^{-n},
$$
as desired.
\end{proof}

For ``flat'' distributions, we have the following lemma. 
It will come in handy in our analysis.

\begin{restatable}{lemma}{flatdist}\label{lemma:flat-dist} Consider drawing uniformly random $(x,a_1,\dots,a_T)$. Let $E\subseteq X\times A^T$ be an event. Suppose there exists $p > 0$ such that $\Pr_{x,a_1,\dots, a_T}[x=x'|E]\le 2^{p}\cdot 2^{-n}$ for all $x'\in X$.
Then, for all $x'\in X$,
    \[
    \Pr_{a_1,\dots, a_T}[E|x=x'] \le 2^{p} \cdot\Pr_{x,a_1,\dots, a_T}[E].
    \]
\end{restatable} 
%
\begin{proof}
By Bayes' rule:
$$
\Pr[E|x=x'] 
= \frac{\Pr[x=x'|E]\cdot \Pr[E]}{\Pr[x=x']} 
\le 2^{p} \cdot  \Pr[E].$$
\end{proof}

\subsubsection{Good Events}\label{sec:two-pass-transfer-good-E}

Before we present the transfer lemma, we have to make some definitions and claims about the one-pass learning algorithm.

\begin{definition} [Good Event $G_{v'}$]  \label{def:good-event}
Let $B$ be our branching program. For every vertex $v'\in V^{(1)}_i$, consider the sub-program $B'$ starting at $v'$ and ending in the last layer of the first pass. 

Consider this sub-program $B'$ as \emph{the whole program} and $v'$ as the starting point.
We define $G_{v'}$ as the event given by applying \Cref{theo:GRT-extractor} on $B'$ with parameters $(4\cdot \ellsigs^{(1)}, \rlen)$.
\end{definition}

All we need from $G_{v'}$ are the following two properties stated in \Cref{theo:GRT-extractor}. First,
\begin{equation}\label{eq:Gv'_whp}
\Pr_{x,a_{i+1},\dots, a_{T}}[(x,a_{i+1},\dots, a_T) \not\in G_{v'}] \le 2^{-4\ellsigs^{(1)}}.
\end{equation}
Moreover, for every $v_1\in V^{(0)}_{T}$, we have
\begin{equation}\label{eq:flat_Gv'}
    \|\doubleP_{x|(v'\wt v_1)\land G_{v'}} \|_{\infty} \le 2^{12\ellsigs^{(1)} + 4}\cdot 2^{-n}.
\end{equation}

\paragraph*{Remark on \Cref{def:good-event}} For $G_{v'}$, the subtlety is that the conditional probabilities used in the stopping rules are defined w.r.t. \emph{the truncated paths in $B'$ starting from $v'$} instead of \emph{the truncated paths in $B$ starting from $v_0$}. This makes a huge difference, for example, $G_{v'} \land (v' \wt v_1)$ is not the same as $v' \to v_1$. (except for the special case of $v' = v_0$, where the two events $G_{v_0} \land (v_0 \wt v_1)$ and $v_0 \to v_1$ do coincide.)
~\\

We will need the fact that such a good event usually happens. 
\begin{claim} [$G_{v'}$ usually happens] \label{Claim:UnlikelyGbar}
For every non-significant state $v'\in V_i^{(1)}$, it holds that
\[
\Pr[(v_0\to v') \land \bar{G_{v'}}] \le 2^{-\ellsigs^{(1)}+2} \cdot \Pr[v_0\to v'].
\]
\end{claim}
\begin{proof}
Depending on whether $v_0\wtone v'$ happens, we have
\[
\Pr[(v_0 \to v') \land \bar{G_{v'}}] \le\Pr[(v_0 \wtone v') \land \bar{G_{v'}}] + \Pr[(v_0\to v')\land \lnot (v_0\wtone v')].
\]
For the first term, we have 
\[
\begin{aligned}
 \Pr[(v_0 \wtone v') \land \bar{G_{v'}} ]
&=  \Pr[\bar{G_{v'}}] \Pr[v_0\wtone v' \mid \bar{G_{v'}}].
\end{aligned}
\]
Observe that $v_0\wtone v'$ depends on $x,a_{\leq i}$, while $\bar{G_{v'}}$ depends on $x, a_{> i}$. This means that for all $x' \in X$, 
\[
(v_0\wtone v') \perp (G_{v'}) \mid x = x'.
\]
Hence,
\begin{align*}
 \Pr[\bar{G_{v'}}] \Pr[v_0\wtone v' \mid \bar{G_{v'}}] 
&\le \Pr[\bar{G_{v'}}] \cdot \max_{x'\in X} \{ \Pr[v_0\wtone v' \mid x=x'] \} \\
&\le 2^{-4\ellsigs^{(1)}} \cdot 
 \max_{x'\in X} \{ \Pr[v_0\wtone v' \mid x=x']
\tag{Eq.~\eqref{eq:Gv'_whp}}
\\
&\le 2^{-4\ellsigs^{(1)}} \cdot 
\Pr[v_0\wtone v']\cdot 2^{\ellflat^{(1)} + 1} 
\tag{\Cref{claim:flat-no-sigv}, \Cref{lemma:flat-dist}} \\
&\le 2^{-\ellsigs^{(1)}+1} \cdot \Pr[v_0\to v'].
\end{align*}
For the second term, we have
\begin{align*}
\Pr[(v_0\to v')\land \lnot (v_0\wtone v')] 
&= \Pr[v_0\to v'] \cdot \Pr_{x\sim \doubleP_{x|v_0\to v'}}[\doubleP_{x|v_0\to v'}(x) \ge  2^{3\ellsigs^{(1)}-n}] \\
&\le \Pr[v_0\to v'] \cdot \frac{ 2^n\cdot \| \doubleP_{x|v_0\to v'} \|^2 }{2^{3\ellsigs^{(1)}-n}} \tag{$\ell_\infty$-truncation trick}\\
&\le 2^{-\ellsigs^{(1)}+1}  \cdot \Pr[v_0\to v'].\tag{$v'$ is not significant}
\end{align*}
Combining both cases completes the proof.
\end{proof}

\subsubsection{Proof of the Main Lemma}\label{sec:two-pass-transfer-core-proof}

With these preparations, we can prove this subsection's main lemma, 
which helps us ``transfer'' from $(v_1 \to v)$ to $(v_0 \to v)$. We recall the statement below.

\twoPassTransfer*



\begin{proof}
We first decompose our goal according to the event $G_{v'}$,
\begin{align}&\sum_{v_1 \in V^{(1)}_T} \sum_{v \in S_{v',v_1,i}} \Pr[(v_0 \to v) \land E(x,v)] \notag \\ 
\leq &\sum_{v_1 \in V^{(1)}_T} \sum_{v \in S_{v',v_1,i}} \Pr[G_{v'} \land (v_0 \to v) \land E(x,v)] +  \label{equ:non-stop} \\ &\sum_{v_1 \in V^{(1)}_T}  \sum_{v \in S_{v',v_1,i}} \Pr[\bar{G_{v'}} \land (v_0 \to v)]. \label{equ:stop} \end{align}

For the first term, note that $v_0\to v$ is equivalent to $(v_1\to v) \land (v' \to v_1)$. Simply by chain rule
\begin{align*}
\eqref{equ:non-stop} 
& = \sum_{v_1 \in V^{(1)}_T} \sum_{v \in S_{v',v_1,i}} \Pr[(v_1 \to v) \land E(x,v) ~ \land ~ G_{v'} \land (v' \to v_1)]\\
&= \sum_{v_1 \in V^{(1)}_T} \sum_{v \in S_{v',v_1,i}} \Pr[(v_1 \to v) \land E(x,v)] \cdot \Pr[G_{v'} \land (v' \to v_1) \mid (v_1 \to v) \land E(x,v)]\\
&\le \sum_{v_1 \in V^{(1)}_T} \sum_{v \in S_{v',v_1,i}} \Pr[(v_1 \to v) \land E(x,v)] \cdot \Pr[G_{v'} \land (v' \wt v_1) \mid (v_1 \to v) \land E(x,v)].
\end{align*}

Here the inequality holds since $v' \to v_1$ implies $v' \wt v_1$. Notice that the event $G_{v'} \land v' \wt v_1$ only depends on $x$ and $a_{>i}$ while the event $v_1 \to v \land E(x, v)$ only depends on $x$ and $a_{\le i}$. This means that for all $x' \in \{0,1\}^n$, 
$$((v_1 \to v) \land E(x, v)) \perp (G_{v'} \land (v' \wt v_1)) \mid x = x'.$$
Hence,
\begin{align*}
\eqref{equ:non-stop} 
&\le \sum_{v_1 \in V^{(1)}_T} \left(\sum_{v \in S_{v',v_1,i}} \Pr[(v_1 \to v) \land E(x,v)]\right) \cdot \max_{x' \in X} \{\Pr[G_{v'} \land (v' \wt v_1) \mid x = x']\} \\
&\le \sum_{v_1\in V^{(1)}_T} \Pr[v_0\to v'] \cdot 2^{-k} \cdot \max_{x' \in X} \{\Pr[G_{v'} \land (v' \wt v_1) \mid x = x']\}
\tag{By Assumption~\eqref{equ:assumption}}\\
&\le 
\sum_{v_1\in V^{(1)}_T} \Pr[v_0\to v'] \cdot 2^{-k} \cdot \Pr[G_{v'} \land (v' \wt v_1)] \cdot 2^{12\ellsigs^{(1)}+4} \tag{By Lemma~\ref{lemma:flat-dist} and Eq.~\eqref{eq:flat_Gv'}} \\
&\le \Pr[v_0 \to v'] \cdot 2^{12\ellsigs^{(1)} - k + 4}.
\end{align*}


For the second term, notice that $v_0 \to v$ implies $(v_0 \to v') \land (v' \wt v)$. We have
\begin{align*}
\eqref{equ:stop} &= \sum_{v_1 \in V^{(1)}_T} \sum_{v\in S_{v',v_1,i}} \Pr[\bar{G_{v'}} \land (v_0\to v) ] \\
&\le \sum_{v_1 \in V^{(1)}_T} \sum_{v\in S_{v',v_1,i}} \Pr[\bar{G_{v'}} \land (v_0\to v') \land (v'\wt v) ] \\
&\le \Pr[\bar{G_{v'}}\land (v_0\to v')] \tag{The events $\{ v'\wt v \}_{v\in V^{(2)}_{i}}$ are mutually exclusive}\\
&\le 2^{-\ellsigs^{(1)}+1}\cdot \Pr[v_0\to v']. \tag{By Claim~\ref{Claim:UnlikelyGbar}}
\end{align*}
These two parts together finish the proof of this lemma.
\end{proof}

\section{Proof of the Two-Pass Result}\label{sec:two-pass-analysis}
We are ready to prove \Cref{theo:two-pass-main-result}.

\subsection{Analyzing the Success Probability} \label{sec:analysis}

First, we would like to show that the program stops with a very small probability. We will analyze stopping due to different rules separately. Here is an outline:
\begin{itemize}
\item Stop in the first pass: see \Cref{sec:two-pass-stop-first}
\item Stop due to traversing too many high-probability edges: see \Cref{sec:two-pass-too-many-high}.
\item Stop due to traversing a bad edge in the second pass: see \Cref{sec:two-pass-bad-not-high}.
\item Stop due to significant values or $\cb$ overflow in the second pass: see \Cref{sec:two-pass-sigv-and-counter}.
\item Stop due to reaching a significant state in the second pass: see \Cref{sec:two-pass-sig-state-stop}.
\end{itemize}

Finally, we wrap up the analysis in \Cref{sec:two-pass-wrap-up}.

\subsubsection{Stop in the First Pass}\label{sec:two-pass-stop-first}

Applying \Cref{theo:GRT-extractor} and verifying the parameters, we conclude that the program stops in the first pass with probability at most $2^{-\ell}$.



\subsubsection{Too Many High-Probability Edges}\label{sec:two-pass-too-many-high}

For each vertex $v_1$ at the end of the first pass, let $E_{v_1}$ be the event the program starts from $v_1$ and traverses more than $\ellhigh = \ell$ high-probability edges. We would like to prove
\[
\sum_{v_1} \Pr[v_0\to v_1] \Pr[E_{v_1} | v_0\to v_1] \le 2^{-\frac{\kext\ell}{4}}.
\]
We observe that
\[
\begin{aligned}
\sum_{v_1} \Pr[v_0\to v_1] \Pr[E_{v_1} | v_0\to v_1]\le \sum_{v_1} \Pr[E_{v_1}].
\end{aligned}
\]
Fix one vertex $v_1$. It suffices to show that $\Pr[E_{v_1}]\le 2^{-\frac{\kext\ell}{3}}$. The desired bound follows because there are at most $2^{\frac{\kext\ell}{32}}$ many $v_1$'s.

We first observe a simple fact. 
\begin{fact}\label{fact:small-high-prob-set}
For any vertex $v$ in the second pass, it holds that $|\High(v)|\le 2^{n-\frac{\kext}{2}}$.
\end{fact}
This follows because, under some distribution over the next sample $a_{i+1}$ (namely, the distribution of $a_{i+1}$ conditioned on $v_0\to v$), each high-probability edge occurs with probability at least $2^{-n+\frac{\kext}{2}}$. As such, there can be at most $2^{n-\frac{\kext}{2}}$ such edges.

Now we are ready to upper bound $\Pr[E_{v_1}]$. The key point is that when we analyze $E_{v_1}$ without conditioning anything about $a_i$, we have that \emph{all of} $a_i$'s are uniformly random. Since $|\High(v)|\le 2^{n-\frac{\kext}{2}}$ by Fact~\ref{fact:small-high-prob-set}, over uniformly random samples, the next edge $a_{i+1}$ belongs to $\High(v)$ with probability at most $2^{-\frac{\kext}{2}}$. We can union-bound over the $T$ edges, and conclude that
\[
\Pr[E_{v_1}] \le \binom{T}{\ell} 2^{-\frac{\kext\ell}{2}} \le 2^{-\frac{\kext \ell}{3}}.
\]
Finally, summing up $\Pr[E_{v_1}]$ over $v_1$ completes the proof.

\paragraph*{Comparison with \cite{garg2019time}.} We note that our proof for high-probability edge overflow is significantly simpler than the one presented in \cite{garg2019time}, which required the use of information theory and a quite delicate calculation. Furthermore, our proof can upper bound the stopping probability by an exponentially small quantity, whereas \cite{garg2019time} can only get a constant (e.g., $\frac{1}{100}$) upper bound. Having an exponentially small stopping probability is crucial for extending our result to the multi-pass case.

\subsubsection{Stop Due to Bad Edges} \label{sec:two-pass-bad-not-high}

By the extractor property, for each non-significant $v$ in the second pass, we have $|\Bad(v)| \le 2^{n-\kext}$. Therefore,
\[
\begin{aligned}
&\Pr[\text{stop due to bad edge in the second pass}] \\
&\qquad= \sum_{i=0}^{T-1} \sum_{v\in V^{(2)}_i} \Pr[v_0\to v] \cdot \Pr[a_{i+1}\in \Bad(v)\setminus \High(v)|v_0\to v] \\
&\qquad\le \sum_{i=0}^{T-1} \Pr[v_0\to v] \cdot 2^{n-\kext} \cdot 2^{\frac{\kext}{2}-n} \\
&\qquad\le T\cdot 2^{-\frac{\kext}{2}}.
\end{aligned}
\]

\subsubsection{Stop Due to Significant Values and Bias Counters (via the Transfer Lemma)} \label{sec:two-pass-sigv-and-counter}

We show that the probability of stopping due to significant values or counter-overflow is small, using the tools developed in \Cref{sec:two-pass-potential} and \Cref{sec:two-pass-transfer}. 

Fix one $i \in [T]$. We define the ``bad event'' indicator $E:X\times V^{(2)}_i\to \{0,1\}$. For each $v\in V^{(2)}_i$, we define:
\begin{itemize}
    \item If $v$ is a significant state, we set $E(x,v) \equiv 0$ for all $x$. We will bound the probability of reaching such states in \Cref{sec:two-pass-sig-state-stop}.
    \item If $v$ not significant but $\cb(v) > \ellbias^{(2)}$, then $E(x,v)\equiv 1$ for all $x$.
    \item Otherwise, we set $E(x,v) = \mathbbm{1}[x\in \SigV(v)]$.
\end{itemize}

We would like to show that 
\[
\sum_{v\in V^{(2)}_i} \Pr[(v_0\to v) \land E(x,v)] \le 2^{-\ell+O(1)}.
\]
Once this is established, we can union-bound over $i\in [T]$ to finish the proof.

We would like to apply \Cref{lemma:Gv-magic}. Let us first establish the assumption required in \Cref{lemma:Gv-magic}. Fix $v'\in V^{(1)}_i$ and $v_1\in V^{(2)}_0$. Recall we have defined $S_{v',v_1,i}$ as the set of $v\in V^{(2)}_i$ that remembers $v'$ and $v_1$. Then, observe that
\begin{align}
&\sum_{\substack{v\in S_{v',v_1,i}\\ v \text{ not significant}}} \Pr[(v_1\to v)\land \mathbbm{1}[x\in \SigV(v)]] \notag \\
&\qquad\le \sum_{v\in S_{v',v_1,i}} \Pr[v_1\to v] \cdot 2^{2\ellsigs^{(2)} - \ellsigv} & \text{(by $\ell_{\infty}$-truncation trick)} \notag \\
&\qquad\le \Pr[v_0\to v'] \cdot 2^{2\ellsigs^{(2)} - \ellsigv} \notag \\
&\qquad\le \Pr[v_0\to v'] \cdot 2^{-14\ell}. \label{equ:sig-value-assum}
\end{align}
Also, by \Cref{cor:edge-potential-exp}, we have
\begin{align}
&\sum_{\substack{v\in S_{v',v_1,i}\\ v \text{ not significant}}} \Pr[v_1\to v\land \mathbbm{1}[\cb(v)>\ellbias^{(2)}] ] \notag \\
&\qquad\le \Pr[v_0\to v'] \cdot \Pr_{v_1\to v}[\cb(v)>\ellbias^{(2)} \mid v_0\to v'] \notag \\
&\qquad\le \Pr[v_0\to v'] \cdot 2^{\ell+1} \cdot 2^{-\ellbias^{(2)}} \notag \\
&\qquad\le \Pr[v_0\to v'] \cdot 2^{-13\ell+1}. \label{equ:count-over-assum-two-pass}
\end{align}

Overall, Eq. \eqref{equ:sig-value-assum} and Eq. \eqref{equ:count-over-assum-two-pass} 
imply that
\[
\sum_{v\in S_{v',v_1,i}} \Pr[v_1\to v\land E(x,v) ] \le \Pr[v_0\to v'] \cdot 2^{-13\ell + 2}.
\]
Now, we apply \Cref{lemma:Gv-magic} to obtain
\[
\sum_{v\in V^{(2)}_i}\Pr[(v_0\to v)\land E(x,v)] \le \sum_{v'\in V_{i}^{(1)}} 2^{-\ell+O(1)}\cdot \Pr[v_0 \to v']  \le  2^{-\ell+O(1)},
\]
as desired.

\subsubsection{Reaching a Significant State} \label{sec:two-pass-sig-state-stop}

We instantiate \Cref{lemma:sig-state-multi-pass} (see \Cref{appendix:sig-state-two-pass}) with parameters $\ellsigs^{(2)} = 18 \ell$ and $\ellbias^{(2)} = 14\ell$. For each fixed significant state $s$ in the second pass, \Cref{lemma:sig-state-multi-pass} implies that
\[
\Pr[v_0\to s] = \Pr[(v_0\to v_1)\land (v_1\to s)]\le \Pr[v_1\to s] \le 2^{-\frac{1}{2}\kext( \ellsigs^{(2)} - \ellbias^{(2)} - \ell - 5) } \le 2^{-\kext\ell}.
\]
Before we apply the modification to the program $B$, we have at most $2^{\frac{\kext\ell}{32}}$ states in each layer of the program. Hence, after the modification, there are at most 
\[
(2^{\frac{\kext\ell}{32}}\cdot T\cdot (\log|X|)^2)^3 \le 2^{\frac{\kext \ell}{8}}
\]
states in the second pass of the program. We can union-bound over all those states to finish the proof.

\subsubsection{Wrap-up: the Success Probability is Small} \label{sec:two-pass-wrap-up}

The previous sections show that the probability of stopping is small.

Denote by $\overline{G}\subseteq X\times A^T$ the union of all stopping events. Let $G$ be the complement of $\overline{G}$. We have shown that
\[
\Pr_{x,a_1,\dots, a_T}[(x,a_1,\dots, a_T)\in \overline{G}] \le 2^{-\frac{2}{3}\ell+O(1)}.
\]
Moreover, for every final vertex $v$ of the program, the event $v_1\to v$ is equivalent to $v_0\to v$, which is, in turn, equivalent to $(v_0\wt v)\land G$. Then, we get
\[
\begin{aligned}
& \| \doubleP_{x|(v_0\wt v)\land G} \|_2 \le 2^{\ellsigs^{(2)}+1} \cdot 2^{-n},\quad \text{and,} \\
& \| \doubleP_{x|(v_0\wt v) \land G} \|_\infty \le 2^{\ellsigv+1} \cdot 2^{-n}.
\end{aligned}
\]
Therefore, conditioning on $v_0\to v$, the probability of guessing $x$ correctly is exponentially small. Since this holds for every $v\in V^{(2)}_T$, we conclude the two-pass learning algorithm succeeds in learning $x$ with an exponentially small probability. This proves \Cref{theo:two-pass-main-result}.

\section{Setup for Multiple Passes}
\label{sec:multi-pass-setup}

In this section, we will set up the notation for our proof of constant-pass learning lower bounds. In \Cref{sec:multi-pass-modify} and \Cref{sec:multi-pass-stop-rules}, we extend the modification process and stop rules to multi-pass. These are more or less natural generalizations of the two-pass case. 

\subsection{Modifying the Program} \label{sec:multi-pass-modify}
First of all, we will generalize the modification in \Cref{sec:modify} to multiple passes. 
    \begin{figure}[H]
    
        \centering
        \scalebox{1.4}{
        \begin{tikzpicture}
      \node (v0) {$v_0$};
      \node[right=2cm of v0] (v'_1) {$v'_1$}; 
      \node[right=2cm of v'] (v1) {$v_1$};
      \node[below=0.2cm of v0] (v1_second) {$v_1$}; 
      \node[right=2cm of v1_second] (v'_2) {$v'_2$};
      \node at (v1|-v'_2) (v2) {$v_2$};
      \node[below=0cm of v'_2] (vd) {$\vdots$};
      \node[below=0.8cm of v'_2] (v'_q1) {$v'_{j-1}$};
      \node at (v1_second|-v'_q1) (vq1) {$v_{j-2}$};
      \node at (v2|-v'_q1) (vq) {$v_{j-1}$};
      \node[below=0.2cm of v'_q1] (v) {$v$};
      \node at (v0|-v) (vq_second) {$v_{j-1}$};
        
      \draw[->, shorten >= 2mm, shorten <= 2mm] (v0) -- (v');
      \draw[->, shorten >= 1mm, shorten <= 2mm]  (v') -- (v1);
      \draw[->, shorten >= 2mm, shorten <= 2mm]  (v1_second) -- (v'_2);
      \draw[->, shorten >= 1mm, shorten <= 2mm]  (v'_2) -- (v2);
      \draw[->, shorten >= 0mm, shorten <= 0mm]  (vq1) -- (v'_q1);
      \draw[->, shorten >= 1mm, shorten <= 0mm]  (v'_q1) -- (vq);
      \draw[->, shorten >= 2.7mm, shorten <= 1.8mm]  (vq_second) -- (v);
      
      \end{tikzpicture}}
        \caption{The computational path of first $j$ passes.}
        \label{fig:my_label}
    \end{figure}
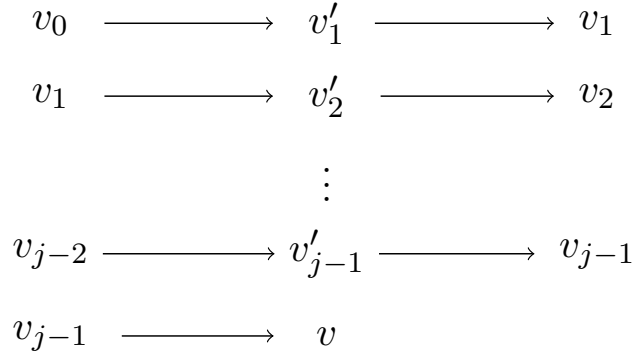

Recall the two-stage modification for two-pass learning. We will perform a similar two-stage modification process pass by pass. For the first two passes, the modification is the same as \Cref{sec:modify}. Then suppose we have already finished the modification for the first $j - 1$ passes for $j > 1$, we will perform \emph{both} of the following two stages for the $j$-th pass, before moving on to modify the $(j+1)$-th pass. 

\paragraph*{Stage 1: Remember the Previous Modified Pass} We use $v_{B_0}$ to index vertex in the $j$-th pass of the \emph{original} program.We will also $v_{B_1}$ to index vertices in the $j$-th pass of the program after stage-$1$ modification.
Similarly as before, for the $j$-th pass of the program, we will force it to remember the last state (of the modified version) of the $(j - 1)$-th pass, which we denote as $v_{j - 1} \in V_T^{(j - 1)}$. Besides this, the $j$-th pass program also runs a copy (of the modified version) of the $(j - 1)$-th pass in its memory. Now each vertex in layer $i$ of the $j$-th pass is a triple $v_{B_1} = (v'_{j -1}, v_{j - 1}, v_{B_0})$. The topology of the program is modified accordingly: If after reading the sample $(a_{i +1}, b_{i + 1})$, $v_{B_0}$ reaches $w_{B_0}$ and $v'$ reaches $w'$, we add an edge from $(v', v_{j - 1}, v_{B_0})$ to $(w', v_{j - 1}, w_{B_0})$ with label $(a_{i + 1}, b_{i + 1})$. 

For every node $v_{B_1}$ after such modification, we can uniquely determine $v'_{j - 1}$ from it. Since $v'_{j - 1}$ is a vertex in the modified version of the $(j - 1)$-th pass, from it, we can then uniquely determine $v'_{j - 2}$. In this way, $v_{B_1}$ remembers all $v'_1, v'_2, \dots, v'_{j - 1}$ as well as $v_1, v_2, \dots, v_{j - 1}$.

\paragraph*{Stage 2: Biasness and High-probability edge Counters} In the second stage, we force the program to remember two counters $\ch^{(j)}$ and $\cb^{(j)}$ for the $j$-th pass. Now each vertex will be of form $v_{B_2} = (v_{B_1}, \chj, \chj)$. 

Fix any vertex $(v_{B_1},\chj,\cbj)$. For each edge labeled $(a,b)$ in $B_1$ that goes from $v_{B_1}$ to $w_{B_1}$, we calculate the increased value of the counters, $c_1$ and $c_2$. Then we add an edge from $(v_{B_1},\chj,\cbj)$ to $(w_{B_1},c_1,c_2)$ with label $(a,b)$. After this modification, we get a new branching program $B_2$. Each vertex $v$ in $B_2$ uniquely determines $2(j - 1)$ counters from pass $2$ to pass $k$, which we denote as $\ch^{(2)}(v), \ch^{(3)}(v), \cdots, \chj(v)$ and $\cb^{(2)}(v), \ch^{(3)}(v), \cdots, \cbj(v)$. 

In the rest of the paper, we always assume that we are working with the modified program $B_2$. So, when we say ``the original prorgam'' in the future, we are always referring to $B_2$. 

We also note that our modification will blow up the width of the program from $W$ to at most $(W^{4j})$ (we assume $W\ge (10\log|X|)^2$ so that the counters do not add a significant overload).

\paragraph*{Computational Path.} Fix the starting point $v_{j - 1} \in V^{(j)}_0$ of the $j$-th pass. The computational path starting from $v_j$ is uniquely determined by $x \in X, a_1, a_2, \dots, a_T \in A$. We use $v_{j - 1} \wt \wtd{v}$ to denote the event that the (untruncated) computational path reaches $\wtd{v}$. 

For any vertex $v'_j \in V^{(j)}_i$ in the middle the $j$-th pass, we consider the subprogram consisting of layers $i, i + 1, \dots, T$ with starting vertex $v'_j$. The computational path from $v'_j$ is defined as the computational path in that subprogram. It is uniquely determined by $x \in X, a_{i+1}, a_{i + 2}, \dots, a_T \in A$. We use $v'_j \wt \wtd{v}$ to denote the event that this computational path reaches $\wtd{v}$. 

\paragraph*{Truncated Path from $v_{j - 1}$ (Informal).} Similarly, for any starting point $v_{j - 1} \in V_0^{(j)}$ of the $j$-th pass, we will use $v_{j - 1} \to \wtd{v}$ to denote the event that the compuataional path from $v_{j - 1}$ reaches $\wtd{v}$ without triggering the stopping rules.  
\subsection{Stopping Rules} \label{sec:multi-pass-stop-rules}
We also need to generalize the stopping rules in \Cref{sec:two-pass-stop-rule}. We mostly only need to replace $v_1$ by $v_{j - 1}$, but we nevertheless list them for completeness.

\paragraph*{Events of Interest.} Let $v \in V_i^{(j)}$ be a vertex in layer $i$ of the $j$-th pass. We have the following bad events.
\begin{itemize}
\item \emph{High-probability edges.} An input $a' \in A$ is of high probability if  
$$\Pr[a_{i + 1} = a' \mid v_0 \to v] \geq 2^{\frac{\kext}{2}-n}.$$
We define $\High(v)\subseteq A$ as the set of such inputs at $v$.
\item \emph{Bad edges.} Define the set of bad edges at $v$ as 
    \[
    \Bad(v) = \left\{a'\in A : \Pr_{x'\sim \doubleP_{x|v_{j - 1}\to v}}[M(a',x')=1] \not\in \left( \frac{1}{2}-2^{-\rext},\frac{1}{2}+2^{-\rext} \right) \right\}.
    \]
    \item \emph{Significant Values.} $\SigV(v)$ is the set of all $x' \in X$ such that $\doubleP_{x|v_{j - 1}\to v}(x') \ge 2^{-n} \cdot 2^{\ellsigv}$.

    Recall that $v$ also remembers the layer-$i$ vertices it has traversed during the first $j-1$ passes, which are denoted by $v'_{1},\dots, v'_{j-1}$. It will be convenient for us to define $\SigV^{(t)}(v) := \SigV(v'_{t})$ for each $t\le j-1$ and $\SigV^{(all)}(v):=\bigcup_{t=1}^{j-1} \SigV(v'_{t}) \cup \SigV(v)$.
    \item \emph{Significant States.} Finally, $v$ is called a significant state, if $\|\doubleP_{x|v_{j - 1}\to v}\|_2\ge 2^{-n}\cdot 2^{\ell^{(j)}_s}$.
\end{itemize}

\paragraph*{Counter Updates.} Whenever we traverse an edge $(a',b')$ from $v$, if $a' \in \High(v)$, we will increase the counter $\ch^{(j)}$ by $1$ and increase the counter $\cb^{(j)}$ by $$\Delta = \left\lfloor - \log\left(\Pr_{x' \sim \doubleP_{x \mid v_{j - 1} \to v}}[M(a',x') = b']\right)\right\rfloor.$$

\paragraph*{Stopping Rules.} For each pass $j$ and a starting vertex $v_{j-1}\in V^{(j)}_0$, we may consider a computational path starting with vertex $v_{j-1}$. In that case, we have the following stopping rules for the computation path. 

\begin{enumerate}
    \item Before traversing the next edge, if $x\in \SigV(v)$, we stop.
    \item When we are about to traverse an edge $(a,b)$ where $a\in \Bad(v)\setminus \High(v)$, we stop.    
    \item If the copy of the (modified) previous pass stops at $v'_{j - 1}$ due to whatever reason (including stopping due to this rule), we also stop.
    \item If $v$ is a significant state, we stop.
    \item When $\ch^{(j)}(v) > \ellhigh$, we stop.
    \item When $\cb^{(j)}(v) > \ellbias^{(j)}$, we stop.
\end{enumerate}

\paragraph*{Truncated Path from $v_{j - 1}$ (Formal).} Initially for layer $0$, $v_{j-1} \to v_{j-1}$ is always true. Suppose we have finished the definition for layer $0, 1, \dots, i$, for each $\wtd{v} \in V^{(j)}_{i+1}$, we can define $v_{j - 1}\to \wtd{v}$ recursively.

$$v_{j - 1}\to \wtd{v} \equiv \bigvee_{\wtd{u}\in V^{(j)}_i} (v_{j - 1}\to \wtd{u}) \land \left[\begin{aligned}
    &\text{From $\wtd{u}$, we traverse an edge and reach $\wtd{v}$} \\ &\text{without meeting Stopping Rules at $\wtd{u}$}
\end{aligned}\right].$$

\paragraph*{Truncated Path from $v'$.} 

For a vertex $v' \in V_i^{(j)}$, $v' \to \wtd{v}$ is defined as the event that the computational path from $v'$ reaches $\wtd{v}$ without triggering the stopping rules. Similar to the second pass, the conditional distributions $\doubleP_{x \mid v_{j - 1} \to \wtd{v}}$ in the stopping rules are still defined using truncated paths from $v_{j - 1}$ instead of that from $v'$. Formally, for each $\wtd{v} \in V_{i + 1}$,
$$v' \to \wtd{v} \equiv \bigvee_{\wtd{u} \in V^{(j)}_i} (v' \rightarrow \wtd{u}) \land \left[\begin{aligned}&\text{From $\wtd{u}$, we traverse an edge and reach $\wtd{v}$} \\  &\text{without meeting Stopping Rules at $\wtd{u}$ where} \\ &\text{(1) the distributions $\doubleP_{x \mid v_{j - 1} \to \wtd{u}}$ in the Rules}\\ & \text{are still defined w.r.t. $v_{j-1} \to \wtd{u}$}, \\ & \text{and (2) } \Pr[a_{i+1} = \text{$a'$}\mid v_0 \to \wtd{u}] \text{ in the} \\ &\text{high-probability rule is still defined w.r.t. $v_0 \to \wtd{u}$}\end{aligned}\right]$$

\subsection{Multi-Pass Learning: Main Result}\label{sec:multi-main-result-statement}

Our main result regarding multi-pass learning algorithms is as follows.

\begin{theorem}\label{theo:multi-pass-main-result}
    Suppose $M$ is a $(\kext,\ellext,\rext)$-$L2$-extractor. Let $\rlen,\ell\ge 1$ be two parameters satisfying the following inequalities.
    \begin{itemize}
        \item $\ell \le \frac{1}{10\cdot 100^{3^{q-1}}q} \min\{\ellext, \kext\}$,
        \item $\rlen \le \min\{ \frac{1}{4} \rext, \frac{\ell}{3} - 4\}$.
    \end{itemize}
    Let $B$ be a $q$-pass learning program for the learning task of $M$. Suppose $B$ has
    \begin{itemize}
        \item Width $2^{\frac{\kext\ell}{8q^{4q}}}$.
        \item Length of each pass $T \le 2^{\rlen}$.
    \end{itemize}
   Denote by $v_0$ the starting vertex of $B$.  Then, there exists an event $G\subseteq X\times A^{T}$ such that
    \[
    \Pr_{x,a_1,\dots, a_T}[(x,a_1,\dots, a_T) \notin G] \le 2^{-(\ell/2^q)}.
    \]
    and for any final vertex $v$ of $B$, it holds that
    \[
    \begin{aligned}
        & \|\doubleP_{x|(v_0\wt v)\land G}\|_{\infty} \le 2^{(100^{3^{q-1}})\ell+1} \cdot 2^{-n}, \quad \text{and} \\
        & \|\doubleP_{x|(v_0\wt v)\land G}\|_{2} \le 2^{(100^{3^{q-1}-1})\ell+1} \cdot 2^{-n}.
    \end{aligned}
    \]
\end{theorem}

We prove \Cref{theo:multi-pass-main-result} by induction on $q\ge 1$. The case of $q=2$ is established by \Cref{theo:two-pass-main-result}. In the rest of the article, we assume \Cref{theo:multi-pass-main-result} is true for $q-1$, and prove it for the case of $q$.


\subsubsection{Keeping Track of Parameters}

Again, we would like to provide the following table, summarizing all important parameters in the multi-pass lower bound proof and naming them. In particular, the meaning for the last two rows will be clear when we prove the multi-pass transfer lemma.

\begin{table}[H]
    \centering
    \begin{tabular}{c|c|c}
      Name  & Explanation & Quantity \\[5pt]
      \hline 
        $\rlen$  & $2^{\rlen}$: Length of the program & $\rlen \le \frac{1}{3}\ell - 4$      \\[5pt]
       \hline
        $\ellsigs^{(j)},~ j\in [q]$  & $j$-th Pass Significant State Threshold & $\ell\cdot 100^{3^{j-1}-1}$      \\[5pt]
       \hline
       $\ellsigv$  & Significant Value Threshold for All Passes & $\ell\cdot 100^{3^{q-1}}$   \\[5pt]
       \hline
       $\ellhigh$  & $\ch$ Threshold for All Passes & $\ell$     \\[5pt]
       \hline
       $\ellbias^{(j)},~ j\in [q]$  & $j$-th Pass $\cb^{(j)}$ Threshold & $\ell\cdot \frac{100^{3^{j-1}-1} - 1}{2}$     \\[5pt]
       \hline
       $\ellflat^{(j)},~ j\in [1,q-1]$  & $j$-th Pass Flat Threshold (for Transfer Lemma) & $\ell \cdot 100^{3^{j-1}}$     \\[5pt]
       \hline
       $\ellgood^{(j)},~ j\in [1,q-1]$  & $j$-th Goodness  Threshold (for Transfer Lemma) & $\ell \cdot 2^{q+2}\cdot 100^{2\cdot 3^{j-1}}$     \\[5pt]
       \hline
    \end{tabular}
    \caption{Parameters for the multi-pass proof}
    \label{tab:parameter-multi-pass}
\end{table}


\section{Multi-Pass Transfer Lemma} \label{sec:mutli-pass-transfer}
In this section, we will present our main technical contribution for multiple passes, the multi-pass transfer lemma. The generalization from two-pass to multi-pass turns out to be highly non-trivial. The multi-pass proof not only requires a deeper understanding of the two-pass proof but also contains new ideas that are not in the two-pass proof.

For any pass $\tau$, we say two vertices $u_1,u_2\in V^{(\tau)}$ in the $\tau$-th pass \emph{consistent}, if they remember the same list of history starting vertices $(v_1,\dots, v_{\tau})$. The main lemma of this section is as follows.

\begin{restatable}{lemma}{multiTransfer}\label{lemma:multi-pass-transfer}
Fix $i\in [T]$. Suppose we have an indicator $E:X\times V^{(j)}_i\to \{0,1\}$. For every pair of consistent $v'_{j-1}\in V^{(j-1)}_i$ and $v_{j-1} \in V^{(j-1)}_T$, define the set
\[
S_{v'_{j-1},v_{j-1},i} \coloneqq \{ v\in V^{(j)}_i : v \text{ remembers $v'_{j-1}$ and $v_{j-1}$ }\}.
\]
Assume that for every such pair $(v'_{j-1}, v_{j-1})$, it holds
\[
\sum_{v\in S_{v'_{j-1},v_{j-1},i}}\Pr[(v_{j-1}\to v) \land E(x,v)] \le 2^{-k} \Pr[v_{j-2}\to v'_{j-1}].
\]
Then, we have
\[
\sum_{v\in V^{(j)}_i} \Pr[(v_0\to v)\land E(x,v)] \le 2^{-k+\ellgood^{(j-1)} + 1} + 2^{j-\ell}.
\]

\end{restatable}

\Cref{lemma:multi-pass-transfer} is the \emph{only} technical statement from this section used to prove the main result, \Cref{theo:multi-pass-main-result} (see \Cref{sec:multi-pass-proof-of-main} for the proof). The rest of the section is devoted to proving \Cref{lemma:multi-pass-transfer}.

\subsection{Overview}
\paragraph*{A Digest the Two-Pass Proof.} Before jumping into the more complicated analysis for multiple passes, it is helpful to first understand the proof for two passes deeper. The proof contains two cases: (1) $G_{v'}$ happens. (2) $\bar{G_{v'}}$ happens. We use very similar techniques in both cases. We will use the first one as an example. 

More specifically, this is \Cref{equ:non-stop} of \Cref{lemma:Gv-magic}. We will sketch its statement and proof steps below.

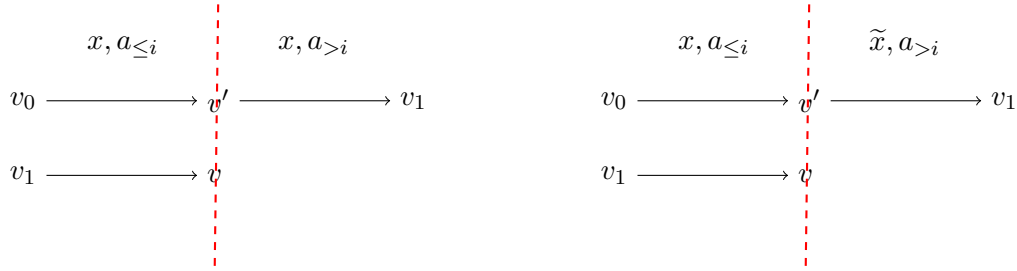
\begin{figure}[H]
    \centering
     \begin{subfigure}[b]{0.4\textwidth}
         \centering
           
    \scalebox{1}{
    \begin{tikzpicture}
  \node (v0) {$v_0$};
  \node[right=2cm of v0] (v') {$v'$}; 
  \node[right=2cm of v'] (v1) {$v_1$};
  \node[below=0.5cm of v0] (v1_second) {$v_1$}; 
  \node[right=2cm of v1_second] (v) {$v$};
  \node[above = of v'](s1) {};
  \node[below = of v](t1) {};
  \node[left = of s1](sl1) {};
  \node[below = 0.3cm of sl1] {$x, a_{\leq i}$};
  \node[right = of s1](sr1) {};
  \node[below = 0.3cm of sr1] {$x, a_{> i}$};

  \draw[->] (v0) -- (v');
  \draw[->] (v') -- (v1);
  \draw[->] (v1_second) -- (v);
  \draw[dashed, red, thick] (s1) -- (t1);
\end{tikzpicture}}
     \end{subfigure}
     \hspace{1cm}
     \begin{subfigure}[b]{0.4\textwidth}
         \centering
          
    \scalebox{1}{
    \begin{tikzpicture}
  \node (v0) {$v_0$};
  \node[right=2cm of v0] (v') {$v'$}; 
  \node[right=2cm of v'] (v1) {$v_1$};
  \node[below=0.5cm of v0] (v1_second) {$v_1$}; 
  \node[right=2cm of v1_second] (v) {$v$};
  \node[above = of v'](s1) {};
  \node[below = of v](t1) {};
  \node[left = of s1](sl1) {};
  \node[below = 0.3cm of sl1] {$x, a_{\leq i}$};
  \node[right = of s1](sr1) {};
  \node[below = 0.2cm of sr1] {$\tilde{x}, a_{> i}$};

  \draw[->] (v0) -- (v');
  \draw[->] (v') -- (v1);
  \draw[->] (v1_second) -- (v);
  \draw[dashed, red, thick] (s1) -- (t1);
\end{tikzpicture}}
     \end{subfigure}
    \caption{(a) We can divide the program into two parts $v_1 \to v$ and $v' \to v_1$. (Recall that $v_1 \to v$ also remembers $v_0 \to v'$.) (b) Our proof is actually ``resampling'' $\tilde{x}$ to decouple these two parts.}
    \label{fig:digest}
\end{figure}

In the first case, the program is divided into two parts: $v_1 \to v$ and $v' \to v_1$ (See \Cref{fig:digest} (a)). Let $E_{v_1}[x, a_{\leq i}]$ denote some event $E_{v_1}$ that only depends on $x, a_{\leq i}$. This dependency will be the only property we need. But for clarity, we would like to point out that specifically in the upper bound of \Cref{equ:non-stop}, $E_{v_1}[x, a_{\leq i}] = (v_1 \to v) \land E(x, v).$ 

  We know that for all fixed $v_1$, $\Pr_{(x,a_{\le i})}[E_{v_1}[x, a_{\leq i}] ] \leq 2^{-k}$. Our goal is to upper bound for a fixed $v'$,
\begin{equation} \label{equ:digest-ex}
\sum_{v_1} \Pr[E_{v_1}[x, a_{\leq i}] \land (G_{v'} \land (v' \wt v_1))[x, a_{> i}] ].
\end{equation}
Here, we add the bracket $[x, a_{> i}]$ to emphasize that the event $G_{v'} \land (v' \wt v_1)$ only depends on $x, a_{> i}$. 
The quantity is Eq.~\eqref{equ:digest-ex} can be interpreted as follows. We only consider the case where $G_{v'}$ happens and let the one-pass adversary $v' \to v_1$ pick the vertex $v_1$. This is the probability that $E_{v_1}$ happens for the vertex $v_1$ picked by the adversary.

Besides these, we know the posterior $\doubleP_{x \mid G_{v'} \land (v' \wt v_1)}$ is flat. Hence, by a simple Bayes' rule (\Cref{lemma:flat-dist}), we know for all $x'$,
\begin{equation}\label{equ:digest-Gv}\Pr[G_{v'} \land (v' \wt v_1) \mid x = x'] \leq 2^{O(\ell)} \cdot \Pr[G_{v'} \land (v' \wt v_1)].\end{equation}

The core of our proof is very simple:
\begin{align}
\eqref{equ:digest-ex} &= \sum_{v_1} \Pr[E_{v_1}[x, a_{\leq i}]] \cdot \Pr[ (G_{v'} \land (v' \wt v_1))[x, a_{> i}] \mid E_{v_1}[x, a_{\leq i}] ] \notag \\
&\leq \sum_{v_1} \Pr[E_{v_1}[x, a_{\leq i}]] \cdot \max_{x' \in X} \Pr[(G_{v'} \land (v' \wt v_1))[x, a_{> i}] \mid x = x'] \notag \\
&\leq 2^{O(\ell)} \cdot \Pr[E_{v_1}[x, a_{\leq i}]] \cdot \Pr[G_{v'} \land (v' \wt v_1)[x, a_{> i}]]. \label{equ:digest-final}
\end{align}

The rest of proof follows simply by plug in $\Pr[E_{v_1}[x, a_{\leq i}]] \leq 2^{-k}$. Let us stop here and appreciate what has really happened:

In the first inequality, we use the fact that the left part of the program ($E_{v_1}$) can only affect the right part ($G_{v'} \land (v' \wt v_1)$) by their common dependency on $x$. Hence we can just consider the worst case effect (letting $x$ to be worst case $x'$). 

In the second inequality, we are using \eqref{equ:digest-Gv}. Intuitively, even when we are considering such worst case $x = x'$, \eqref{equ:digest-Gv} still tells us the vertex $v_1$ picked by the adversary $v' \wt v_1$ will not differ too much from the case where $x$ is uniformly random. This is simply because an adversary with flat posterior cannot tell a specific value $x = x'$ apart from the uniform distribution $x \sim X$. 

Furthermore, we know
$$\eqref{equ:digest-final} = 2^{O(\ell)} \cdot \Pr[E_{v_1}[x, a_{\leq i}] ~~\land~~  (G_{v'} \land (v' \wt v_1))[\tilde{x}, a_{> i}]]$$
where $(G_{v'} \land (v' \wt v_1))[\tilde{x}, a_{> i}]$ is the same event but determined using an independently sampled $\tilde{x}$. 

Now one can see the whole picture: To resolve the issue caused by the common dependency on $x$, we ``resample'' a uniformly random $\tilde{x}$. Then we let the left part of \Cref{fig:digest} (a) be generated using $x, a_{\leq i}$ and let the right part be generated using $\tilde{x}, a_{> i}$. (See \Cref{fig:digest} (b).) This gives us 
$$
E_{v_1}[x, a_{\leq i}] ~~\land~~ (G_{v'} \land (v' \wt v_1))[\tilde{x}, a_{> i}],
$$
a random process that decouples the left and right parts and is very easy to analyze. However, the process we want to study is still
$$
E_{v_1}[x, a_{\leq i}] ~~\land~~ (G_{v'} \land (v' \wt v_1))[x, a_{> i}].
$$
We can relate them because when $G_{v'}$ happens, $v' \wt v_1$ cannot tell any $x = x'$ apart from the uniformly resampled $\tilde{x}$. This intuition will be heavily used in the proof for the multi-pass case.

\paragraph*{Informal Statement.} Let us understand the statement of \Cref{lemma:multi-pass-transfer} on an intuitive level. We have as assumption the following:
\begin{align*} 
\forall v'_{j - 1} \in V_i^{(j - 1)}, v_{j - 1} \in V^{(j)}_0, \quad \Pr_{x, a_{\leq i}}\left[\text{truncated path from }v_{j - 1}\text{ stops at  $V^{(j)}_i$} \ \middle\vert \ v_{j - 2} \to v'_{j - 1}\right] \leq 2^{-k}.
\end{align*}
Given this assumption, we aim to show the following about the truncated path from $v_0$:
\begin{align*}
\Pr_{x , a_1, a_2, \dots, a_T}\left[\text{truncated path from }v_0\text{ stops at  $V^{(j)}_i$} \right] \leq  2^{-k+O(\ell)}. 
\end{align*}
Here the last statement is weaker than the two-pass case for technical reasons. But it is already sufficient for our purpose.

\paragraph*{Proof Overview.} Now we will give an overview of this proof and highlight the main difficulty. Similar to the two-pass proof, we will define good event $G_{v'_{j - 1}}$ to capture whether the computational path from $v'_{j - 1}$ to $v_{j - 1}$ stops. (Roughly speaking, $G_{v'_{j - 1}}$ happens if the path does not stop.) But the specific stopping rules used here are rather complicated. For details, see the formal definition in \Cref{sec:multi-pass-good}. Our proof divides into two cases just like the proof for two-pass:

\begin{itemize}
    \item \textbf{The good event $G_{v'_{j - 1}}$ happens.} (\Cref{fig:good_happen}) \\
    \begin{figure}[H]
    
        \centering
        \scalebox{1.3}{
        \begin{tikzpicture}
      \node (v0) {$v_0$};
      \node[right=2cm of v0] (v'_1) {$v'_1$}; 
      \node[right=2cm of v'] (v1) {$v_1$};
      \node[below=0.2cm of v0] (v1_second) {$v_1$}; 
      \node[right=2cm of v1_second] (v'_2) {$v'_2$};
      \node at (v1|-v'_2) (v2) {$v_2$};
      \node[below=0cm of v'_2] (vd) {};
      \node[left=of vd] (vd1) {$\vdots$};
      \node[right=of vd] (vd2) {$\vdots$};
      \node[below=0.8cm of v'_2] (v'_q1) {$v'_{j-1}$};
      \node at (v1_second|-v'_q1) (vq1) {$v_{j-2}$};
      \node at (v2|-v'_q1) (vq) {$v_{j-1}$};
      \node[below=0.2cm of v'_q1] (v) {$v$};
      \node at (v0|-v) (vq_second) {$v_{j-1}$};
      \node[above = 0.4cm of v'_1](s1){};
      \node[below = 0.4cm of v](t1){};
      \node[left = of s1](s1l){};
      \node[below = 0.2cm of s1l](lx){$x, a_{\leq i}$};
      \node[right = of s1](s1r){};
      \node[below = 0.1cm of s1r](rx){$\tilde{x}, a_{> i}$};
      
      \draw[->, shorten >= 2mm, shorten <= 2mm] (v0) -- (v');
      \draw[->, shorten >= 1mm, shorten <= 2mm]  (v') -- (v1);
      \draw[->, shorten >= 2mm, shorten <= 2mm]  (v1_second) -- (v'_2);
      \draw[->, shorten >= 1mm, shorten <= 2mm]  (v'_2) -- (v2);
      \draw[->, shorten >= 0mm, shorten <= 0mm]  (vq1) -- (v'_q1);
      \draw[->, shorten >= 1mm, shorten <= 0mm]  (v'_q1) -- (vq);
      \draw[->, shorten >= 2.7mm, shorten <= 1.8mm]  (vq_second) -- (v);
      \draw[-, thick, dashed, red] (s1) -- (t1);
      \end{tikzpicture}}
        \caption{Decouple $v_{j - 1} \to v$ and $v'_{j - 1} \to v_{j - 1}$ for $v \in V_j^{(2)}$.}
        \label{fig:good_happen}
    \end{figure}
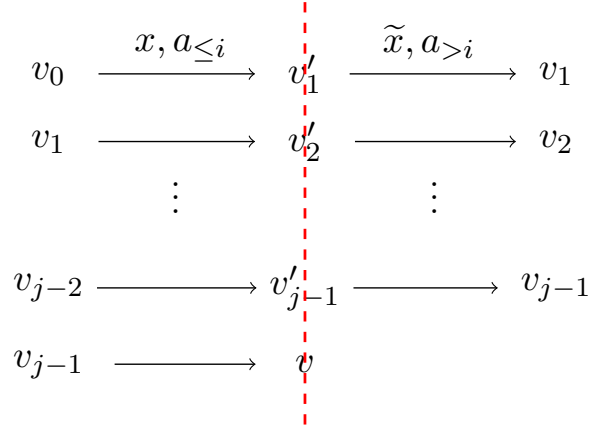

    In this case, we are facing the same adaptivity issue as in the two-pass case because both the truncated path $v_{j - 1}\to v$
    and $v'_{j - 1} \to v_{j - 1}$ depend on $x$. The strategy for two-pass generalizes quite easily. Since the posterior distribution $\doubleP_{x \mid G_{v'_{j - 1}}\land\; (v'_{j - 1} \wt v_{j - 1})}$ is flat by our stopping rules, when $G_{v'_{j - 1}}$ happens, $v'_{j - 1} \wt v_{j - 1}$ cannot distinguish any specific value $x = x'$ and the resampled $\tilde{x}$. Hence we can decouple these two parts in the same way as before.
 
    \item \textbf{The complement event $\bar{G_{v'_{j - 1}}}$ happens:} This is the more challenging case. Recall that in this case, we have the following adaptivity issue: Although for all $v'_{j - 1}$, $\Pr_{x, a_{> i}}\left[\bar{G_{v'_{j - 1}}}\right]$ is exponentially small, $v'_{j - 1}$ is picked by the process $v_0 \to v'_{j - 1}$. Since $\bar{G_{v'_{j - 1}}}$ is determined by $x, a_{> i}$, and $v_0 \to v'_{j - 1}$ is determined by $x, a_{\leq i}, a_{> i}$ (thus both events depend on $a_{>i}$), an almighty adversary might pick the $v'_{j - 1}$ for which $\bar{G_{v'_{j - 1}}}$ happens.

 In the two-pass proof, this was not an issue, and we crucially relied on the fact that $v_0 \to v'_1$ is a one-pass algorithm that \emph{only} depends on $x, a_{\leq i}$.
    But now $v_0 \to v'_{j - 1}$ depends not only on $x, a_{\leq i}$, but also on $a_{> i}$. To prove this lemma, we necessarily need some new ideas. 
    
    \begin{figure}
        \centering
        \scalebox{1.3}{
        \begin{tikzpicture}
      \node (v0) {$v_0$};
      \node[right=2cm of v0] (v'_1) {$v'_1$}; 
      \node[right=2cm of v'] (v1) {$v_1$};
      \node[below=0.2cm of v0] (v1_second) {$v_1$}; 
      \node[right=2cm of v1_second] (v'_2) {$v'_2$};
      \node at (v1|-v'_2) (v2) {$v_2$};
      \node[below=0cm of v'_2] (vd) {};
      \node[left=of vd] (vd1) {$\vdots$};
      \node[right=of vd] (vd2) {$\vdots$};
      \node[below=0.8cm of v'_2] (v'_q1) {$v'_{j-2}$};
      \node at (v1_second|-v'_q1) (vq1) {$v_{j-3}$};
      \node at (v2|-v'_q1) (vq) {$v_{j-2}$};
      \node[below=0.2cm of v'_q1] (v) {$v'_{j - 1}$};
      \node at (v0|-v) (vq_second) {$v_{j-2}$};
      \node at (vq|-vq_second) (v_final) {$v_{j-1}$};
      \node[above = 0.4cm of v'_1](s1){};
      \node[below = 0.4cm of v](t1){};
      \node[left = of s1](s1l){};
      \node[below = -0.15cm of s1l](lx){$\tilde{x}, a_{\leq i}$};
      \node[right = of s1](s1r){};
      \node[below = 0cm of s1r](rx){$x, a_{> i}$};
      
      \draw[->, shorten >= 2mm, shorten <= 2mm] (v0) -- (v');
      \draw[->, shorten >= 1mm, shorten <= 2mm]  (v') -- (v1);
      \draw[->, shorten >= 2mm, shorten <= 2mm]  (v1_second) -- (v'_2);
      \draw[->, shorten >= 1mm, shorten <= 2mm]  (v'_2) -- (v2);
      \draw[->, shorten >= 0mm, shorten <= 0mm]  (vq1) -- (v'_q1);
      \draw[-, thick, dashed, red] (s1) -- (t1);
      \draw[->, shorten >= 0mm, shorten <= 0mm]  (v'_q1) -- (vq);
      \draw[->, shorten >= 0mm, shorten <= 0mm]  (vq_second) -- (v);
      \draw[->, shorten >= 0mm, shorten <= 0mm]  (v) -- (v_final);

      \draw[orange,thick,dashed] ($(v0.north west)+(-0.2,0)$)  rectangle ($(v'.south east)+(0.2,-0)$); 
      \node[left = 0.5cm of v0, orange] (w1) {Pre-processing.};
      \draw[blue,thick,dashed] ($(v1_second.north west)+(-0.2,0)$)  rectangle ($(v.south east)+(0,-0)$); 
      \node[below = 1cm of w1, blue] (w2) {Post-processing.};
      
      \end{tikzpicture}}
        \caption{After decoupling,  we will view the left part as pre-processing/post-processing stages.}
        \label{fig:multi-pass-pre-post}
    \end{figure}
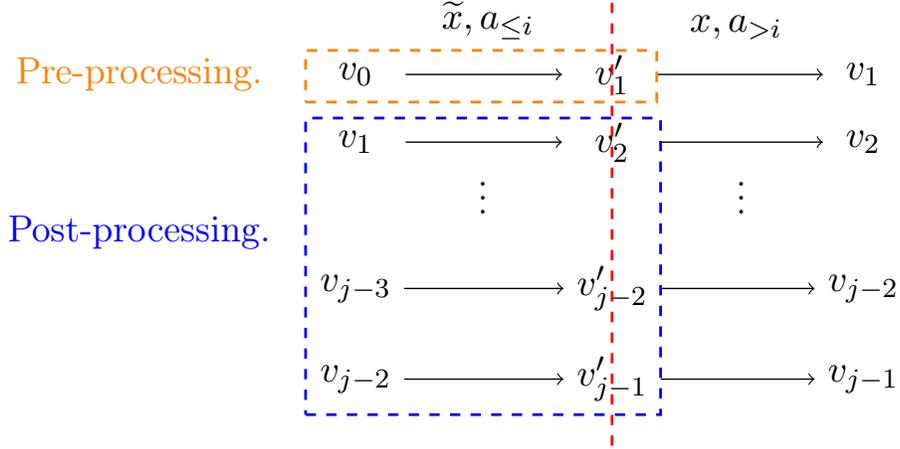

    \end{itemize}

\paragraph*{$\bar{G_{v'_{j - 1}}}$ rarely happens.} We now discuss in detail how we handle the second case. Recall that we are proving the transfer lemma for the $j$-th pass of the program.
The vertex $v'_{j - 1}$ is in the $(j - 1)$-th pass of the program. Our proof contains the following steps:
    \begin{enumerate}
        \item \textbf{(Decoupling.)} Observe that $v_0 \to v'_{j - 1}$ is equivalent to $(v_{j - 2} \to v'_{j - 1}) \land (v'_{j - 2} \to v_{j - 2})$. We can assume that $\doubleP_{x \mid v_{j - 2} \to v'_{j - 1}}$ is flat because, by induction, we can truncate significant values in the $(j-1)$-th pass without increasing stopping probability too much. Hence we can resample a hidden vector $\tilde{x} \in X$ and ``generate'' the left part $v_{j - 2} \to v'_{j - 1}$ using this resampled vector $\tilde{x}$. (See \Cref{fig:multi-pass-pre-post}.)

        \item \textbf{(Reduction to $j - 1$ Passes.)}
        \begin{itemize}
            \item \textbf{(A Toy Case.)} Let us first consider a toy case: Suppose for all $0 \leq t \leq j - 2$, $v'_{t+1} = v_{t}$, i.e., the left part of this program is trivial. Then the first $(j-1)$ passes of our program reduce to a $(j - 1)$-pass program with starting vertex $v'_1 = v_0$ that depends only on $a_{>i}$. Denote the program as $B_{>i}$. We define $G_{v'_{j-1}}$ to capture whether $B_{>i}$ stops in the $(j-1)$-th pass \emph{starting} from $v'_{j-1}$. 
            \\
            
            By induction, we assume \Cref{theo:multi-pass-main-result} has been proven for all $(j - 1)$-pass programs. This means
            \[
            \sum_{v'_{j-1}} \Pr[(v'_1 \to v'_{j-1})\land \bar{G_{v'_{j - 1}}}] \le 2^{-\Omega(\ell)}.
            \]
            That is to say, although $v'_{j - 1}$ is still picked by the first $(j - 2)$-passes, since all together this is a $(j - 1)$-pass program, we know that the probability of stopping is exponentially small. \\ 
        \item \textbf{(The General Case.)} Notice that although the left part of the program may not be trivial in the general case, it is still completely independent of $x, a_{> i}$. \\
        
        If we view the right part as a $(j - 1)$-pass program, its inputs are $a_{> i}$ and $b_{> i}$. Then $v_0 \to v'_1$ is just a pre-processing stage before the first pass, while for all $1 \leq t\leq j - 2$, $v_{t} \to v'_{t + 1}$ is a post-processing stage between the $t$-th pass and the $(t + 1)$-th pass. These stages are independent of the inputs $a_{> i}, b_{> i}$. They depend on the randomized $\tilde{x}, a_{\leq i}$, which can be thought of as the internal randomness of the program. \\
        
        Naturally, a branching program with such extra randomized pre-processing/post-processing stages that are independent of the inputs can be simulated by a deterministic program by fixing the randomness of these stages. \\

        \item \textbf{(Reduction to the Toy Case.)} We will now fix the randomness of $\tilde{x}, a_{\leq i}$. For any fixed $\tilde{x}, a_{\leq i}$ and $0 \leq t \leq j - 2$, the transition from $v_{t}$ to $v'_{t+1}$ is deterministic. Moreover, by our modification in \Cref{sec:multi-pass-modify}, $v'_{t + 1}$ remembers $v_{t}$. Hence this transition is injective. Then the left part performs essentially the same functionality as a trivial program. 
        \end{itemize}        
    \end{enumerate}

\subsection{Extending Good Events to Multi-Pass} \label{sec:multi-pass-good}

Roughly speaking, the good event $G_{v'_{j - 1}}$ happens when the path from $v'_{j - 1}$ to $v_j$ does not stops. In this section, we will specify the stopping rules used in this definition. In our proof, we need two properties from $G_{v'_{j - 1}}$:
\begin{itemize}
    \item When $G_{v'_{j - 1}}$ happens, the posterior distribution $\doubleP_{x \mid (G_{v'_{j-1}} \land\; (v'_{j- 1} \wt v_{j-1})) }$ should be flat. So that we can decouple the left and the right parts of the program. 
    \item For $v'_{j - 1}$ that is adaptively picked by ``the first $(j - 2)$-passes with pre/post-processing'', i.e., $v_0 \to v'_{j - 1}$, the probability that $\bar{G_{v'_{j - 1}}}$ happens should be small. 
\end{itemize}
The remaining task is finding the correct set of stopping rules that guarantee these two properties. Compared with the good event $G_{v'}$ defined in the two-pass transfer lemma, there are two complications:
\begin{itemize}
    \item In the two-pass version, since $v'$ is a vertex in the first pass, it has no associated counters. But $v'_{j - 1}$ has counters. We need to specify the counter-related stopping rules carefully. 

    Specifically, we will define the stopping rules using the new counters that start counting from $v'_{j - 1}$, instead of the old counters attached to $v'_{j - 1}$, which started counting from $v_{j - 2}$. 
    
    This resembles the two-pass case: The posterior distributions $\doubleP_{x \mid v'_1 \to \tilde{v}}$ (for first pass vertex $\tilde{v}$) used in the stopping rules of $G_{v'}$, are defined w.r.t. $v'_1 \to \tilde{v}$ instead of $v_0 \to \tilde{v}$.
    
    \item In the two-pass version, we only need that for all fixed $v'$, the probability of $\bar{G_{v'}}$ is small. But here, $v'_{j - 1}$ is picked by a (computationally bounded) adaptive procedure, and we still require the probability of $\bar{G_{v'_{j - 1}}}$ to be small. 

    For example when $j = 3$, the whole process $v_0 \to v'_{j - 1} \to v_{j - 1}$ contains two passes. Hence same as our analysis for two-pass algorithms, when defining $G_{v'_{j - 1}}$, we cannot stop on all bad edges. Otherwise, the probability of $\bar{G_{v'_{j - 1}}}$ would be too large. Hence it is necessary to introduce stopping rules related to high-probability-edges and counters. But now, as the program has pre/post-processing stages, even the definition of high-probability edges requires a little more care.  
\end{itemize}

\subsubsection{Probabilistic Subprogram}  Suppose $v'_{j - 1} \in V^{(j - 1)}_i$. To define the good event $G_{v'_{j - 1}}$, similar to the two-pass setting, we will consider a subprogram $B_{> i}$ of the original program $B$. But the definition here is more complicated: It involves resampling a uniform $\tilde{x} \in X$ and ``decoupling'' the left part and the right part of \Cref{fig:multi-pass-pre-post}.

\begin{definition}[Probabilistic Subprogram $B_{> i}$] Given the original program $B$, for any pass $j$ and layer $i$, define the $(j-1)$-pass \textit{probabilistic} subprogram $B_{>i}$ as follows:
\begin{itemize}
    \item The input for $B_{> i}$ is $(a_{i + 1}, b_{i + 1}), \dots, (a_T, b_T) \in A \times \{-1, 1\}$.
    \item The set of vertices of $B_{> i}$ is 
\[
V(B_{> i}) = \bigcup_{\tau=1}^{j-1} \bigcup_{t=i}^T V^{(\tau)}_{t}.
\]
Namely, it contains all vertices of $B$ that are in the last $(T-i+1)$ layers of the first $(j-1)$ passes. Note here the $i$-th layer in the original program is the starting layer of $B_{>i}$. 
    \item $B_{> i}$ has \emph{internal randomness}: Uniformly sampled $\tilde{x} \in X$ and $a_{\leq i} \in A^{i}$.
    \item The edges \emph{within} each pass remain unchanged. Namely, for each $\tau \in [j - 1]$ and $t\in [i, T-1]$, we keep all the edges between $V^{(\tau)}_{t}$ and $V^{(\tau)}_{t+1}$.
    \item The edges \emph{between} passes have no labels. Each vertex at the last layer of that pass has a unique outgoing edge that is determined by \emph{internal randomness} $\tilde{x}$ and $a_{\leq i}$. Specifically, for each $\tau \in [2,j-1]$, we connect $v_{\tau - 1} \in V^{(\tau - 1)}_T$ to a vertex $v'_{\tau} \in V_i^{(\tau)}$ if and only if the following holds:
    \begin{itemize}
        \item Let $$\tilde{b}_1 = M(a_1, \tilde{x}), \tilde{b}_2 = M(a_2, \tilde{x}), \dots, \tilde{b}_{i} = M(a_{i}, \tilde{x}).$$ In the $\tau$-th path of the original branching program $B$, the computational path starting from $v_{\tau - 1}$ on partial input $$(a_1, \tilde{b}_1), (a_2, \tilde{b}_2), \dots, (a_{i}, \tilde{b}_{i})$$ arrives at $v'_\tau$.  (Note that this corresponds to the post-processing stage in \Cref{fig:multi-pass-pre-post}.)
    \end{itemize}
    \item Finally, the starting vertex of $B_{> i}$ is the vertex $v'_1\in V^{(1)}_i$, such that in the original program $B$, the computational path starting from $v_0$ on partial input $$(a_1, \tilde{b}_1), (a_2, \tilde{b}_2), \dots, (a_{i}, \tilde{b}_{i})$$ arrives at $v'_1$. (Note this corresponds to the pre-processing stage in \Cref{fig:multi-pass-pre-post}.)
\end{itemize}
Crucially, the starting vertex and the edges between two adjacent passes only depend on the \emph{internal randomness}, but not on the input to $B_{> i}$.    
\end{definition}

We will need the following observation about the probabilistic subprogram.

\begin{claim}\label{claim:probab-B-connection}
For each $v'_{j-1}\in V^{(j-1)}_i$, we observe that
\begin{align}
\Pr[B_{> i}\text{ reaches } v'_{j-1}] = \Pr[v_{j-2}\wt v'_{j-1}] \cdot \Pr[v'_{j-2}\wt v_{j-2}].\label{equ:probabilisic-B-connection}
\end{align}
\end{claim}

\begin{proof}
To see this, note that $B_{> i}$ reaches $v'_{j-1}$ if and only if both of the following events happen:
\begin{enumerate}
\item $B_{> i}$ has $v'_1$ (which is remembered in $v'_{j-1}$) as the starting vertex. For each integer $\tau\in [2,j-1]$, the edge between passes moves from $v_{\tau - 1}$ to $v'_{\tau}$ (which are also remembered in $v'_{j-1}$) . This event only depends on the internal randomness $\tilde{x}, a_{< i}$.
It happens with probability 
\begin{align*}
\Pr_{\tilde{x}, a_{< i}}\left [ \bigwedge_{\tau \in [j - 1]} v_{\tau - 1} \wt v'_{\tau} \right] = \Pr_{\tilde{x}, a_{< i}}[v_{j-2}\wt v'_{j - 1}].
\end{align*}
here the equality follows from the fact that every pass remembers the previous passes. 
\item For each $\tau \in [j-2]$, $B_{> i}$ moves from $v'_{\tau}$ to $v_{\tau}$. This event only depends on $x, a_{\geq i}$ and happens with probability
\begin{align*}
\Pr_{x, a_{\geq i}}\left[\bigwedge_{\tau \in [j-2]} (v'_{\tau}\wt v_{\tau})\right] = \Pr[v'_{j - 2} \wt v_{j - 2}].
\end{align*}
here the equality follows from the fact that every pass remembers the previous passes. 
\end{enumerate}
Finally, observe that these two events are independent of each other, because the first one depends solely on $\tilde{x}, a_{< i}$, while the second one depends solely on $x, a_{\geq i}$. This completes the proof.
\end{proof}

\subsubsection{Good Events $G_{v'_{j - 1}}$} \label{sec:multi-transfer-good}

Informally speaking, we will define $G_{v'_{j - 1}}$ to be the event that the truncated path in $B_{> i}$ from $v'_{j - 1} \in V^{(j - 1)}_i$ does not stop in pass $j - 1$. For clarity, we will now fully expand this definition. 

\paragraph*{Modifying the probabilistic subprogram} First, to define counter-related stopping rules. We will apply the modification stages in \Cref{sec:multi-pass-modify} to $B_{> i}$ and get a new probabilistic program $\widetilde{B}_{> i}$. Compared with the modification for deterministic branching programs, we only have to make two changes for the modification stages: 
\begin{itemize}
    \item Now the vertices in the first pass also need to remember the starting vertex $v'_1$ because it is no longer fixed for probabilistic programs.
    \item Additional to the starting vertex $v'_\tau$ of pass $\tau$, every vertex in pass $\tau$ will also remember the last vertex from the previous pass, $v_{\tau - 1}$. 
    
    (For deterministic branching programs, the starting vertex of this pass is always the same as the last vertex of the previous pass. But this is not the case for probabilistic programs. ) 
\end{itemize} 

In this modification, for a vertex $v$ in the $\tau$-th pass, the original counters associated with it are ignored.
Instead, we add new counters $\widetilde{\ch}^{(\tau)}(v)$ and $\widetilde{\cb}^{(\tau)}(v)$ to it. 

Also, it remembers the new counters from previous passes, $
\widetilde{\ch}^{(2)}(v), \dots, \widetilde{\ch}^{(\tau - 1)}(v)$ and $
\widetilde{\cb}^{(2)}(v),\dots, \widetilde{\cb}^{(\tau - 1)}(v)$.

\paragraph*{Stopping Rules.} The stopping rules are the same as \Cref{sec:multi-pass-stop-rules}. We highlight all the changes:
\begin{itemize}
    \item The conditional distributions $\doubleP_{x \mid v_{\tau - 1} \to v}$ ($v$ is in pass $\tau$) are replaced by their analogue $\doubleP_{x \mid v'_\tau \to v}$ for $\widetilde{B}_{> i}$. (This affects both the bad events and the counter updates.)
    \item The high-probability edges $\High(v)$ for vertex $v$ in layer $t$ is now defined using the probability $$\Pr[a_{t + 1} = a \mid \widetilde{B}_{>i} \text{ reaches } v \text{ without stopping}].$$ This is the analogue of $\Pr[a_{t + 1} = a \mid v_0 \to v]$ for this probabilistic program $\widetilde{B}_{> i}$.
    \item The counter-related stopping rules are w.r.t. new counters $\widetilde{\ch}^{(\tau)}(v)$ and $\widetilde{\cb}^{(\tau)}(v)$.
\end{itemize}

Since these stop rules are defined only for the probabilistic subprogram $\tilde{B}_{> i}$, in the definition of $G_{v'_{j - 1}}$, we will never stop in the middle of the pre/post-processing stages ($v_{\tau - 1} \wt v'_\tau$). In the $\tau$-th pass, we will only stop during $v'_{\tau} \wt v_\tau$. 

We also need to specify the \emph{thresholds} for defining the stopping rules. 
Given the $(j-1)$-pass program $\tilde{B}_{>i}$, we set relevant stopping thresholds as specified in \Cref{tab:parameter-multi-pass} with ``$\ell$'' set to $2^{j+2}\cdot \ellflat^{(j-1)}$.
\paragraph*{Good Event Definition.} The good event $G_{v'_{j - 1}}  \subseteq \{(x, a_{> i}) \mid x \in X, a_{> i} \in A^{T - i}\}$ is defined as the event that starting from vertex $v'_{j - 1}$, the computational path determined by $x, a_{> i}$ reaches $V^{(j - 1)}_T$ without triggering any of these stopping rules.  

\paragraph{Fixing Internal Randomness} In our definition of $B_{> i}$, we sampled its internal randomness $\tilde{x}, a_{\leq i}$ uniformly at random. This is crucial for \Cref{claim:probab-B-connection} to hold. However, we will now show that the distribution of $\tilde{x}, a_{\leq i}$ actually does not matter for the definition of $G_{v'_{j - 1}}$. This means, even if we completely fix the internal randomness $\tilde{x}, a_{\leq i}$ and get a deterministic program, the good event $G_{v'_{j - 1}}$ defined w.r.t. this deterministic program is still the same subset of $\{(x, a_{> i}) \mid x \in X, a_{> i} \in A^{T - i}\}$ as the good event defined w.r.t. the probabilistic program with uniform internal randomness.

We will now sketch the main idea, and the formal proof will be given in \Cref{appendix:multi-pass}. The idea is the following: We can view the left part of \Cref{fig:multi-pass-pre-post} as Alice and the right part as Bob. Then at the beginning of the $t$-th pass, the vertex $v'_{t - 1}$ can be seen as a message sent from Alice to Bob. Similarly, $v_t$ can be seen as a message sent from Bob to Alice, as shown in the following figure.

\begin{figure}[H]
    \centering
   
\begin{tikzpicture}[
  rect/.style={rectangle, draw, text centered, minimum height=11em, minimum width=5em},
  arrow/.style={->, >=stealth, thick}
]
  \node[rect] (alice) {Alice};
  \node[below = -1.5cm of alice](xaA) {$\tilde{x},a_{\leq i}$};
  \node[rect, right=4cm of alice] (bob) {Bob};
  \node[below = -1.5cm of bob](xaB) {$x,a_{> i}$};

  \def\numMessages{3}

  \foreach \i [evaluate=\i as \y using 1-(\i-0.5)*1.7/(\numMessages+1)] in {1,...,\numMessages}{
    \coordinate (alice_m\i) at ($(alice.east)!\y!(alice.north east)$);
    \coordinate (bob_m\i) at ($(bob.west)!\y!(bob.north west)$);
    \pgfmathsetmacro\j{int((\i + 1) / 2)};
    \ifodd\i
      \draw[arrow] ([xshift=0.3cm]alice_m\i) -- node[midway, above] {$v'_\j$} ([xshift=-0.3cm]bob_m\i);
    \else
      \draw[arrow] ([xshift=-0.3cm]bob_m\i) -- node[midway, above] {$v_\j$} ([xshift=0.3cm]alice_m\i);
    \fi
  }

  \node at ($(alice.east)!0.5!(bob.west)+ (0, -0.5cm)$) (vdots) {$\vdots$};

  \coordinate (alice_mn) at ($(alice.east)!0.6!(alice.south east)+ (0, -0.4cm)$);
  \coordinate (bob_mn) at ($(bob.west)!0.6!(bob.south west)+ (0, -0.4cm)$);
  \draw[arrow] ([xshift=0.3cm]alice_mn) -- node[midway, above] {$v'_j$} ([xshift=-0.3cm]bob_mn);
\end{tikzpicture}

    \caption{The messages between Alice and Bob.}
    \label{fig:communication}
\end{figure}
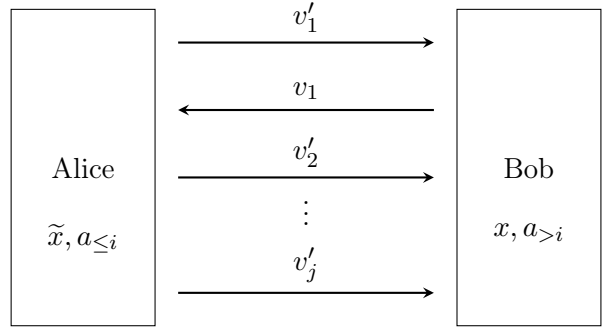

Consider the following communication game: 
\begin{itemize}
    \item Alice holds her private input $\tilde{x}, a_{\leq i}$ sampled from some distribution, while Bob holds his private input $x, a_{> i}$ sampled uniformly at random from $X \times A^{T - i}$.
    \item Each time Alice receives a message $v_{t - 1}$ (where $t \in [j]$), she computes $v'_t \gets f_t(v_t, \tilde{x}, a_{\leq i})$ for some function $f_t$ and sends it to Bob.
    \item  Then Bob computes $v_t \gets g_t(v'_t, x, a_{ > i})$ for some function $g_t$ and sends it to Alice.
\end{itemize}  
Our claim is that conditioning on the entire transcript $\Pi$, the posterior distribution of Bob's private input, $\Pr[x, a_{> i} \mid \Pi]$, is the same no matter how Alice's private input is sampled. To see this, notice that once we have conditioned on the entire transcript, Alice's and Bob's input would be independent of each other (because the set of inputs consistent with the trasncript always forms a combinatorial rectangle).

We will first explain how this model corresponds to our program $\widetilde{B}_{> i}$.
\begin{itemize}
    \item In $\widetilde{B}_{> i}$, the left part of \Cref{fig:multi-pass-pre-post} depends on $\tilde{x}, a_{\le i}$ which is the internal randomness. The right part of \Cref{fig:multi-pass-pre-post} depends on $x, a_{> i}$.
    \item When the left part receives a vertex $v_{t - 1}$ (where $t \in [j]$), it will determine $v'_t$ by simulating the computational path from $v_{t - 1}$ using $\tilde{x}, a_{\leq i}$. This corresponds to $v'_t \gets f_t(v_t, \tilde{x}, a_{\leq i})$.  
    \item When the right part receives a vertex $v'_t$, it will simulate the truncated path from $v_t$ using $x, a_{> i}$. The computational path from $v_t$ is independent of how $\tilde{x}, a_{\leq i}$ are sampled. 
    
    For the truncated path, we will have to show that the stopping rules are independent of $\tilde{x}, a_{\leq i}$ as well. Once we have shown this, the right part would correspond to $v_t \gets g_t(v'_t, x, a_{> i})$. 
    
    \item By our new modification, the vertex $v_1$ remembers the starting vertex $v'_1$, and the vertex $v'_2$ remembers the last vertex from the previous pass which is $v_1$. In general, every message contains the whole prefix of the transcript. 
\end{itemize}
So now we need to show that the stopping rules are independent of how $\tilde{x}, a_{\leq i}$ are sampled. Once we have shown this, we would prove (1) a strict correspondence between this communication game and $\widetilde{B}_{> i}$; (2) $G_{v'_{j - 1}}$ is independent of $\tilde{x}, a_{\leq i}$, since $G_{v'_{j - 1}}$ is a statement about whether Bob stops in the $(j - 1)$-th pass. 

Intuitively, the stopping rules are independent of how $\tilde{x}, a_{\leq i}$ is sampled for the following reason: For any pass $j$, any vertex $v'_j\in V^{(j)}_{i}$ remembers the whole prefix of the transcript (i.e., $v_{j-1},\dots, v_1$ and $v'_{j-2},\dots, v'_{1}$). By our claim above, the distribution of $(x, a_{> i})\mid v'_j$ is independent of how $\tilde{x}, a_{\leq i}$ are sampled. Then by one induction over the layers, we can show that the conditional probabilities used in the stopping rules, namely,
\[
\doubleP_{x \mid v'_j \to v} \quad \text{ and } \quad
\Pr[a_{t + 1} = a' \mid \widetilde{B}_{>i} \text{ reaches } v \text{ without stopping}]
\]
are independent of how $\tilde{x}, a_{\leq i}$ are sampled.

For the details of this proof, see \Cref{appendix:multi-pass}. 

\paragraph*{$G_{v'}$ implies flat distribution.} 

We need the following useful claim.

\begin{claim}\label{claim:multi-transfer-Gv-useful}
For every $v_{j-1}\in V^{(j-1)}_T$ in the \emph{original program}, it holds that
\[
\left\| \doubleP_{x|(v'_{j-1}\wt v_{j-1})\land G_{v'_{j-1}}} \right\|_{\infty} \le 2^{\ellgood^{(j-1)}+1}\cdot 2^{-n},
\]
where we recall that $\ellgood^{(j-1)}$ is defined as (see \Cref{tab:parameter-multi-pass}):
\[
\ellgood^{(j-1)} = (100^{3^{(j-2)}}) \cdot (2^{j+2}\cdot \ellflat^{(j-1)}).
\]
\end{claim}

\begin{proof} 
For every vertex $v_{j-1}$ in the final layer of $\tilde{B}_{>i}$, we fix the internal randomness of $\tilde{B}_{>i}$ so that 
$\Pr[\hat{B}_{>i} \text{ reaches } v_{j-1}] > 0$. Let $\hat{B}_{>i}$ denote the resulting deterministic program.
Then, over uniformly random $x,a_{>i}$, the event $\mathbbm{1}[\hat{B}_{>i} \text{ reaches } v_{j-1}]$ is equivalent to
\[
(v'_{j-1}\wt v_{j-1}) \land G_{v'_{j-1}}.
\]
Applying \Cref{theo:multi-pass-main-result} on $\hat{B}_{>i}$ with parameter $2^{j+2}\cdot \ellflat^{(j-1)}$, we know that
\[
\| \doubleP_{x|(v'_{j-1}\wt v_{j-1}) \land G_{v'_{j-1}}} \|_{\infty} \le 2^{\ellgood^{(j-1)}+1} \cdot 2^{-n}.
\]

We have shown the claimed bound holds for vertices in the program $\tilde{B}_{>i}$ (where the new counters $\widetilde{\cb},\widetilde{\ch}$ have been added). This also implies the same bound for the original program's $v_{j-1}$. Because the distribution of $\doubleP_{x|(v'_{j-1}\wt v_{j-1})\land G_{v'_{j-1}}}$ is a convex combination of $\left\{\doubleP_{x|(v'_{j-1}\wt v)\land G_{v'_{j-1}}}\right\}_v$ where $v$ runs over all vertices in $\tilde{B}_{>i}$ that are consistent with $v_{j-1}$. 
\end{proof}




\subsection{Proof of the Multi-Pass Transfer Lemma}

Our proof will be divided into two cases based on whether the good event $G_{v'_{j - 1}}$ (defined as in \Cref{sec:multi-pass-good}) happens. 

\subsubsection{$G_{v'_{j-1}}$ Usually Happens}

First, we would like to show the following lemma, saying that $\bar{G_{v'_{j - 1}}}$ usually does not happen.

\begin{lemma}\label{lemma:pass-elimination}
    For every $i\in [T]$, it holds that
    \[
    \sum_{\substack{v'_{j-1}\in V^{(j-1)}_i\\ v'_{j-1} \text{ not significant}}} \Pr[(v_0\to v'_{j-1})\land \bar{G_{v'_{j - 1}}}] \le 2^{j} \cdot 2^{-\ell}.
    \]
\end{lemma}

\newcommand{\Lar}{\mathrm{Large}}

\begin{proof}
For each $v'_{j-1}$, define
\[
{\Lar}(v'_{j-1}) = \{ x'\in X: \doubleP_{x|v_{j-2}\to v'_{j-1}} > 2^{\ellflat^{(j-1)}} \cdot 2^{-n} \}.
\]
We have
\begin{align}
&~~~~ \sum_{v'_{j-1}\in V^{(j-1)}_i} \Pr[v_0\to v'_{j-1} \land \bar{G_{v'_{j - 1}}}]  \notag \\
& \le \sum_{v'_{j-1}\in V^{(j-1)}_i} \Pr[(v_{0}\to v'_{j-1}) \land (x\notin \Lar(v'_{j-1})) \land \bar{G_{v'_{j - 1}}}] + \label{eq:v0-to-vj1-non-sigx} \\
& ~~~~ \sum_{v'_{j-1}\in V^{(j-1)}_i} \Pr[(v_0\to v'_{j-1}) \land (x\in \Lar(v'_{j-1}))] \label{eq:v0-to-vj1-sigx}.
\end{align}
By applying the $(j-1)$-pass version\footnote{Since we will prove Lemma~\ref{lemma:multi-pass-transfer} by induction on $j$, we can assume that lemma is true for the $(j-1)$-th pass.} of the transfer lemma (\Cref{lemma:multi-pass-transfer}) on \eqref{eq:v0-to-vj1-sigx}, we get that \eqref{eq:v0-to-vj1-sigx} is bounded by 
\[
2^{-\ellflat^{(j-1)} + 2 \ellsigs^{(j-1)} + \ellgood^{(j-2)}+1} + 2^{j-1}\cdot 2^{-\ell} \le \frac{3}{2} \cdot 2^{j-1}\cdot 2^{-\ell}.
\]
(By \Cref{tab:parameter-multi-pass}, we see that $\ellflat^{(j-1)}-2\ellsigs^{(j-1)}-\ellgood^{(j-2)}\ge 10\ell$.)
\\

Next, we show that \eqref{eq:v0-to-vj1-non-sigx} is small. Note that the event $v_0\to v'_{j-1}$ is equivalent to $(v_{j-2}\to v'_{j-1}) \land (v'_{j-2}\to v_{j-2})$. We use $v_{j-2}\wtone v'_{j-1}$ to denote the event $(v_{j-2}\to v'_{j-1}) \land (x\notin \Lar(v'_{j-1}))$. We then otain
\[
\Pr[(v_{0}\to v'_{j-1}) \land (x\notin \Lar(v'_{j-1})) \land \bar{G_{v'_{j - 1}}}] = \Pr[(v_{j-2}\wtone v'_{j-1}) \land \bar{G_{v'_{j - 1}}} \land (v'_{j-2}\to v_{j-2})].
\]
Using chain rule to first observe and condition on $(v'_{j-2}\to v_{j-2}) \land \bar{G_{v'_{j - 1}}}$, we obtain
\begin{align}
&~~~~\sum_{v'_{j-1}\in V^{(j-1)}_i} \Pr[(v_{j-2}\wtone v'_{j-1}) \land \bar{G_{v'_{j - 1}}} \land (v'_{j-2}\to v_{j-2})] \notag \\
&= \sum_{v'_{j-1}\in V^{(j-1)}_i} \Pr[v_{j-2}\wtone v'_{j-1} \mid (v'_{j-2}\to v_{j-1}) \land \bar{G_{v'_{j - 1}}}] \cdot \Pr[(v'_{j-2}\to v_{j-2}) \land \bar{G_{v'_{j - 1}}}] \label{equ:vj1-to-nonsig-Gbar} 
\end{align}
By \Cref{lemma:flat-dist} and the fact that $v'_{j-1}$ is not significant,
\[
\begin{aligned}
&\Pr[v_{j-2}\wtone v'_{j-1} \mid (v'_{j-2}\to v_{j-2}) \land \bar{G_{v'_{j - 1}}}] \\
&\qquad \le \max_{x'\in X} \left\{ \Pr[v_{j-2}\wtone v'_{j-1} \mid x=x'] \right\} \\
&\qquad\le 2^{\ellflat^{(j-1)}+1} \Pr[v_{j-2}\wt v'_{j-1}].
\end{aligned}
\]
We can proceed to bound Eq.~\eqref{equ:vj1-to-nonsig-Gbar} as
\begin{align*}
\eqref{equ:vj1-to-nonsig-Gbar} &\le 
2^{\ellflat^{(j-1)} + 1} \sum_{v'_{j-1}\in V^{(j-1)}_i} \Pr[v_{j-2}\wt v'_{j-1}] \cdot \Pr[(v'_{j-2}\wt v_{j-2}) \land \bar{G_{v'_{j - 1}}}].
\end{align*}

Recall the definition of the probabilistic program $B_{>i}$. Using \Cref{claim:probab-B-connection}, we can establish that
\begin{align}
&~~~~ \sum_{v'_{j-1}\in V^{(j-1)}_i} \Pr[v_{j-2}\wt v'_{j-1}] \cdot \Pr[(v'_{j-2}\wt v_{j-2}) \land \bar{G_{v'_{j - 1}}}] \notag \\
&= \sum_{v'_{j-1}\in V^{(j-1)}_i} \Pr[v_{j-2}\wt v'_{j-1}] \cdot \Pr[v'_{j-2}\wt v_{j-2}] \cdot \Pr[\bar{G_{v'_{j - 1}}} \mid v'_{j-2}\to v_{j-2} ] \notag \\
&= \sum_{v'_{j-1}}\Pr_{(x,a_{>i})}\left[ B_{>i}\text{ reaches } v'_{j-1} \right] \cdot \Pr_{(x,a_{>i})}\left[ \bar{G_{v'_{j - 1}}} \mid B_{>i}\text{ reaches } v'_{j-1} \right]. \notag \\
&= \sum_{v'_{j-1}}\Pr_{(x,a_{>i})}\left[ (B_{>i}\text{ reaches } v'_{j-1}) \land \bar{G_{v'_{j - 1}}} \right].\label{equ:pass-elimination-goal}
\end{align}
We emphasize that we do not apply any stopping rule on $B_{>i}$. Thus, the event $(B_{>i} \text{ reaches } v'_{j-1})$ means the computational path of $B_{>i}$ reaches $v'_{j-1}$. We can replace the subprogram $B_{>i}$ in \eqref{equ:pass-elimination-goal} with $\tilde{B}_{>i}$. Then, for an arbitrary but fixed internal randomness of $\tilde{B}_{>i}$, the resulting program $\hat{B}_{>i}$ is a deterministic $(j-1)$-pass program of width at most $2^{\frac{\kext \ell}{8\cdot {(j-1)}^{4(j-1)}}}$. Also, since the definition $\bar{G_{v'_{j - 1}}}$ does not depend on the realization of $\hat{B}_{>i}$, it suffices to bound
\[
\Ex_{\hat{B}_{>i}:\text{fixing randomness of }\tilde{B}_{>i}}\left[ \sum_{v'_{j-1}}  \Pr\left[(\text{the computation path of } \hat{B}_{>i}\text{ reaches } v'_{j-1}) \land \bar{G_{v'_{j - 1}}} \right]
\right].
\]
Observe that the events
\[
\left\{ (\text{the computation path of } \hat{B}_{>i} \text{ reaches } v'_{j-1}) \land \bar{G_{v'_{j - 1}}} \right\}_{v'_{j-1}}
\]
are mutually exclusive. Furthermore, all of them imply that $\hat{B}_{>i}$ stops\footnote{Either $\hat{B}_{>i}$ arrives at $v'_{j-1}$ without stopping and it stops in the $(j-1)$-th pass due to $\bar{G_{v'_{j-1}}}$, or $\hat{B}_{>i}$ stops in the first $(j-2)$ passes.}. 
Therefore, by applying \Cref{theo:multi-pass-main-result} on $\hat{B}_{>i}$, we get\footnote{Recall we define stopping thresholds for $\hat{B}_{>i}$ with ``$\ell$'' set to $2^{(j+2)}\ellflat^{(j-1)}$, and $\hat{B}_{>i}$ is a $(j-1)$-pass program.}
\[
\begin{aligned}
\eqref{equ:pass-elimination-goal} \le \Ex_{\hat{B}:\text{fixing randomness of }\tilde{B}_{>i}} \left[ \Pr_{x\sim X,a_{i+1},\dots, a_T\sim A}\left[ \hat{B}_{>i} \text{ stops} \right] \right] \le 2^{-4\ellflat^{(j-1)}}.
\end{aligned}
\]
Consequently,
\[
\eqref{eq:v0-to-vj1-non-sigx} \le 2^{\ellflat^{(j-1)}+1} \cdot 2^{-4\ellflat^{(j-1)}} \le 2^{-\ellflat^{(j-1)}},
\]
as desired.

Combining the bounds on \eqref{eq:v0-to-vj1-sigx} and \eqref{eq:v0-to-vj1-non-sigx}, the lemma is proved.
\end{proof}

\subsubsection{Proof of the Main Lemma}

Recall the statement of the lemma.

\multiTransfer*

\begin{proof}
    
First, we have
\begin{align}
    \sum_{v\in V^{(j)}_i} \Pr[(v_0\to v)\land E(x,v)] 
    &\le \sum_{v\in V^{(j)}_i} \Pr[(v_0\to v)\land E(x,v) \land G_{v'_{j-1}}] + \label{equ:multi-transfer-G}\\
    &~~~~ \sum_{v\in V^{(j)}_i} \Pr[(v_0\to v) \land \bar{G_{v'_{j - 1}}}]  \label{equ:multi-transfer-Gbar}
\end{align}

\paragraph*{When $G_{v'_{j-1}}$ does not happen.} We first bound \eqref{equ:multi-transfer-Gbar} using \Cref{lemma:pass-elimination}. Note that $(v_0\to v)$ is equivalent to $(v_0\to v'_{j-1}) \land (v'_{j-1}\to v)$. Fixing one $v'_{j-1}$, we have 
\[
\sum_{\substack{v\in V^{(j)}_i \\ v \text{ remembers }v'_{j-1}}} \Pr[\bar{G_{v'_{j - 1}}}\land (v_0\to v'_{j-1})\land (v'_{j-1}\to v)] \le \Pr[\bar{G_{v'_{j - 1}}} \land (v_0\to v'_{j-1})].
\]
Hence,
\[
\eqref{equ:multi-transfer-Gbar} \le \sum_{v'_{j-1} \in V^{(j-1)}_i} \Pr[(v_0\to v'_{j-1}) \land \bar{G_{v'_{j - 1}}}] \le 2^{j-\ell}.
\]

\paragraph*{When $G_{v'_{j-1}}$ happens.} Next, we bound \eqref{equ:multi-transfer-G}. Note that $(v_0\to v)$ is equivalent to $(v'_{j-1}\to v_{j-1})\land (v_{j-1}\to v)$. We get
\begin{align}
\Pr[(v_0\to v)\land E(x,v) \land G_{v'_{j-1}}] 
&= \Pr[(v'_{j-1}\to v_{j-1}) \land (v_{j-1}\to v)\land E(x,v) \land G_{v'_{j-1}}] \notag \\
&\le \Pr[ (v_{j-1}\to v)\land E(x,v) \land (v'_{j-1}\wt v_{j-1})\land G_{v'_{j-1}}]. \label{equ:transfer-step-1}
\end{align}
Using the chain rule by first observing and conditioning on $(v_{j-1}\to v)\land E(x,v)$, we get
\begin{align}
&~~~~ \Pr[ (v_{j-1}\to v)\land E(x,v) \land (v'_{j-1}\wt v_{j-1})\land G_{v'_{j-1}}] \notag \\
&\le \Pr[ (v'_{j-1}\wt v_{j-1})\land G_{v'_{j-1}} \mid (v_{j-1}\to v)\land E(x,v) ] \cdot \Pr[(v_{j-1}\to v)\land E(x,v)].\label{equ:transfer-step-2}
\end{align}
Observe that $v'_{j-1}\wt v_{j-1}\land G_{v'_{j-1}}$ depends on $x$ and $a_{> i}$, while $(v_{j-1}\to v)\land E(x,v)$ depends on $x$ and $a_{\leq i}$. Hence, we obtain
\begin{align}
& \Pr[ (v'_{j-1}\wt v_{j-1})\land G_{v'_{j-1}} \mid (v_{j-1}\to v)\land E(x,v) ] \notag \\
&\le \max_{x'\in X} \{ \Pr[ (v'_{j-1}\wt v_{j-1})\land G_{v'_{j-1}} \mid x=x']\} \notag \\
&\le 2^{\ellgood^{(j-1)}+1} \Pr[ (v'_{j-1}\wt v_{j-1})\land G_{v'_{j-1}}] & \text{(\Cref{claim:multi-transfer-Gv-useful} and \Cref{lemma:flat-dist})} \notag \\
&\le 2^{\ellgood^{(j-1)}+1} \Pr[v'_{j-1} \wt v_{j-1}]. \label{equ:transfer-step-3}
\end{align}
From Eq.~\eqref{equ:transfer-step-1}, Eq.~\eqref{equ:transfer-step-2} and Eq.~\eqref{equ:transfer-step-3}, we finally arrive at
\[
\Pr[(v_0\to v)\land E(x,v) \land G_{v'_{j-1}}] \le 2^{\ellgood^{(j-1)}+1}
\Pr[(v_{j-1}\to v)\land E(x,v)] \cdot \Pr[v'_{j-1}\wt v_{j-1}].
\]
Going back to \eqref{equ:multi-transfer-G}, we have
\[
\begin{aligned}
\eqref{equ:multi-transfer-G}
&\le 2^{\ellgood^{(j-1)}+1} \sum_{v\in V^{(j)}_i} \Pr[(v_{j-1}\to v)\land E(x,v)] \cdot \Pr[v'_{j-1}\wt v_{j-1}] \\
&\le 2^{\ellgood^{(j-1)}+1} \sum_{v'_{j-1}\in V^{(j-1)}_i} \sum_{v_{j-1}\in V^{(j-1)}_T} \sum_{v\in S_{v'_{j-1},v_{j-1},i}} \Pr[(v_{j-1}\to v)\land E(x,v)] \cdot \Pr[v'_{j-1}\wt v_{j-1}] \\
&\le 2^{-k + \ellgood^{(j-1)}+1} \sum_{v'_{j-1}\in V^{(j-1)}_i} \sum_{v_{j-1}\in V^{(j-1)}_T} \Pr[v_{j-2}\to v'_{j-1}] \cdot \Pr[v'_{j-1}\wt v_{j-1}] \\
&\le 2^{-k + \ellgood^{(j-1)}+1} \sum_{v'_{j-1}\in V^{(j-1)}_i} \sum_{v_{j-1}\in V^{(j-1)}_T} \Pr[v_{j-2}\wt v'_{j-1}] \cdot \Pr[v'_{j-1}\wt v_{j-1}].
\end{aligned}
\]
Here, we always require the pair of enumerated vertices $v'_{j-1},v_{j-1}$ to be consistent with each other.

\paragraph{Final step.} Finally, it suffices to show that
\begin{align}
\sum_{v'_{j-1}\in V^{(j-1)}_i} \sum_{v_{j-1}\in V^{(j-1)}_T} \Pr[v_{j-2}\wt v'_{j-1}] \cdot \Pr[v'_{j-1}\wt v_{j-1}] = 1. \label{equ:reflection-trick}
\end{align}
This summation can be interpreted as
\[
\sum_{v'_{j-1}\in V^{(j-1)}_i} \sum_{v_{j-1}\in V^{(j-1)}_T} \Pr_{x,\tilde{x},a_1,\dots,a_T}[\mathbbm{1}[v_{j-2}\wt v'_{j-1} \text{ using } \tilde{x},a_{\le i}] \land \mathbbm{1}[v'_{j-1}\wt v_{j-1} \text{ using } x,a_{>i}]].
\]
For any fixed $(x,\tilde{x},a_1,\dots, a_T)$, construct input $(a_t,b_t)_{t\in [T]}$ where $b_t = M(a_t,\tilde{x})$ for $t \le i$ and $b_t = M(a_t,x)$ for $t > i$. Consider running the first $j-1$ pass of the program on $(a_t,b_t)_{t\in [T]}$. Let $(v'_{j-1},v_{j-1})$ be the pair of vertices visited in this process. It is easy to see that $(v'_{j-1},v_{j-1})$ is the only pair such that 
\[
\mathbbm{1}[v_{j-2}\wt v'_{j-1} \text{ using } \tilde{x},a_{\le i}] \land \mathbbm{1}[v'_{j-1}\wt v_{j-1} \text{ using } x,a_{>i}] = \text{True}.
\]
Hence, we may conclude that the events
\[
\left\{\mathbbm{1}[v_{j-2}\wt v'_{j-1} \text{ using } \tilde{x},a_{\le i}] \land \mathbbm{1}[v'_{j-1}\wt v_{j-1} \text{ using } x,a_{>i}] \right\}_{v'_{j-1}, v_{j-1}}
\]
are mutually exclusive under the probability space $(x,\tilde{x},a)$. Furthermore, their union exactly covers the probability space. Hence, Eq.~\eqref{equ:reflection-trick} is verified.

\end{proof}

\section{Tight Lower Bound for Constant-Pass Learning}
In this section, we will prove our main result. 

\subsection{Potential Analysis} \label{sec:mult-pass-potential}
In this section, we will extend the potential analysis in \Cref{sec:two-pass-potential} to multiple passes. 

\paragraph*{Potential analysis and counter overflow.} We would like to carry out a similar potential analysis. The proof is largely identical/similar to the proofs in \Cref{sec:two-pass-potential}. Here, we only outline the key steps and defer the formal proofs to \Cref{appendix:multi-pass}.

Consider the $j$-th pass of the program. We define the potential of a vertex $v\in V^{(j)}$ as
\[
\Phi(v) := 2^{\cb^{(j)}(v)-\ch^{(j)}(v)}.
\]
For any edge $e$ between $(u,v)$ with label $(a',b')$, we defne its potential to be
\[
\Phi(e) = \Phi^{(a',b')}(u) = \Phi(v).
\]

Fix a starting vertex $v_{j-1}\in V^{(j)}_0$ for the $j$-th pass. Note that $v_{j-1}$ remembers the starting vertex of the $(j-1)$-th pass, which is denoted by $v_{j-2}\in V^{(j-1)}_0$. We consider the natural coupling between the truncated path starting at $v_{j-2}$ and the path starting at $v_{j-1}\in V^{(j)}_0$. For every $v'_{j-1}\in V^{(j-1)}_{i}$, we would like to prove
\[
\Pr_{v_{j-1}\to v}[\cb(v) > k \mid v_{j-2}\to v'_{j-1}] \le 2^{\ellbias^{(j-1)} -\ellbias^{(j)} + \ell+1}.
\]

\subsubsection{Potential Grows Slowly}

We start by proving the following lemma, the analog of \Cref{lem:edge-potential} in the multi-pass setting.


\begin{restatable}{lemma}{multiPotentialEdge}\label{lemma:edge-potetial-multi}
Fix $v_{j-1}\in V^{(j)}_0$ and $u\in V^{(j)}_i$. For any edge $e$ labeled $(a,b)$, we have
\[
\Phi^{(a,b)}(u) \le \Phi(u) \frac{1+2\cdot 2^{-\rext}}{2\Pr_{x|v_{j-1}\to u}[M(a,x)=b]}.
\]
\end{restatable} 

We note that the proof is identical to the proof of \Cref{lem:edge-potential} (basically, we only need to change relevant symbols). We defer the proof to Appendix~\ref{appendix:multi-pass}. 

The following lemma is the core of our analysis. Its two-pass analog is \Cref{lemma:two-pass-pot-slow-grow}.

\begin{restatable}{lemma}{multiPotentialGrow}\label{lemma:multi-pass-pot-slow-grow}
    Fix a starting vertex $v_{j-1}\in V^{(j)}_0$ in the $j$-th pass of the program. For every $i\in [T]$ and $v'_{j-1}\in V^{(j-1)}_i$ such that $\Pr[v_{j-2}\to v'_{j-1}] \ne 0$, we have
    \[
    \Ex_{v_{j-1}\to v}[\Phi(v)\mid v_{j-2}\to v'_{j-1}] = \frac{\Ex[\Phi(v)\cdot \mathbbm{1}[v_{j-2}\to v'_{j-1}]]}{\Pr[v_{j-2}\to v'_{j-1}]} \le (1+2^{-2\rlen + 2})^i \cdot 2^{\cb^{(j-1)}(v'_{j-1})}.
    \]
\end{restatable}

Compared with the two-pass case (\Cref{lemma:two-pass-pot-slow-grow}), we have an extra term $2^{\cb^{(j-1)}(v'_{j-1})}$, because it is no longer true that, for every edge $(a',b')$ that goes from $u'_{j-1}$ to $v'_{j-1}$, the probability that we traverse this edge is close to $2^{-n-1}$. Recall that it was the case in the two-pass setting (i.e., $j=2$), because we can always stop whenever we meet a bad edge in the first pass (i.e., $j=1$). However, for larger $j$, we sometimes need to traverse high-probability edges in pass $(j-1)$, which can be very biased. Fortunately, the vertices in the $(j-1)$-th pass have remembered a counter $\cb^{(j-1)}(v'_{j-1})$ to account for the bias introduced by traversing high-probability edges. We can take advantage of the counters and prove \Cref{lemma:multi-pass-pot-slow-grow}.

As a final remark, by taking $j=2$ in \Cref{lemma:multi-pass-pot-slow-grow}, the lemma coincides with \Cref{lemma:two-pass-pot-slow-grow}, because there are no bias counters in the first pass ($\cb^{(1)}(v'_1)\equiv 0$).

Given the definition of counters, we prove \Cref{lemma:multi-pass-pot-slow-grow} by slightly modifying the argument of \Cref{lemma:two-pass-pot-slow-grow}. We defer the details to \Cref{appendix:multi-pass}.

\subsubsection{Analyzing $\cb^{(j)}$ Overflow} 

As a corollary of \Cref{lemma:multi-pass-pot-slow-grow}, we upper bound the probability of $\cb^{(j)}$ overflow.

\begin{corollary}\label{coro:multi-counter-overflow}
    Fix $v_{j-1}\in V^{(j)}_{0}$ (which also fixes $v_{j-2}\in V^{(j-1)}_0$). For any layer $i\in [T]$ and $v'_{j-1}\in V^{(j-1)}_{i}$, we have
    \[
    \Pr_{v_{j-1}\to v}[\cb^{(j)}(v) > \ellbias^{(j)} \mid v_{j-2} \to v'_{j-1}] \le 2^{\ellbias^{(j-1)}-\ellbias^{(j)}+\ell+1}.  
    \]
\end{corollary}

\begin{proof}
By the stopping rule from $\ch^{(j)}$, we know that if we have not stopped, we always have $\ch^{(j)}(v) \le \ell$. Also, if we did not stop at $v'_{j-1}$, we have $\cb^{(j-1)}(v'_{j-1})\le \ellbias^{(j-1)}$. We apply Markov's inequality and get
\[
\begin{aligned}
\Pr_{v_{j-1}\to v}[\cb^{(j)}(v) > \ellbias^{(j)} \mid v_{j-2} \to v'_{j-1}] 
&\le \frac{1}{2^{\ellbias^{(j)}}} \Ex_{v_{j-1}\to v}[2^{\cb^{(j)}(v)} \mid v_{j-2} \to v'_{j-1}] \\
&\le \frac{2^{\ell}}{2^{\ellbias^{(j)}}} \Ex_{v_{j-1}\to v}[\Phi(v) \mid v_{j-2} \to v'_{j-1}] \\
&\le \frac{2^{\ell}\cdot (1+2^{-\rlen + 2})^{i} \cdot 2^{\cb^{(j-1)}(v'_{j-1})} }{2^{k}} \\
&\le 2^{\ellbias^{(j-1)}- \ellbias^{(j)} + \ell + 1},
\end{aligned}
\]
as desired.
\end{proof}


\subsection{Proof of the Main Result} \label{sec:multi-pass-proof-of-main}
In this section, we will prove our main result. We will use $j$ to denote the total number of passes of our program.

We would like to show that the program stops with a very small probability. We will analyze stopping due to different rules separately. Here is an outline:
\begin{itemize}
\item Stop in the first $(j-1)$-th pass: By induction, we have
\begin{align}
\Pr[\text{stop in the first $(j-1)$ pass}] \le 2^{-\frac{\ell}{2^{j-1}}} \label{equ:multi-pass-first-j-1-bound}
\end{align}
\item Stop due to traversing too many high-probability edges in the $j$-th pass: see \Cref{sec:multi-pass-edges}.
\item Stop due to traversing a bad edge in the $j$-th pass: see also \Cref{sec:multi-pass-edges}.
\item Stop due to significant values or $\cb$ overflow in the $j$-th pass: see \Cref{sec:multi-pass-sigv-and-counter}.
\item Stop due to reaching a significant state in the $j$-th pass: see \Cref{sec:multi-sig-state-stop}.
\end{itemize}

Finally, we wrap up our analysis in \Cref{sec:multi-pass-wrap-up}.

\subsubsection{Stop on Edges}\label{sec:multi-pass-edges}

Consider the $j$-th pass of the program. We show that the probability of stopping due to traversing a bad edge, or due to traversing too many high-probability edges, is small.

The proof is largely identical to the two-pass setting. For bad edges, we have
\[
\begin{aligned}
&~~~~ \Pr[\text{stop due to bad edge in the $j$-th pass}] \\
&= \sum_{i=0}^{T-1} \sum_{v\in V^{(j)}_i} \Pr[v_0\to v] \cdot \Pr[a_{i+1}\in \Bad(v)\setminus \High(v)|v_0\to v] \\
&\le \sum_{i=0}^{T-1} \Pr[v_0\to v] \cdot 2^{n-\kext} \cdot 2^{\frac{\kext}{2}-n} \\
&\le T\cdot 2^{-\frac{\kext}{2}}.
\end{aligned}
\]
Now we consider high-probability edges. For each $v_{j-1}\in V^{(j)}_0$, let $E_{v_{j-1}}$ denote the event that the program starts from $v_{j-1}$ and traverses more than $\ell$ high-probability edges. We have
\[
\begin{aligned}
&\qquad \sum_{v_{j-1}} \Pr[v_0\to v_{j-1}] \Pr[E_{v_{j-1}} \mid v_0\to v_{j-1}] \\
& \le \sum_{v_{j-1}} \Pr[E_{v_{j-1}}] \\
& \le 2^{\frac{\kext\ell}{10}} \cdot \binom{T}{\ell} 2^{-\frac{\kext \ell}{2}} \\
& \le 2^{-\frac{\kext\ell}{10}}.
\end{aligned}
\]

Overall, by our assumption that $\min\{\ellext,\kext\} \ge 100^{3^{j-1}} \ell$, the probability of stopping on edges is at most
\begin{align}
   \Pr[\text{stop on edges in the $j$-th pass}] \le  2^{-\ellgood^{(j-1)}}. \label{equ:multi-edge-bound}
\end{align}

\subsubsection{Stop due to $\cb^{(j)}$ Overflow and Significant Values}\label{sec:multi-pass-sigv-and-counter}

We show that the probability of stopping due to significant values or counter-overflow in the $j$-th pass is small.

Fix one $i \in [T]$, we define the ``bad event'' indicator $E:X\times V^{(j)}_i\to \{0,1\}$. For each $v\in V^{(j)}_i$, we define:
\begin{itemize}
    \item If $v$ is a significant state, we set $E(x,v) \equiv 0$ for all $x$. We will bound the probability of reaching such states in \Cref{sec:multi-sig-state-stop}.
    \item If $v$ is not significant but $\cb(v) > {\ellbias^{(j)}}$, we set $E(x,v)\equiv 1$ for all $x$.
    \item Otherwise, we set $E(x,v) = \mathbbm{1}[x\in \SigV(v)]$.
\end{itemize}

We would like to show that 
\begin{align*}
\sum_{v\in V^{(j)}_i} \Pr[v_0\to v \land E(x,v)] \le 2^{-8\ellgood^{(j-1)}} + 2^{j-\ell}. 
\end{align*}
Once established, we can union-bound over $i\in [T]$.

We would like to apply \Cref{lemma:multi-pass-transfer}. Let us first establish the assumption of \Cref{lemma:multi-pass-transfer}. Fix $v'_{j-1}\in V^{(j-1)}_i$ and $v_{j-1}\in V^{(j)}_0$ to be a consistent pair. Recall we have defined 
\[
S_{v'_{j-1},v_{j-1},i}\coloneqq \left\{ v\in V^{(j)}_i : v \text{ remembers } v'_{j-1},v_{j-1}\right\}.
\]
Then, observe that
\begin{align}
&~~~~ \sum_{\substack{v\in S_{v'_{j-1},v_{j-1},i}\\ v \text{ not significant}}} \Pr[v_{j-1}\to v\land \mathbbm{1}[x\in \SigV(v)]] \notag \\
&\le \sum_{v\in S_{v'_{j-1},v_{j-1},i}} \Pr[v_{j-1}\to v] \cdot 2^{2\ellsigs^{(j)} - \ellsigv} & \text{(by $\ell_{\infty}$-truncation trick)} \notag \\
&\le \Pr[v_{j-2}\to v'_{j-1}] \cdot 2^{2\ellsigs^{(j)} - \ellsigv} \notag \\
&\le \Pr[v_{j-2}\to v'_{j-1}] \cdot 2^{-10\ellgood^{(j-1)}} \label{equ:multi-sig-value-assum}
\end{align}

Also, by \Cref{coro:multi-counter-overflow}, we have
\begin{align}
&~~~~ \sum_{\substack{v\in S_{v'_{j-1},v_{j-1},i}\\ v \text{ not significant}}} \Pr[v_{j-1}\to v\land \mathbbm{1}[\cb^{(j)}(v)>{\ell^{(j)}_b}] ] \notag \\
&\le \Pr[v_{j-2}\to v'_{j-1}] \cdot \Pr_{v_{j-1}\to v}[\cb^{(j)}(v)>{\ell^{(j)}_b} \mid v_{j-2}\to v'_{j-1}] \notag \\
&\le \Pr[v_{j-2}\to v'_{j-1}] \cdot 2^{-\ellbias^{(j)} + \ellbias^{(j-1)} + \ellhigh + 1} \notag \\
&\le \Pr[v_{j-2}\to v'_{j-1}] \cdot 2^{-10\ellgood^{(j-1)}} \label{equ:count-over-assum} 
\end{align}

Combining Eq.~\eqref{equ:multi-sig-value-assum} and Eq.~\eqref{equ:count-over-assum}, the assumption of \Cref{lemma:multi-pass-transfer} is established with $k = 10\ellgood^{(j-1)}-1$, which allows us to invoke \Cref{lemma:multi-pass-transfer} and finish the proof. Overall, we get
\begin{align}
 \Pr[\text{stop due to $\cb$ or significant value}] \le T\cdot (2^{j-\ell} + 2^{-8\ellgood^{(j-1)}}) \le 2^{-\ell/2}  \label{equ:multi-counter-sigv-bound}
\end{align}

\subsubsection{Reaching a Significant State}\label{sec:multi-sig-state-stop}

We instantiate \Cref{lemma:sig-state-multi-pass} (see \Cref{appendix:sig-state-two-pass}) with parameters (copied from \Cref{tab:parameter-multi-pass}):
\[
\ellsigs^{(j)} = \ell\cdot 100^{3^{j-1}-1} \quad \text{ and } \quad \ellbias^{(j)} = \ell \cdot \frac{100^{3^{j-1}-1}}{2}.
\]
For each fixed significant state $s$ in the $j$-th pass, \Cref{lemma:sig-state-multi-pass} implies that
\[
\Pr[v_0\to s] = \Pr[(v_0\to v_{j-1})\land (v_{j-1}\to s)]\le \Pr[v_{j-1}\to s] \le 2^{-\frac{1}{2}\kext( \ellsigs^{(j)} - \ellbias^{(j)} - \ellhigh - 5) } \le 2^{-\frac{\kext\ellsigs^{(j)}}{100}}.
\]
Before we apply the modification to the program $B$, we have at most $2^{\frac{\kext\ell}{8j^{4j}}}$ states in each layer of the program. Hence, after the modification, there are at most 
\[
T\cdot (2^{\frac{\kext\ell}{8j^{4j}}})^{4j} \le 2^{\frac{\kext \ell}{8}}
\]
states in the $j$-th pass of the program. We can union-bound over all those states, to conclude that
\begin{align}
\Pr[\text{stop due to significant state in the $j$-th pass}] \le  2^{-\ellgood^{(j-1)}}. \label{equ:multi-sig-state-bound}
\end{align}

\subsubsection{Concluding the Proof}\label{sec:multi-pass-wrap-up}

Combining Eq. \eqref{equ:multi-pass-first-j-1-bound}, \eqref{equ:multi-edge-bound}, \eqref{equ:multi-counter-sigv-bound} and \eqref{equ:multi-sig-state-bound}, we finally conclude that
\begin{align}
\Pr[\text{the program stops}] \le 2^{-\frac{\ell}{2^{j-1}}} + 2^{-\ell/2} + 2^{-\ellgood^{(j-1)} + 10} \le 2^{-\frac{\ell}{2^{j}}}.
\end{align}

The rest argument is similar to the two-pass case. Denote by $\overline{G}\subseteq X\times A^T$ the union of all stopping events. Let $G$ be the complement of $\overline{G}$. We have shown that
\[
\Pr_{x,a_1,\dots, a_T}[(x,a_1,\dots, a_T)\in \overline{G}] \le 2^{-\frac{\ell}{2^{j}}}.
\]
Moreover, for every final vertex $v$ of the program, the event $v_{j-1}\to v$ is equivalent to $v_0\to v$, which is, in turn, equivalent to $(v_{0}\wt v)\land G$. Note that
\[
\begin{aligned}
& \| \doubleP_{x|(v_{0}\wt v) \land G} \|_\infty \le 2^{\ellsigv+1} \cdot 2^{-n}.
\end{aligned}
\]
Therefore, conditioning on $v_0\to v$, the probability of guessing $x$ correctly is exponentially small. Since this holds for every $v\in V^{(j)}_T$, we conclude the two-pass learning algorithm succeeds in learning $x$ with an exponentially small probability. This proves \Cref{theo:two-pass-main-result}.

\section{Acknowledgement}

We are grateful to Wei Zhan for helpful comments and suggestions on an earlier version of the paper. We also thank FOCS reviewers for their valuable comments.

\bibliographystyle{alpha}
\bibliography{references}

\begin{thebibliography}{GKLR21}

\bibitem[Bar86]{Barrington86}
David A.~Mix Barrington.
\newblock Bounded-width polynomial-size branching programs recognize exactly
  those languages in nc{\({^1}\)}.
\newblock In {\em {STOC}}, pages 1--5. {ACM}, 1986.

\bibitem[BGY18]{BGY18}
Paul Beame, Shayan~Oveis Gharan, and Xin Yang.
\newblock Time-space tradeoffs for learning finite functions from random
  evaluations, with applications to polynomials.
\newblock In {\em {COLT}}, volume~75 of {\em Proceedings of Machine Learning
  Research}, pages 843--856. {PMLR}, 2018.

\bibitem[Csa76]{Csanky}
L.~Csanky.
\newblock Fast parallel matrix inversion algorithms.
\newblock {\em SIAM Journal on Computing}, 5(4):618--623, 1976.

\bibitem[DKS19]{DKS19}
Yuval Dagan, Gil Kur, and Ohad Shamir.
\newblock Space lower bounds for linear prediction in the streaming model.
\newblock In {\em {COLT}}, volume~99 of {\em Proceedings of Machine Learning
  Research}, pages 929--954. {PMLR}, 2019.

\bibitem[DS18]{DaganS18}
Yuval Dagan and Ohad Shamir.
\newblock Detecting correlations with little memory and communication.
\newblock In {\em {COLT}}, volume~75 of {\em Proceedings of Machine Learning
  Research}, pages 1145--1198. {PMLR}, 2018.

\bibitem[GKLR21]{DBLP:conf/approx/GargKLR21}
Sumegha Garg, Pravesh~K. Kothari, Pengda Liu, and Ran Raz.
\newblock Memory-sample lower bounds for learning parity with noise.
\newblock In {\em {APPROX-RANDOM}}, volume 207 of {\em LIPIcs}, pages
  60:1--60:19. Schloss Dagstuhl - Leibniz-Zentrum f{\"{u}}r Informatik, 2021.

\bibitem[GKR20]{GKR20}
Sumegha Garg, Pravesh~K. Kothari, and Ran Raz.
\newblock Time-space tradeoffs for distinguishing distributions and
  applications to security of goldreich's {PRG}.
\newblock In {\em {APPROX-RANDOM}}, volume 176 of {\em LIPIcs}, pages
  21:1--21:18. Schloss Dagstuhl - Leibniz-Zentrum f{\"{u}}r Informatik, 2020.

\bibitem[GLM20]{GLM20}
Alon Gonen, Shachar Lovett, and Michal Moshkovitz.
\newblock Towards a combinatorial characterization of bounded-memory learning.
\newblock In {\em NeurIPS}, 2020.

\bibitem[GRT18]{GargRT18-extractor}
Sumegha Garg, Ran Raz, and Avishay Tal.
\newblock Extractor-based time-space lower bounds for learning.
\newblock In {\em {STOC}}, pages 990--1002. {ACM}, 2018.

\bibitem[GRT19]{garg2019time}
Sumegha Garg, Ran Raz, and Avishay Tal.
\newblock Time-space lower bounds for two-pass learning.
\newblock In {\em 34th Computational Complexity Conference (CCC)}, 2019.

\bibitem[KRT17]{KRT17}
Gillat Kol, Ran Raz, and Avishay Tal.
\newblock Time-space hardness of learning sparse parities.
\newblock In {\em {STOC}}, pages 1067--1080. {ACM}, 2017.

\bibitem[LRZ23]{liu2023memory}
Qipeng Liu, Ran Raz, and Wei Zhan.
\newblock Memory-sample lower bounds for learning with classical-quantum hybrid
  memory.
\newblock {\em arXiv preprint arXiv:2303.00209}, 2023.

\bibitem[MM17]{MM17}
Dana Moshkovitz and Michal Moshkovitz.
\newblock Mixing implies lower bounds for space bounded learning.
\newblock In {\em {COLT}}, volume~65 of {\em Proceedings of Machine Learning
  Research}, pages 1516--1566. {PMLR}, 2017.

\bibitem[MM18]{MM18}
Dana Moshkovitz and Michal Moshkovitz.
\newblock Entropy samplers and strong generic lower bounds for space bounded
  learning.
\newblock In {\em {ITCS}}, volume~94 of {\em LIPIcs}, pages 28:1--28:20.
  Schloss Dagstuhl - Leibniz-Zentrum f{\"{u}}r Informatik, 2018.

\bibitem[MSSV22]{MSSV22}
Annie Marsden, Vatsal Sharan, Aaron Sidford, and Gregory Valiant.
\newblock Efficient convex optimization requires superlinear memory.
\newblock In {\em {COLT}}, volume 178 of {\em Proceedings of Machine Learning
  Research}, pages 2390--2430. {PMLR}, 2022.

\bibitem[MT17]{MT17}
Michal Moshkovitz and Naftali Tishby.
\newblock Mixing complexity and its applications to neural networks.
\newblock {\em CoRR}, abs/1703.00729, 2017.

\bibitem[Raz16]{Raz16}
Ran Raz.
\newblock Fast learning requires good memory: {A} time-space lower bound for
  parity learning.
\newblock In {\em {FOCS}}, pages 266--275. {IEEE} Computer Society, 2016.

\bibitem[Raz17]{raz2017time}
Ran Raz.
\newblock A time-space lower bound for a large class of learning problems.
\newblock In {\em 2017 IEEE 58th Annual Symposium on Foundations of Computer
  Science (FOCS)}, pages 732--742. IEEE, 2017.

\bibitem[Sha14]{Shamir14}
Ohad Shamir.
\newblock Fundamental limits of online and distributed algorithms for
  statistical learning and estimation.
\newblock In {\em {NIPS}}, pages 163--171, 2014.

\bibitem[SSV19]{SSV19}
Vatsal Sharan, Aaron Sidford, and Gregory Valiant.
\newblock Memory-sample tradeoffs for linear regression with small error.
\newblock In {\em {STOC}}, pages 890--901. {ACM}, 2019.

\bibitem[SVW16]{SVW16}
Jacob Steinhardt, Gregory Valiant, and Stefan Wager.
\newblock Memory, communication, and statistical queries.
\newblock In {\em {COLT}}, volume~49 of {\em {JMLR} Workshop and Conference
  Proceedings}, pages 1490--1516. JMLR.org, 2016.

\end{thebibliography}

\appendix


\section{Probability of Reaching Significant States} \label{appendix:sig-state-two-pass}
We will show that, for every $j \ge 1$, fixing a starting vertex $v_{j-1}\in V^{(j)}_0$, the probability of reaching a significant state from $v_{j-1}$ is small.

\subsection{Setup}

We will use $v_0$ to denote $v_{j-1}$ to avoid heavy notation. We also use $B$ to denote the sub-program with starting with $v_{j-1}\in V^{(j)}_0$ and ending at $v_j\in V^{(j)}_T$. 

Note that $B$ is a one-pass program. However, we cannot use the analysis of \cite{GargRT18-extractor} directly because we need to handle a set of very different stopping rules for $B$. 

\paragraph*{Review of the stopping rule in the $j$-th pass.} To begin with, let us review the stopping rules we have defined for the $j$-th pass of the program. Suppose we traversed from $v_{j-1}$ to a vertex $v\in V^{(j)}_i$. We will apply the following stopping rules (copied from \Cref{sec:multi-pass-stop-rules}).

\begin{enumerate}
    \item If $v$ is a significant state, we stop.
    \item Before traversing the next edge, if $x\in \SigV(v)$, we stop.
    \item When we are about to traverse an edge $(a,b)$ where $a\in \Bad(v)\setminus \High(v)$, we stop.    
    \item If the copy of the (modified) previous pass stops at $v'_{j - 1}$ due to whatever reason (including stopping due to this rule), we also stop. 
    \item When $\ch^{(j)}(v) > \ell$, we stop.
    \item When $\cb^{(j)}(v) > {\ell^{(j)}_b}$, we stop.
\end{enumerate}

We should pay special attention to Rule 4. Suppose $v$ remembers $v'_{1},\dots, v'_{j-1}$. Then, Rule 4 means that if one of the following events happens, we also stop:
\begin{itemize}
    \item The hidden $x$ satisfies $x\in \bigcup_{t=1}^{j-1} \SigV(v'_{t})$.
    \item The next edge $(a_{i+1},b_{i+1})$ satisfies $a\in \bigcup_{t=1}^{j-1} \left( \Bad(v'_{t}) \setminus \High(v'_{t}) \right)$.
    \item One of $v'_{t}$ is a significant state, or satisfies $\ch^{(t)}(v'_{t}) > \ell$, or has $\cb^{(t)}(v'_{t}) > {\ell^{(j)}_b}$.
\end{itemize}

\paragraph*{Summarizing the Stopping Rules.} We would like to sort out the different stopping rules and summarize them into the following categories.

\begin{enumerate}
    \item \emph{Significant States.} If $\|\doubleP_{x|v_0\to v}\|_2\ge 2^{-n}\cdot 2^{\ell_s}$, we call $v$ a significant state.
    \item \emph{Bad states.} If the state $v$ satisfies certain conditions, we stop at $v$ right away.\footnote{Namely, if $v$ remembers a state $v'_t$ in the previous pass, where $v'_t$ stops due to being a significant state or triggering counter-overflow.}
    \item \emph{Bad values.} If the hidden $x$ satisfies that $x\in \SigV^{(all)}(v)$, we stop.\footnote{Recall we define $\SigV^{(all)}(v)$ as $\SigV(v)\cup \bigcup_{t=1}^{j-1} \SigV(v'_{t})$.} 
    \item \emph{Bad edges.} If the next edge $(a,b)$ satisfies that 
    \[
    a\in \left( \Bad(v)\setminus \High(v) \right) \cup \left( \bigcup_{t=1}^{j-1} \left( \Bad(v'_{t}) \setminus \High(v'_{t}) \right)\right),
    \]
    we stop.
    \item \emph{Counter overflow.} If $\ch^{(j)}(v) > \ell$ or $\cb^{(j)}(v) > {\ell^{(j)}_b}$, we stop.
\end{enumerate}

\paragraph*{The main lemma.} Now, we are ready to state the core lemma of this section.

\begin{lemma}\label{lemma:sig-state-multi-pass}
Suppose $M$ is an $(\ellext, \rext, \kext)$-$L2$-extractor. Consider the program $B$ for the learning task of $M$ with the aforementioned stopping rules. If all of the following conditions hold.
\begin{itemize}
     \item The threshold for significant states is $\ell_s$.
     \item The threshold for significant values is $\ellsigv$.
     \item The threshold for $\ch^{(j)}$ sets to $\ell$.
     \item The threshold for $\cb^{(j)}$ sets to ${\ell^{(j)}_b}$.
\end{itemize}
Suppose the program has length $T = 2^{\rlen}$, and the following inequalities are true.
\begin{itemize}
    \item $\ell_s + \ellsigv + {\ell^{(j)}_b} + 5 < \ellext$. 
    \item $\ellsigv\ge \log_2(j) + 2\ell_s + {\ell^{(j)}_b} + 2\rlen + 5$. 
    \item $\rlen \le \frac{1}{4} \min(\rext, \kext)$.
\end{itemize}
Now, if $s$ is a significant state of $B$, we have
\[
\Pr[v_0\to s] \le 2^{-\frac{1}{2}\kext (\ell_s - {\ell^{(j)}_b} - \ell - 5)}.
\]
\end{lemma}

The rest of this subsection is devoted to the proof of \Cref{lemma:sig-state-multi-pass}. For brevity, we use $\doubleP_{x|v}$ (and $\Pr[v]$) to denote $\doubleP_{x|v_0\to v}$ (and $\Pr[v_0\to v]$).

For each edge $e$ in $B$, we use $\Pr[e]$ to denote the probability that the program traverses $e$, and $\doubleP_{x|e}$ to denote the conditional distribution of $x$ conditioning on $v_0\to e$.

\subsection{Understanding $\doubleP_{x|v}$ and $\doubleP_{x|e}$}

To start, we show that $\SigV^{(all)}(v)$ only contains a tiny amount of $x'\sim \doubleP_{x|v}$, as shown in the following lemma. 

\begin{claim}\label{claim:v1-sig-value-stop}
If $v$ is a non-significant vertex, we have
\[
\Pr_{x'\sim \doubleP_{x|v}}[x'\in \SigV^{(all)}(v) ] \le j\cdot 2^{2\ell_s - \ellsigv}.
\]
\end{claim}

\begin{proof}
We have
\[
\begin{aligned}
\Pr_{x'\sim \doubleP_{x|v}}[x'\in \SigV(v)]
&\le \frac{\Ex_{x'\sim \doubleP_{x|v}}[\doubleP_{x|v}(x')]}{2^{\ellsigv-n}} \\
&= \frac{2^n\cdot \| \doubleP_{x|v} \|^2}{2^{\ellsigv - n}} \\
&\le 2^{2\ell_s - \ellsigv}.
\end{aligned}
\]
For every $t < j$, we have
\[
\begin{aligned}
\Pr_{x'\sim \doubleP_{x|v}}[x'\in \SigV(v'_{t})]
&\le \frac{\Ex_{x'\sim \doubleP_{x|v}}[\doubleP_{x|v'_{t}}(x')]}{2^{\ellsigv-n}} \\
&\le \frac{2^n\cdot \| \doubleP_{x|v} \| \cdot \| \doubleP_{x|v'_{t}} \|}{2^{\ellsigv - n}} & \text{(Cauchy-Schwarz)} \\
&\le 2^{2\ell_s - \ellsigv}.
\end{aligned}
\]
With a simple union bound, we obtain
\[
\Pr_{x'\sim \doubleP_{x|v}}[x'\in \SigV^{(all)}(v) ] \le j\cdot 2^{2\ell_s - \ellsigv},
\]
as desired.
\end{proof}

Next, suppose we traverse an edge $e=(v,u)$ in the program. We consider how $\doubleP_{x|v}$ is related to $\doubleP_{x|e}$, and derive the following claim.

\begin{claim} \label{edge-evolution-second-pass}
   For any edge $e=(v,u)$ of $B$ labeled by $(a,b)$ such that $\Pr[e] > 0$, we claim
   \[
       \doubleP_{x|e}(x') = 
       \begin{cases}
           0 & \text{if $x'\notin \SigV^{(all)}(v)$ or $M(a,x')\ne b$,} \\
           \doubleP_{x|v}(x') \cdot {c_e}^{-1} & \text{otherwise.}
       \end{cases}
    \]
    where $c_e$ is a normalization factor satisfying:
    \[
    c_e\ge 
     \begin{cases}
        \frac{1}{2} -  2^{-2\rlen} & \text{ if $a\notin \High(v)$,} \\
        2^{-\Delta(e) - 1}(1 - 2^{-2\rlen}) & \text{otherwise.} 
    \end{cases}
    \]
    Here, $\Delta(e)$ is defined as $\cb(u)-\cb(v)$.
\end{claim}

\begin{proof}
    Let $e=(v,u)$ be an edge of $B$ labeled by $(a,b)$ and such that $\Pr[e] > 0$. By the design of the stopping rules, we observe that
    \[
    \doubleP_{x|e}(x') = 
   \begin{cases}
       0 & \text{if $x'\in \SigV^{(all)}(v)$ or $M(a,x')\ne b$,} \\
       \doubleP_{x|v}(x') \cdot {c_e}^{-1} & \text{otherwise.}
   \end{cases}
    \]
    where $c_e$ is the normalization factor, given by
    \[
    c_e = \Pr_{x‘\sim \doubleP_{x|v}}[x'\notin \SigV^{(all)}(v) \land M(a,x')=b].
    \]
    Since the path does not stop on $v$, by \Cref{claim:v1-sig-value-stop},
    \[
    \Pr_{x'\sim \doubleP_{x|v}}[x'\in \SigV^{(all)}(v) ] \le j\cdot 2^{2\ell_s - \ellsigv}.
    \]
    Note that $\Pr[e] > 0$ implies that $a\notin \Bad(v)\setminus \High(v)$. Now we shall consider two cases:
    \begin{itemize}
        \item If $a\notin \High(v)$, then $a\notin \Bad(v)$, implying that 
        \[
        \Pr_{x'\sim \doubleP_{x|v}}[M(a,x') = b] \in \left( \frac{1}{2} - 2^{-\rext}, \frac{1}{2} + 2^{-\rext}\right).
        \]
        \item Otherwise, we have
        \[
        \Pr_{x'\sim \doubleP_{x|v}}[ M(a,x') = b ] \ge 2^{-\Delta(e) - 1}.
        \]
    \end{itemize}
    In the former case, we have
    \[
    c_e \ge \Pr_{x'\sim \doubleP_{x|v}}[M(a,x') = b]  - \Pr_{x'\sim \doubleP_{x|v}}[x'\in \SigV^{(all)}(v) ] \ge \frac{1}{2} - 2^{\rext} - j\cdot 2^{2\ell_s-\ellsigv} \ge \frac{1}{2} - 2^{-2\rlen}.
    \]
    In the latter case, we use $\ellsigv \ge \log_2(j) + 2\ell_s + {\ell^{(j)}_b} + 2\rlen + 5$ to get
    \[
    c_e \ge \Pr_{x'\sim \doubleP_{x|v}}[M(a,x') = b]  - \Pr_{x'\sim \doubleP_{x|v}}[x'\in \SigV^{(all)}(v) ] \ge 2^{-\Delta(e)-1} - j\cdot 2^{2\ell_s - \ellsigv} \ge 2^{-\Delta(e) - 1} (1-2^{-2\rlen}).
    \]
    Thus,
    \[
    c_e\ge 
     \begin{cases}
        \frac{1}{2} - 2^{-2\rlen} & \text{if $a\notin \High(v)$,} \\
        2^{-\Delta(e) - 1}(1 - 2^{-2\rlen})  & \text{otherwise.} 
    \end{cases}
    \]
    This completes the proof.
\end{proof}

\subsection{Bounding the Norm of $\doubleP_{v_1\to s}$}

We claim the following.

\begin{lemma}\label{lemma:Pe-bound}
    For any edge $e$ such that $\Pr[e] > 0$, it holds that
    \[
    \| \doubleP_{x|e} \|_2 \le 2^{\ell_s + {\ell^{(j)}_b} + 2}\cdot 2^{-n}.
    \]
\end{lemma}

\begin{proof}
    Let $e=(v,u)$ be an edge labeled by $(a,b)$. Since $\Pr(e) > 0$, the vertex $v$ is not significant (as otherwise $\Ttwo$ stops on $v$ and $\Pr(e) = 0$). Thus,
    \[
    \| \doubleP_{x|v} \|_2 \le 2^{\ell_s} \cdot 2^{-n}.
    \]
    Let $v'$ be the vertex in the first pass remembered by $v$. By \Cref{edge-evolution-second-pass}, for any $x'\in X$, we have
    \[
    \doubleP_{x|e}(x') = 
    \begin{cases}
    0 & \text{if $x'\in \SigV^{(all)}(v)$\ or\ $M(a,x') \ne b$,} \\
    \doubleP_{x|v}\cdot c_e^{-1} & \text{otherwise,}
    \end{cases}
    \]
    where $c_e$ satisfies $c_e\ge 2^{-{\ell^{(j)}_b} - 2}$. Consequently, 
    \[
    \| \doubleP_{x|e} \|_2 \le {c_e}^{-1} \| \doubleP_{x|v} \|_2 \le 2^{\ell_s + {\ell^{(j)}_b} + 2} \cdot 2^{-n},
    \]
    as desired.
\end{proof}

\begin{corollary}\label{coro:sig-state-ell2-bound}
If $v$ is a vertex such that $\Pr[v] > 0$, then $\| \doubleP_{x|s} \|_2 \le 2^{\ell_s + {\ell^{(j)}_b} + 2}\cdot 2^{-n}$.
\end{corollary}

\begin{proof}
    Note that $\doubleP_{x|s}$ is a convex combination of $\{\doubleP_{x|e}\}_e$ where $e$ enumerates all incoming edges of $s$. The desired bound follows by Jensen's inequality and Lemma~\ref{lemma:Pe-bound}.
\end{proof}

\subsection{Measuring the Progress}

Let $k = \frac{\kext}{2}$. We introduce the following progress function (we use $\cb,\ch$ to denote $\cb^{(j)}, \ch^{(j)}$):
\[
Z_i = \sum_{v\in V_i} \Pr[v_0\to v] \cdot 2^{-k\cdot (\ch(v)+\cb(v))} \cdot \langle\doubleP_{x|v_1\to v}, \doubleP_{x|v_1\to s}\rangle^k.
\]
It is clear that $Z_0 = 2^{-2nk}$ and $\langle \doubleP_{x|v_1\to s}, \doubleP_{x|v_1\to s}\rangle^k \ge 2^{-2nk + 2\ell_s k}$. We show the following lemma.

\begin{lemma}\label{lemma:Zi-grow}
It holds that $Z_{i+1}\le Z_i \cdot (1+2^{-2\rlen+2})^k + (2^{-2n+2})^k$.
\end{lemma}

Suppose $s$ is at the $T$-th layer (which is the worst-case scenario). Assuming Lemma~\ref{lemma:Zi-grow}, we can bound $Z_T$ by applying it $T$ times and expanding it
\[
\begin{aligned}
Z_T & \le Z_0 \cdot (1+2^{-2\rlen+4})^{kT} + \sum_{i=0}^{T-1} (2^{-2n+2})^k \cdot (1+2^{-2\rlen+4})^{k(T-i)} \\
& \le T \cdot (Z_0 + (2^{-2n+2})^k) \cdot (1+2^{-2\rlen+4})^{kT} \\
& \le T \cdot (2^{-2n} \cdot 8 )^k \cdot 2 \\
& \le T \cdot 2^{-2nk} \cdot 2^{3k+1}.
\end{aligned}
\]
We obtain $Z_T\le 2^{4k+2r}\cdot 2^{-2kn}$. Consequently, we have
\[
\Pr[v_1\to s] \le \frac{2^{k\cdot (\ch(s)+\cb(s))} \cdot Z_T}{\langle \doubleP_{x|v_1\to s}, \doubleP_{x|v_1\to s}\rangle^k}\le 2^{-k(\ell_s - {\ell^{(j)}_b}-\ell-10)},
\]
which finishes the proof of \Cref{lemma:sig-state-multi-pass}.

The rest of the section is devoted to the proof of Lemma~\ref{lemma:Zi-grow}. The proof is adapted from \cite{GargRT18-extractor}. The main modification is that we are forced to handle some bad edges (because they are of high probability when conditioning on $v_1\to v$, we cannot simply stop before traversing such edges). However, our potential function has an extra term involving $\cb(v)+\ch(v)$ to account for the ``progress'' incurred by traversing bad edges.

For every $i\in [T]$, denote by $\Gamma_i$ the set of all edges $e=(v,\wtd{v})$ from $V_{i-1}$ to $V_{i}$ such that $\Pr[e] > 0$. We define the potential in the transition layer of $\Gamma_i$ as
\[
Z'_i = \sum_{e\in \Gamma_i} \Pr[e] \cdot 2^{-k\cdot (\ch(e)+\cb(e))} \cdot \langle\doubleP_{x|e}, \doubleP_{x|s}\rangle^k,
\]
where we define $\ch(e)$ and $\cb(e)$ as $\ch(\wtd{v}),\cb(\wtd{v})$.
The proof that $Z_i\le Z'_i$ is identical to \cite{GargRT18-extractor}. In the following, we prove 
\[
Z'_{i+1} \le Z_i \cdot (1+2^{-2\rlen+4})^k + (2^{-2n+2})^k.
\]

Fix $v\in V_i$ such that $\Pr[v] > 0$. Denote by $\Gamma(v)$ the set of outgoing edges from $v$ with non-zero traversing probability. It suffices to show that
\begin{claim}\label{claim:v1-to-e-slow-grow}
We have
\[
\sum_{e\in \Gamma(v)} \frac{\Pr[e] \cdot 2^{-k(\ch(e)+\cb(e))}}{\Pr[v]\cdot 2^{-k(\ch(v)+\cb(v))}} \cdot \langle \doubleP_{x|e}, \doubleP_{x|s} \rangle^k \le 
\langle \doubleP_{x|v}, \doubleP_{x|s} \rangle^k \cdot (1+2^{-2\rlen+4})^k + (2^{-2n+2})^k.
\]
\end{claim}

In the rest of the section, we prove \Cref{claim:v1-to-e-slow-grow}.

\begin{proof}
    If $v$ is a significant state, then $\Gamma(v)$ is empty, and the inequality follows trivially. In the following, we assume $v$ is not significant.

    Define $P:X\to \mathbb{R}_{\ge 0}$ as 
    \[
    P(x') = \mathbbm{1}[x'\notin \SigV^{(all)}(v)]\cdot \doubleP_{x|v}(x').
    \]
    By definition, we have $\| P \|_{\infty} \le 2^{\ellsigv}\cdot 2^{-n}$. Then we define $f:X\to \mathbb{R}^+$ as
    \[
    f(x') = P(x') \cdot \doubleP_{x|s}(x').
    \]
    It follows from \Cref{coro:sig-state-ell2-bound} that
    \begin{align}
    \| f \|_2 \le \| P \|_\infty \cdot \| \doubleP_{x|s} \|_2 \le 2^{\ellsigv+\ell_s + {\ell^{(j)}_b}+2} \cdot 2^{-2n}. \label{eq:bound-on-f-l2}
    \end{align}

    For any edge $e\in \Gamma(v)$ labeled by $(a,b)$ and any $x'\in X$, we have
    \[
    \doubleP_{x|e}(x') \cdot \doubleP_{x|s}(x') = \begin{cases}
        0 & \text{if $M(a,x')\ne b$} \\
        f(x') \cdot {c_e}^{-1} & \text{otherwise} 
    \end{cases}
    \]
    where $c_e$ is the normalization factor given in \Cref{edge-evolution-second-pass}. Consequently, 
    \begin{align}
    \langle \doubleP_{x|e}, \doubleP_{x|s} \rangle 
    &= {c_e}^{-1} \Ex_{x'\sim X} [ f(x') \cdot \mathbbm{1}[M(a,x')=b] ]. \label{eq:Pe-Ps-ip}
    \end{align}
    
    \paragraph*{High-probability edges.} If $a\in \High(v)$, we have
    \[
    \langle \doubleP_{x|e}, \doubleP_{x|s} \rangle^k \le \langle \doubleP_{x|v}, \doubleP_{x|s} \rangle^k \cdot {c_e}^{-k}.
    \]
    We also observe that
    \[
    \begin{aligned}
    2^{-k(\cb(e)+\ch(e))} 
    &\le 2^{-k(\cb(v)+\ch(v))} \cdot 2^{-\Delta(e)-1} \\
    &\le  2^{-k(\cb(v)+\ch(v))} \cdot {c_e}^k \cdot (1+2^{-2\rlen+2})^k,
    \end{aligned}
    \]
    and
    \[
    \Pr[e]\le \Pr[v]\cdot 2^{-n}.
    \]
    Consequently,
    \[
     \frac{\Pr[e]\cdot 2^{-k(\cb(e)+\ch(e))}}{\Pr[v]\cdot 2^{-k(\cb(v)+\ch(v))}} \frac{\langle \doubleP_{x|e}, \doubleP_{x|s} \rangle^k }{ \langle \doubleP_{x|v},\doubleP_{x|s} \rangle^k } \le 2^{-n} \cdot (1+2^{-\rlen+2})^k .
    \]
    Since there are at most $2^{n-\frac{\kext}{2}}$ high probability $a\in A$, we conclude that
    \[
    \sum_{e:e\in \High(v)} \frac{\Pr[e]\cdot 2^{-k(\cb(e)+\ch(e))}}{\Pr[v]\cdot 2^{-k(\cb(v)+\ch(v))}} \langle \doubleP_{x|e}, \doubleP_{x|s}\rangle^k \le 2^{-\frac{\kext}{2}}\cdot (1+2^{-2\rlen+2})^k \cdot \langle \doubleP_{x|v},\doubleP_{x|s}\rangle^{k}.
    \]
    \paragraph*{Other edges.} In the following, we deal with other edges. By \Cref{edge-evolution-second-pass}, we have ${c_e}^{-1} \le 2(1+2^{-2\rlen+2})$ for such edges. Suppose an edge has label $(a,b)$, then
    \[
    \frac{\Pr[e]\cdot 2^{-k(\cb(e)+\ch(e))}}{\Pr[v]\cdot 2^{-k(\cb(v)+\ch(v))}} = 2^{-n} \cdot \Pr_{x'\sim \doubleP_{x|v}}[M(a,x') = b].
    \]
    Depending on whether $\|f\|_1\le 2^{-2n}$ or not, we need to consider two sub-cases.

    \begin{itemize}
        \item If $\| f\|_1 \le 2^{-2n}$, then it follows from \eqref{eq:Pe-Ps-ip} that
        \[
        \langle \doubleP_{x|e}, \doubleP_{x|s} \rangle \le {c_e}^{-1} \cdot \| f\|_1 \le 4\cdot 2^{-2n}.
        \]
        Summing up over all edges concludes the proof, as 
        \[
        \sum_{e\notin \High(v)} \frac{\Pr[e]\cdot 2^{-k(\cb(e)+\ch(e))}}{\Pr[v]\cdot 2^{-k(\cb(v)+\ch(v))}} \langle \doubleP_{x|e}, \doubleP_{x|s}\rangle^k \le (2^{-2n+2})^k.
        \]
        \item If $\|f\|_1 \ge 2^{-2n}$, we shall apply the extractor property. For every $a\in A$, we define
        \[
        t(a) := \frac{| \langle M_a, f \rangle|}{\| f\|_1 }.
        \]
        Note that
        \[
        \Ex_{x'\sim X}[f(x')\cdot \mathbbm{1}[M(a,x')=b]] \le \frac{1+|\langle M_a,f\rangle|}{2}.
        \]
        Applying the bound of $c_{e}$ and the definition of $t(a)$ on \eqref{eq:Pe-Ps-ip}, we obtain
        \begin{align}
        \langle \doubleP_{x|e},\doubleP_{x|s} \rangle \le (1+2^{-2\rlen+2}) \cdot (1+t(a))^k. \label{eq:Pe-Ps-with-a}
        \end{align}
        Raising \eqref{eq:Pe-Ps-with-a} to the power of $k$ and taking the expectation over $e$, we obtain,
        \[
        \sum_{e\notin \High(v)} \frac{\Pr[e]\cdot 2^{-k(\cb(e)+\ch(e))}}{\Pr[v]\cdot 2^{-k(\cb(v)+\ch(v))}} \langle \doubleP_{x|e}, \doubleP_{x|s}\rangle^k 
        \le 
        \langle \doubleP_{x|v},\doubleP_{x|s} \rangle^k \Ex_{a\sim A}[(1+t(a))^k] \cdot (1+2^{-2\rlen+2})^k.
        \]
        We claim that 
        \begin{align}
        \Ex_{a\sim A}[(1+t(a))^k] \le (1+2^{-2\rlen})^k. \label{eq:claim-on-ta}
        \end{align}
        The claim holds because our assumption on $f$ implies that $\frac{\|f\|_2}{\|f\|_1}\le 2^{\ell}$, allowing us to use the extractor property of $M$ to argue that most $t(a)$ behave nicely. For now, we assume the claim and obtain
        \[
        \sum_{e\notin \High(v)} \frac{\Pr[e]\cdot 2^{-k(\cb(e)+\ch(e))}}{\Pr[v]\cdot 2^{-k(\cb(v)+\ch(v))}} \langle \doubleP_{x|e}, \doubleP_{x|s}\rangle^k 
        \le 
        \langle \doubleP_{x|v},\doubleP_{x|s} \rangle^k \cdot (1+2^{-2\rlen+1})^k.
        \]
    \end{itemize}

    \paragraph*{Wrap-up.} Before verifying \eqref{eq:claim-on-ta}, we finish the rest part of the proof.
    We have
    \begin{align*}
    &~~~~ \sum_{e} \frac{\Pr[e]\cdot 2^{-k(\cb(e)+\ch(e))}}{\Pr[v]\cdot 2^{-k(\cb(v)+\ch(v))}} \langle \doubleP_{x|e}, \doubleP_{x|s}\rangle^k \\
    &\le \sum_{e\in \High(v)} \left( ... \right) + \sum_{e\not\in \High(v)} \left( ... \right) \\
    &\le \langle \doubleP_{x|v}, \doubleP_{x|s} \rangle^k \cdot (1+2^{-2\rlen+3})^k \cdot 2^{-k}  \ + \tag{High-probability edges} \\
    &~~~~ \langle \doubleP_{x|v}, \doubleP_{x|s} \rangle^k \cdot (1+2^{-2\rlen+3})^k + (2^{-2n+2})^k \tag{Other edges} \\
    &\le \langle \doubleP_{x|v},\doubleP_{x|s} \rangle^k \cdot (1+2^{-2\rlen+4})^k + (2^{-2n+2})^k,
    \end{align*}
    which proves \Cref{claim:v1-to-e-slow-grow}.

    \paragraph*{Apply the Extractor property.} We have yet to verify \eqref{eq:claim-on-ta}. Recall we have assumed $\| f\|_1 \ge 2^{-2n}$. We also have $\| f \|_2 \le 2^{\ellext} \cdot 2^{-2n}$ by \eqref{eq:bound-on-f-l2}. Consequently, we have
    \[
    \frac{\|f\|_2}{\|f\|_1} \le 2^{\ellext}.
    \]
    Since $M$ is a $(\ellext,\rext,\rext)$-$L_2$-extractor, there can be at most $2^{-\kext} |A|$ many $a\in A$ with $t(a) \ge 2^{-\rext}$. Also, we always have $1+t(a)\le 2$. 
    
    Recall that $\rlen \le \frac{1}{4} \min(\rext, \kext)$ and $k =\frac{\kext}{2}$. Hence,
    \[
    \begin{aligned}
    \Ex_{a\sim A}[(1+t(a))^k ]  
    &\le \Pr_{a\sim A}[t(a) > 2^{-\rext}] \cdot 2^k
    + \Pr_{a\sim A}[t(a) \le 2^{-\rext}] \cdot (1+2^{-\rext})^k  \\
    &\le 2^{-\kext} \cdot 2^k + (1+2^{-\rext})^k \\
    &\le (1+2^{-2\rlen})^k.
    \end{aligned}
    \]
    This verifies \eqref{eq:claim-on-ta}, and completes the proof.
\end{proof}

\section{Missing Proofs for Multi-Pass} \label{appendix:multi-pass}
\subsection{Fixing the Randomness}

To check the definition of $G_{v'_{j - 1}}$ actually does not depend on the distribution of internal randomness. We need to examine the definition for the good event $G_{v'_{j - 1}}$ carefully. It involves two things, the computational path from $v'_{j - 1}$ and the stopping rules. First, note that the computational path from $v'_{j - 1}$ to $V^{(j - 1)}_T$ is completely determined by $x, a_{> i}$, and it does not depend on internal randomness. 

Second, we will show that the stopping rules do not depend on internal randomness. Formally, we will for any pass $j$, 
\begin{enumerate}
    \item For all $v \in V^{(j)}_t$ ($i \leq t \leq T$), for fixed $v'_j$, the event $v'_j \to v$ (the truncated path from $v'_j$ reaches $v$ without stopping) and $\doubleP_{x \mid v'_j \to v}$ are both independent of the internal randomness $x', a'_{\leq i}$. 
    \item For any fixed $v \in V^{(j)}_t$, ($i \leq t \leq T$) and $\tilde{x} \in X, \tilde{a}_{> i} \in A^{T - i}$, the probability
    $$\Pr[x = \tilde{x}, a_{> i} = \tilde{a}_{> i} \mid \widetilde{B}_{>i} \text{ reaches } v \text{ without stopping}].$$
    is independent of the internal randomness $x', a'_{\leq i}$.
\end{enumerate}

The base case is the empty program, and this trivially holds. Now suppose this is true for Pass $1, 2, \dots, j - 1$. We will prove it for Pass $j$ by the following induction. 
\begin{enumerate}[label=\Alph*.]
    \item Initially, for $ V^{(j)}_i$, i.e. the pass-$j$ starting layer of $\widetilde{B}_{> i}$, the event $v'_j \to v'_j$ (which is always true) and $\doubleP_{x \mid v'_j \to v'_j}$ (which is just uniform) are clearly independent of the internal randomness. For $v \in  V^{(j)}_i$ and any fixed $\tilde{x}, \tilde{a}_{> i}$, consider the probability 
    $$\Pr[x = \tilde{x}, a_{> i} = \tilde{a}_{> i} \mid \widetilde{B}_{>i} \text{ reaches } v \text{ without stopping}].$$
    Notice that (1) $v \in V^{(j)}_i$ remembers $v_{j - 1} \in V^{(j - 1)}_T$, (2) from induction hypothesis, $x, a_{> i}$ is independent of $x', a'_{\leq i}$ conditioning on $\widetilde{B}_{>i} \text{ reaches } v_{j - 1}$, and (3) the event $v_{j - 1} \wt v$ is determined solely by internal randomness $x', a'_{< i}$.  Together (2) and (3) implies $$x, a_{> i} \perp (v_{j - 1} \wt v) \mid \widetilde{B}_{>i} \text{ reaches } v \text{ without stopping}.$$
    Hence, for any fixed $\tilde{x}, \tilde{a}_{> i}$,
    \begin{align*}
        &\Pr[x = \tilde{x}, a_{> i} = \tilde{a}_{> i} \mid \widetilde{B}_{>i} \text{ reaches } v \text{ without stopping}] \\
    = & \Pr[x = \tilde{x}, a_{> i} = \tilde{a}_{> i} \mid \widetilde{B}_{>i} \text{ reaches } v_{j - 1} \text{ without stopping} \land (v_{j- 1} \wt v)] \tag{By (1)}\\
    = &\Pr[x = \tilde{x}, a_{> i} = \tilde{a}_{> i} \mid \widetilde{B}_{>i} \text{ reaches } v_{j - 1} \text{ without stopping}] \tag{By (2) + (3)}
    \end{align*}
    Hence it does not depend on $x', a'_{\leq i}$. \label{item:outer-base}
    \item Suppose this is true for all $v \in V^{(j)}_i, V^{(j)}_{i + 1}, \dots, V^{(j)}_{t - 1}$. We now want to prove it for $v \in V^{(j)}_t$. 
    \begin{enumerate}[label=B.\arabic*.]
        \item Independence for event $v'_j \to v$  and $\doubleP_{x \mid v'_j \to v}$: \label{item:inner-1}
        \begin{itemize}
            \item First, let us consider the possibility of stopping at vertex $u \in V^{(j)}_{t - 1}$. By the induction hypothesis, the event $v'_j \to u$ and $\doubleP_{x \mid v'_j \to u}$ are both independent of the internal randomness. Since whether we stop at $u$ only depends on $\doubleP_{x \mid v'_j \to u}$ and $x$, we know it is independent of internal randomness. 
            \item Second, let us consider the possibility of stopping on the edge from $u \in V^{(j)}_{t - 1}$ to $v \in V^{(j)}_t$. We need to check the independence for $\High(u)$ and the counters. $\High(u)$ depends on the probability 
            $$\Pr[a_{t + 1} = \tilde{a} \mid \widetilde{B}_{>i} \text{ reaches } u \text{ without stopping}],$$
            which by the induction hypothesis, does not depend on internal randomness. For the counters, $\ch$ only depend on $\High(u)$. $\cb$ depends on $\doubleP_{x \mid v'_j \to u}$ whose independence is from induction hypothesis.
            \item In Conclusion, as event $v'_j \to v$ is equivalent to $\bigvee_{u \in V^{(j)}_{t - 1}} v'_j \to u \land u \to v$. As (1) by the induction hypothesis, $v'_j \to u$ is independent of $x', a'_{< i}$, (2)  $u \wt v$ depends only on $x, a_t$ and is independent of $x', a'_{< i}$, and (3) from the case analysis above, whether we stop before reaching $v$ does not depend on $x', a'_{< i}$, we know that this event is independent of $x', a_{< i}$. \\
            
            Similarly, note $$\doubleP_{x \mid v'_j \to v} = \sum_{u \in V^{(j)}_{t - 1}} \frac{\doubleP_{x \mid v'_j \to u} \cdot \Pr[u \to v \mid v'_j \to u] \cdot \Pr[v'_j \to u] }{\Pr[v'_j \to v]}.$$
            (1) From the induction hypothesis, we know $\doubleP_{x \mid v'_j \to u}$ and $v'_j \to u$ are independent of $x', a_{\leq i}$. (2) We have just shown that $v_j \to v$ is also independent of $x', a_{\leq i}$, so is $\Pr[v'_j \to v]$. (3) From the case analysis above we know $\Pr[u \to v \mid v'_j \to u]$ is also independent of $x', a_{\leq i}$.
            
            Since each term is independent of $x', a'_{< i}$, we know $\doubleP_{x \mid v'_j \to v}$ is also independent of that.
    \end{itemize}
    \item Independence for $\Pr[x = \tilde{x}, a_{> i} = \tilde{a}_{> i} \mid \widetilde{B}_{>i} \text{ reaches } v \text{ without stopping}].$:
    \begin{itemize}
        \item Notice that (1) $v$ remembers $v'_j$ (which is the starting vertex of this layer), (2) we have proved in \Cref{item:outer-base} that $x, a_{> i}$ is independent of $x', a'_{\leq i}$ conditioning on $\tilde{B}_{> i}$ reaches $v_{j - 1}$, and (3) by \Cref{item:inner-1}, we know $v'_j \to v$ is independent of $x', a'_{\leq i}$ and is solely determined by $x, a_{> i}$. 

        Putting (2) and (3) together, we know 
        $$x, a_{> i} \perp (v'_j \to v) \mid \tilde{B}_{> i} \text{ reaches } v \text{ without stopping}.$$
        
        Hence, for any fixed $\tilde{x}, \tilde{a}_{> i}$, 
        \begin{align*}
            &\Pr[x = \tilde{x}, a_{> i} = \tilde{a}_{> i} \mid \widetilde{B}_{>i} \text{ reaches } v \text{ without stopping}] \\
            = &\Pr[x = \tilde{x}, a_{> i} = \tilde{a}_{> i} \mid \widetilde{B}_{>i} \text{ reaches } v'_j \text{ without stopping} \land v'_j \to v] \tag{By (1)}\\
            = &\Pr[x = \tilde{x}, a_{> i} = \tilde{a}_{> i} \mid \widetilde{B}_{>i} \text{ reaches } v'_j \text{ without stopping}] \tag{By (2) + (3)}
        \end{align*}
    \end{itemize}
    \end{enumerate}
\end{enumerate}

\subsection{Potential Analysis for Multi-Pass}\label{appendix:potential-for-multi-pass}

We prove \Cref{lemma:edge-potetial-multi} and \Cref{lemma:multi-pass-pot-slow-grow} in this section. For brevity, in this section, we use $u',v'$ to denote vertices in the $(j-1)$-th pass, and $u,v$ to denote vertices in the $j$-th pass.

We start by proving \Cref{lemma:edge-potetial-multi}. Recall its statement below.

\multiPotentialEdge*

\begin{proof}
We need to discuss the four cases: Let $v$ be the vertex we reach after traversing this edge.
\begin{itemize}
    \item We stopped on $u$ due to stopping rules. Then $\Phi^{a,b}(u) = \Phi(\fail) = 0$ since we force $u$ to traverse to $\fail$ after reading $(a,b)$.
    \item $a \not\in \Bad(u)$: In this case, both $\cb^{(j)}$ and $\ch^{(j)}$ are unchanged. We know $\Phi^{a,b}(u) = \Phi(u)$.  By the definition of bad edges, we know that $$\Pr_{x \mid v_{j-1} \to u}[M(a,x) = b] \in \left(\frac{1}{2} - 2^{-\rext}, \frac{1}{2} + 2^{-\rext}\right).$$
    
    As a result $$\frac{\Phi^{a,b}(u) }{\Phi(u)} = 1 \leq \frac{1 + 2 \cdot 2^{-\rext}}{2 \Pr_{x \mid v_1 \to u}[M(a,x) = b]}.$$
    \item  $a \in \Bad(u) \setminus \High(u)$: By our stopping rule, we will stop on this edge. Hence $\Phi^{a,b}(u) = \Phi(\fail) = 0$.  The inequality trivially holds. 
    \item $a \in \High(u)$: In this case, the counter $\ch^{(j)}(v) = \ch^{(j)}(u) + 1$ and $\cb^{(j)}(v) = \cb^{(j)}(u) + \Delta$ with $ \Delta = \left\lfloor-\log\left( \Pr_{x \mid v_{j-1}\to u}[M(a,x)=b] \right) \right\rfloor$. Hence, 
    $$\frac{\Phi^{a,b}(u) }{\Phi(u)} = 2^{\Delta - 1} \leq \frac{1}{2 \Pr_{x \mid v_{j-1} \to u}[M(a,x) = b]}.$$
\end{itemize}
Having verified all possible cases, we conclude the validity of the lemma.
\end{proof}

Next, we prove \Cref{lemma:multi-pass-pot-slow-grow}.

\multiPotentialGrow*

\begin{proof}
    We use induction on $i\in [T]$. For the case that $i=0$, $v_{j-2}$ (resp. $v_{j-1}$) is the only vertex such that $\Pr[v_{j-2}\to v']$ (resp. $\Pr[v_{j-1}\to v]\ne 0$). The lemma holds trivially. Next, suppose the lemma holds for $i$, we prove it for the case of $i+1$.

    For any edge $e'$ linking $u'\in V^{(j-1)}_{i}$ and $v'\in V^{(j-1)}_{i+1}$, let $v_{j-2}\to e'$ denote the event that the program traverses the edge $e'$ without stopping at $u'$.

    Fix one $v'\in V^{(j-1)}_{i+1}$. Denote by $\Gamma^-(v')$ the set of incoming edge $e'$ to $v'$ such that $\Pr[v_{j-2}\to e'] \ne 0$. We observe that
    \[
    \Pr[v_{j-2}\to v'] = \sum_{e'\in \Gamma^-(v')}\Pr[v_{j-2}\to e'],
    \]
    and
    \[
    \Ex[\Phi(v) \cdot \mathbbm{1}[v_{j-2}\to v']] = \sum_{e'\in \Gamma^{-}(v')} \Ex_{v_{j-1}\to e}[\Phi(e) \cdot \mathbbm{1}[v_{j-2}\to e']].
    \]
    When we write $\mathbbm{1}[v_{j-2}\to e']$ in the expectation, we are considering the natural coupling between the truncated paths from $v_{j-2}$ and $v_{j-1}$. Thus, $\mathbbm{1}[v_{j-2}\to e']$ denotes the event that the first pass traverses the edge $e'$. Observe that the value of $\mathbbm{1}[v_{j-2}\to e']$ is uniquely determined after conditioning on $v_{j-1}\to e$. More precisely, $\mathbbm{1}[v_{j-2}\to e']$ is true, if and only if $e$ links two vertices $(u,v)$ where $u$ remembers $u'$, and the label of $e$ is exactly $(a',b')$. 
    
    The following lemma is the key step in the induction proof.

    \begin{lemma}\label{lemma:multi-pass-pot-vertex-to-edge}
    Assuming that for every $u'\in V^{(j-1)}_i$ with $\Pr[v_{j-2}\to u'] \ne 0$, we have
    \[
    \Ex[\Phi(u) \mid v_{j-2}\to u'] := \frac{\Ex_{u\sim \Ttwo}[\Phi(u) \cdot \mathbbm{1}[v_{j-2}\to u']]}{\Pr[v_{j-2}\to u']}\le (1+2^{-2\rlen+2})^{i}\cdot 2^{\cb^{(j-1)}(u')}
    \]
    Then, for every $e'$ being an edge from $V^{(1)}_{i}$ to $V^{(1)}_{i+1}$ such that $\Pr[v_{j-2}\to e'] \ne 0$, we have
    \[
    \frac{\Ex_{e\sim \Ttwo[i,i+1]}[\Phi(e) \cdot \mathbbm{1}[v_{j-2}\to e']]}{\Pr[v_{j-2}\to e']} \le (1+2^{-2\rlen+2})^{i+1} \cdot 2^{\cb^{(j-1)}(e')},
    \]
    where we define the bias counter for every edge $e'=(u',v')$ as $\cb^{(j-1)}(e') \coloneqq \cb^{(j-1)}(v')$.
    \end{lemma}

    Before proving \Cref{lemma:multi-pass-pot-vertex-to-edge}, we assume it and quickly prove \Cref{lemma:multi-pass-pot-slow-grow}.

    \[
    \begin{aligned}
        \Ex[\Phi(v) \cdot \mathbbm{1}[v_{j-2}\to v']] 
        &= \sum_{e'\in \Gamma^{-}(v')} \Ex_{v_1\to e}[\Phi(e) \cdot \mathbbm{1}[v_{j-2}\to e']] \\
        &\le \sum_{e'\in \Gamma^-(v')} \Pr[v_{j-2}\to e'] \cdot (1+2^{-2\rlen+2})^{i+1}\cdot 2^{\cb^{(j-1)}(e')} \\
        &\le \Pr[v_{j-2}\to v'] \cdot (1+2^{-2\rlen+2})^{i+1} \cdot 2^{\cb^{(j-1)}(v')}.
    \end{aligned}
    \]
    Re-arranging proves the desired bound for $v'$. As the argument holds for every $v'\in V^{(j-1)}_{i+1}$, we have verified the lemma for $i+1$. By induction on $i$, this completes the proof.
\end{proof}

We are left to prove \Cref{lemma:multi-pass-pot-vertex-to-edge}.

\begin{proof}[Proof of \Cref{lemma:multi-pass-pot-vertex-to-edge}]
    Suppose $e'$ connects $u'\in V^{(j-1)}_i$ and $v' \in V^{(j-1)}_{i+1}$ with label $(a',b')$. As $\Pr[v_{j-2}\to e'] > 0$, we know $e'\notin \Bad(u')\setminus \High(u')$. 
    
    \paragraph*{Lower-bounding denominator.} Consider the denominator term. We have
    \begin{align*}
    \Pr[v_{j-2}\to e'] 
    &= \Pr[a_{i+1}=a'] \cdot \Pr_{\Tone}[(v_{j-2}\to u')\land (x\notin \SigV(u'))\land (M(a',x')=b')\mid a_{i+1}=a'] \\
    &\ge  2^{-n} \cdot \Pr[v_{j-2}\to u'] \left( \Pr_{x'\sim \doubleP_{x|v_{j-2}\to u'}}[M(a',x')=b'] - 2^{2\ell^{(j-1)}_s-\ellsigv}\right).
    \end{align*}
    We consider two cases depending on whether $e'\in \High(u')$.
    \begin{itemize}
        \item Case 1. $e'\notin \High(u')$. In this case, $e'$ is not a bad edge, and $\cb^{(j-1)}(e') = \cb^{(j-1)}(u')$. Since $\min(\rext,\ellsigv) - 2\ell^{(j-1)}_s\ge 2\rlen+1$, we have
        \[
         \Pr_{x'\sim \doubleP_{x|v_{j-2}\to u'}}[M(a',x')=b'] - 2^{2\ell^{(j-1)}_s-\ellsigv} \ge \frac{1}{2}-2^{-\rext} - 2^{2\ell^{(j-1)}_s-\ellsigv} \ge \frac{1}{2} - 2^{-2\rlen}.
        \]
        As a result
        \[
        \Pr[v_{j-2}\to e'] \cdot 2^{\cb^{(j-1)}(e')} \ge (\frac{1}{2} - 2^{-2\rlen}) \Pr[v_{j-2}\to u'] \cdot 2^{\cb^{(j-1)}(u')}.
        \]
        \item Case 2. $e'\in \High(u')$. In this case, recall 
        \[
        2^{-(\cb^{(j-1)}(e')-\cb^{(j-1)}(u'))} \le 2\cdot \Pr_{x'\sim \doubleP_{x|v_{j-2}\to u'}}[M(a',x')=b'].
        \]
        Since $\ellsigv \ge \ellhigh^{(j-1)}+2\ell^{(j-1)}_{s} + 2\rlen$, we have
        \[
         2^{-(\cb^{(j-1)}(e')-\cb^{(j-1)}(u'))} \cdot2^{-2\rlen} \le 2^{2\ell^{(j-1)}_s-\ellsigv}.
        \]
        The two inequalities above imply that
        \[
        \Pr_{x'\sim \doubleP_{x|v_{j-2}\to u'}}[M(a',x')=b'] - 2^{2\ell^{(j-1)}_s-\ellsigv} \ge 2^{-(\cb^{(j-1)}(e')-\cb^{(j-1)}(u'))} \left( \frac{1}{2} - 2^{-2\rlen} \right).
        \]
        Consequently,
        \[
        \Pr[v_{j-2}\to e'] \cdot 2^{\cb^{(j-1)}(e')} \ge (\frac{1}{2} - 2^{-2\rlen}) \Pr[v_{j-2}\to u'] \cdot 2^{\cb^{(j-1)}(u')}.
        \]
    \end{itemize}
    \paragraph*{Upper-bounding numerator.} On the other hand, consider the truncated paths from $v_{j-2}$ and $v_{j-1}$ (under the natural coupling). Conditioning on $v_{j-1}\to e$ for some $e\in V^{(j)}_i\times V^{(j)}_{i+1}$, recall that $\mathbbm{1}[v_{j-2}\to e']$ is true if and only if $e$ has label $(a',b')$ and connects two vertices $(u,v)$ where $u$ remembers $u'$. We consider the distribution of $v_{j-1}\to u$ for $u\in V^{(j-1)}_i$ and calculate the probability that the program traverses an edge from $u$ with label $(a',b')$. That is, we can write
    \begin{align*}
    \Ex_{v_{j-1}\to e}[\Phi(e) \cdot \mathbbm{1}[v_{j-2}\to e']] 
    &= \Ex_{v_{j-1}\to u,x’}\left[ \Phi^{a',b'}(u)\cdot \mathbbm{1}[v_{j-2}\to u'] \cdot \mathbbm{1}[(x'\notin \SigV(u))\land (M(a',x')=b')] \right] \cdot \\ &\qquad \Pr[a_{i+1}=a'] \\
    &\le 2^{-n} \Ex_{v_{j-1}\to u}\left[ \Phi^{a',b'}(u)\cdot \mathbbm{1}[v_{j-2}\to u'] \cdot \Pr_{x'\sim \doubleP_{x|v_{j-1}\to u}}[M(a',x')=b'] \right] \\
    &\le 2^{-n} \Ex_{v_{j-1}\to u}\left[ \Phi(u)\cdot \frac{1+2^{-\rext+1}}{2} \mathbbm{1}[v_{j-2}\to u']  \right]  \\
    &\le \Ex_{v_{j-1}\to u}\left[ \Phi(u) \cdot \mathbbm{1}[v_{j-2}\to u']\right] \cdot 2^{-n} \left( \frac{1}{2} + 2^{-\rext}\right).
    \end{align*}
    We justify the derivation. The first inequality is because $a_{i+1}$ is independent of $a_{\le i}$ and $x$. The second inequality follows because dropping the condition $x'\notin \SigV(u)$ does not make the expression larger. Furthermore, conditioning on $(v_1\to u)\land (v_{j-2}\to u')$, $x$ is distributed as $\doubleP_{x|v_1\to u}$. The third inequality is by \Cref{lemma:edge-potetial-multi}.

    Combining the two bounds, we obtain
    \[
    \frac{ \Ex_{e}[\Phi(e) \cdot \mathbbm{1}[v_{j-2}\to e']] } { \Pr[v_{j-2}\to e'] \cdot 2^{\cb^{(j-1)}(e')} } \le \frac{\Ex\left[ \Phi(u) \cdot \mathbbm{1}[v_{j-2}\to u']\right]}{\Pr[v_{j-2}\to u']  \cdot 2^{\cb^{(j-1)}(u')}} \cdot \frac{\frac{1}{2} + 2^{-\rext}}{\frac{1}{2} - 2^{-2\rlen}}\le (1+2^{-2\rlen+2})^{i+1},
    \]
    as claimed.
\end{proof}

\section{A Learning Algorithm with Multiple Passes} \label{appendix:multi-pass-upper-bound}
In this section, we show the following algorithm result: if $q$ passes are allowed, one can learn parity with $O(n^{2}/\log(q))$ bits of memory and a polynomial number of samples.

\subsection{Setup}

We will give multi-pass upper bounds for parity learning.
\begin{itemize}
    \item An unknown $x \in \{0, 1\}^n$ is chosen uniformly and randomly. 
    \item The input stream contains $T$ samples $(a_1, b_1), \dots, (a_T, b_T)$, where $a_i$ is chosen uniformly and randomly from $\{0,1\}^n$ and $b_i = \langle x, a_i \rangle$.
    \item The multi-pass branching program reads these $T$ samples and output its guess of $x$.
\end{itemize}

\begin{theorem} \label{theo:multi-pass-algorithm}
For any $q \le 2^n$, there exists a branching program of $2^{O(n^2/\log q)}$ width, $q$ passes, and $O(qn)$ samples that solves $n$-bit parity learning with constant probability.
\end{theorem}

\subsection{Block circuit}

Although we ultimately want an upper bound for multi-pass branching programs, it is easier to first work with the following type of circuit. We will show that algorithms under this circuit model naturally implies algorithms for branching programs. 

\begin{definition}[Block Circuit]
    A block circuit $C: \{0,1\}^{n \cdot c} \to \{0,1\}^{m \cdot c}$ with capacity $c$ and depth $d$ is a depth-$d$ multi-output circuit with following wires and gates:

    \begin{itemize}

        \item Each wire carries $c$ bits. We view each $c$ bits as a single \emph{block}.
        \item It has $n$ input wires and $m$ output wires. 
        \item The gates have fan-in at most $4$. A gate $g$ with fan-in $k$ can be computing any function $g: \{0,1\}^{k \cdot c} \to \{0,1\}^c$.
    \end{itemize}
    
\end{definition}

We will describe our algorithms in terms of block circuits. Algorithms for block circuits imply algorithms for multi-pass branching programs.

\begin{claim} \label{circuit-transfer-lemma}
Let $C$ be a block circuit with capacity $c$ and depth $d$ that computes a function $f:\{0,1\}^{n \cdot c} \to \{0,1\}^{m \cdot c}$. Then there exists a branching program $B$ with width $2^{4 c \cdot (d + m)}$ that computes the same function $f$ in $m \cdot 4^{d}$ passes.
\end{claim}

\begin{proof}
First of all, it is sufficient to prove this statement for $ m = 1$. This is because if we can construct a $4^{d}$-pass branching program for computing a single-output-gate block circuit, by simply applying it to all $m$ output gates sequentially, we can prove this statement. Note such a sequential application blows up the number of passes by $m$, and we need  $m \cdot c$ extra space to remember the answer from each application. 

We will prove the $m = 1$ case via induction. In the base case, when depth $d = 0$, the circuit is trivial and the program only needs space $c$ to remember the output. 

Suppose for depth $d - 1$, all block circuits with capacity $c$ and depth $d - 1$ can be transformed into a branching program $B$ with width $2^{4 \cdot c\cdot (d - 1)}$ and $4^{d - 1}$ passes. 

Now consider the unique output gate $g$. It has at most $4$ fan-in wires, each carrying the output of a depth-$(d - 1)$ capacity-$c$ block circuit. We can use $4\cdot 4 ^{d - 1} \le 4^{d}$ passes to compute the block carried by each input wire and use $4c$ bits to store them. Then by the non-uniformity of the branching program, we can hard-wire $g$ into the transition. Hence computing $g$ does not require extra space. Storing the $4c$ bits of intermediate results blows up the width from $2^{4c(d-1)}$ to $2^{4c(d-1)}\cdot 2^{4c}$. This gives a $4^{d}$-pass $2^{4cd}$-width branching program for computing the output of $C$. 

\end{proof}

\subsection{Proof of \Cref{theo:multi-pass-algorithm}}

First, we collect $n$ samples $(a_1, b_1), (a_2, b_2), \dots, (a_n, b_n)$. Let $A$ be the matrix with $a_i^T$ as its $i$-th row. We will perform Gaussian Elimination to solve $Ax = b$ and get $x$. The trick is that we will view $\tilde{A}^{(0)} = [A \mid b]$ as a $K\times (K + 1)$ block matrix with blocks of size $(n/K) \times (n/K)$ each. (Although the vector $b$ only takes one column, we might just fill the rest $(n/K) - 1$ columns in those blocks by $0$.)

Then we work with the block circuit with capacity $c = (n/K)^2$. Each wire will be able to carry exactly one block in $\tilde{A}^{(0)}$. 

In the $i$-th step of the Gaussian Elimination, let our current matrix be $\tilde{A}^{(i)}$. We first multiply all blocks of the $i$-th row ($i \in [K]$) with $\left(\tilde{A}^{(i)}_{i,i}\right)^{-1}$. (When it is not invertible, the algorithm immediately fails.) Then we use the $i$-th row to eliminate the $i$-th block of the $j$-th row for all $j \neq i$.

Since $a_1, a_2, \dots, a_n$ are uniformly random, the matrix $A$ is also uniformly random. Note in the $i$-th step, the difference $\tilde{A}_{i,i}^{(i)} - \tilde{A}_{i,i}^{(0)}$ only depends on the original block $\tilde{A}_{j,k}^{(0)}$'s for $1 \leq j \leq K, 1 \leq k \leq i, (j,k) \neq (i,i)$. Conditioning on any realization of all such $A_{j,k}$'s, this difference is then fixed. Under such conditioning, the block $\tilde{A}^{(i)}_{i,i}$ is still uniformly random because $\tilde{A}_{i,i}^{(0)}$ is uniformly random. With at least $\frac{1}{4}$ probability, the random $n/K \times n/K$ matrix $\tilde{A}_{i,i}^{(i)}$ over $\mathbb{F}_2$ is invertible. After performing such Gaussian Elimination, $\tilde{A}^{(n)} = [I \mid b']$. Then we know $x = b'$. 

In the corresponding block circuit, each gate is of form 
\begin{align*}
    \begin{cases}
    \tilde{A}_{j,k}^{(i + 1)} =  \left(\tilde{A}^{(i)}_{i,i}\right)^{-1} \cdot \tilde{A}^{(i)}_{i,k} & \text{if $i = j$}\\
    \tilde{A}_{j,k}^{(i + 1)} = \tilde{A}^{(i)}_{j,k} - 
    \tilde{A}^{(i)}_{j,i} \cdot \left(\tilde{A}^{(i)}_{i,i}\right)^{-1} \cdot \tilde{A}^{(i)}_{i,k} & \text{otherwise}
    \end{cases}.
\end{align*}
So it has fan-in at most $4$. This block circuit for Gaussian Elimination will have depth $K$, and capacity $c = (n / K)^2$. Finally, it will have output $x$ with $m = K$ output gates. 

Let $K = \frac{1}{5}\log q$. By \Cref{circuit-transfer-lemma}, this gives a branching program with width $2^{O(n^2 /  \log q)}$  and $K\cdot 4^{K}$ passes that solve learning parity with at least $\frac{1}{4^K}$ probability. We can boost this probability to constant by sequentially repeating this algorithm $4^K$ times, each with a new set of samples. This blows up the number of passes to $K \cdot 4^{K} \cdot 4^{K} \le 2^{5K} \le q$. Hence we end up with a branching program with at most $q$ passes. The number of samples used is $4^K \cdot n \le qn$.

\end{document}